\documentclass[twoside,11pt]{article}

\usepackage[usenames,dvipsnames,svgnames,table]{xcolor}

\usepackage{enumerate}
\newcommand{\tb}{\textbf}
\newtheorem{mythm}{Theorem}
\newtheorem{mylemma}{Lemma}
\usepackage{algorithm}
\usepackage{algorithmic}

\usepackage{url}
\usepackage{bbm}

\usepackage{epsfig}
\usepackage{amsmath,amsfonts,graphicx}
\usepackage{jmlr2e}
\usepackage{graphicx} 
\usepackage{url}
\usepackage{framed}
\usepackage{wrapfig}
\usepackage{enumitem}
\usepackage{tabularx}
\usepackage{cleveref}

\crefname{thm}{theorem}{theorems}
\crefname{thm}{Theorem}{Theorems}
\crefname{assum}{assumption}{assumptions}
\Crefname{assum}{Assumption}{Assumptions}
\usepackage{xcolor}
\usepackage{dsfont}

\usepackage{amsfonts}
\usepackage{amsmath}
\newcommand{\mb}{\mathbf}

\usepackage{color}
\newcommand{\qirong}[1]{{\color{blue}{}}}

\DeclareRobustCommand\onedot{\futurelet\@let@token\@onedot}
\def\@onedot{\ifx\@let@token.\else.\null\fi\xspace}

\newcommand\ci{\perp\!\!\!\perp}

\jmlrheading{}{}{}{}{}{}

\ShortHeadings{Latent Variable Modeling with Diversity-Inducing Mutual Angular Regularization}{Xie, Deng and Xing}
\firstpageno{1}

\begin{document}

\title{Latent Variable Modeling with Diversity-Inducing Mutual Angular Regularization}

\author{\name Pengtao Xie \email pengtaox@cs.cmu.edu \\
\addr Machine Learning Department\\
Carnegie Mellon University\\
Pittsburgh, PA 15213, USA
\AND
\name Yuntian Deng \email yuntiand@cs.cmu.edu \\
\addr Language Technologies Institute\\
Carnegie Mellon University\\
Pittsburgh, PA 15213, USA
\AND
\name Eric Xing \email epxing@cs.cmu.edu \\
\addr School of Computer Science\\
Carnegie Mellon University\\
Pittsburgh, PA 15213, USA}

\editor{}

\maketitle

\begin{abstract}
One central task in machine learning (ML) is to extract underlying patterns, structure and knowledge from observed data, which is essential for making effective use of big data for many applications. Among the various ML models and algorithms designed for pattern discovery, latent variable models (LVMs) are a large family of models providing a principled and effective way to uncover knowledge hidden behind data and have been widely used in text mining, computer vision, speech recognition, computational biology and recommender systems. Due to the dramatic growth of volume and complexity of data, several new challenges have emerged and cannot be effectively addressed by existing LVMs: 1) In the event that the popularity of patterns behind big data is distributed in a power-law fashion, where a few dominant patterns occur frequently whereas most patterns in the long-tail region are of low popularity, standard LVMs are inadequate to capture the long-tail patterns, which can incur significant information loss. 2) To cope with the rapidly growing complexity of patterns present in big data, ML practitioners typically increase the size and capacity of LVMs, which incurs great challenges for model training, inference, storage and maintenance --- how to reduce model complexity without compromising expressivity? 3) There exist substantial redundancy and overlapping amongst patterns discovered by existing LVMs from massive data, making them hard to interpret --- how to promote low redundancy? To addresses the three challenges discussed above, we develop a novel regularization technique for LVMs, which controls the geometry of the latent space during learning to enable the learned latent components of LVMs to be diverse in the sense that they are favored to be mutually different from each other, to accomplish long-tail coverage, low redundancy, and better interpretability. We propose a mutual angular regularizer (MAR) to encourage the components in LVMs to have larger mutual angles. The MAR is non-convex and non-smooth, entailing great challenges for optimization. To cope with this issue, we derive a smooth lower bound of the MAR and optimize the lower bound instead. We show that the monotonicity of the lower bound is closely aligned with the MAR to qualify the lower bound as a desirable surrogate of the MAR. Using neural network (NN) as an instance, we analyze how the MAR affects the generalization performance of NN. On two popular latent variable models --- restricted Boltzmann machine and distance metric learning, we demonstrate that MAR can effectively capture long-tail patterns, reduce model complexity without sacrificing expressivity and improve interpretability.
\end{abstract}


\section{Introduction}
\label{sec:intro}
One central task in machine learning (ML) is to extract underlying patterns from observed
data \citep{bishop2006pattern,han2011data,fukunaga2013introduction}, which is essential for making effective use of big data for many applications \citep{council2013frontiers,jordan2015machine}.
For example, in the healthcare domain \citep{sun2013big}, with the
the prevalence of new technologies, data from electronic healthcare records (EHR), sensors, mobile applications, genome, social media is growing in both volume and complexity at an unexpected rate \citep{sun2013big}. Distilling high-value information such as clinical phenotypes \citep{ho2014marble,wang2015rubik}, treatment plans \citep{razali2009generating,martin2002method}, patient similarity \citep{wang2011integrating,sun2012supervised} etc., from such data is of vital importance, but highly challenging. Among the various ML models and algorithms designed for pattern discovery, latent variable models (LVMs) \citep{rabiner1989tutorial,bishop1998latent,knott1999latent,blei2003latent,hinton2006fast,
airoldi2009mixed,blei2014build} or latent space models (LSMs) \citep{rumelhart1985learning,deerwester1990indexing,olshausen1997sparse,
lee1999learning,xing2002distance} are a large family of models providing a principled and effective way to uncover knowledge hidden behind data and have been widely used in text mining \citep{deerwester1990indexing,blei2003latent}, computer vision \citep{olshausen1997sparse,lee1999learning}, speech recognition \citep{rabiner1989tutorial,hinton2012deep}, computational biology \citep{xing2007bayesian,song2009keller} and recommender systems \citep{gunawardana2008tied,koren2009matrix}. For instance, semantic-oriented distillation models such as Latent Semantic Indexing (LSI)~\citep{deerwester1990indexing}, Latent Dirichlet Allocation (LDA)~\citep{blei2003latent}, Sparse Coding~\citep{olshausen1997sparse}, have led to a number of breakthroughs in automatic extraction of topics \citep{blei2003latent}, entity types \citep{shu2009latent}, storylines \citep{ahmed2012timeline} from textual information. Multi-layer neural networks \citep{krizhevsky2012imagenet,bahdanau2014neural} and latent graphical models \citep{salakhutdinov2009deep,mohamed2011deep} have demonstrated great success in automatic learning of low-, middle- and high-level features from raw data and greatly advanced image classification \citep{krizhevsky2012imagenet}, machine translation \citep{bahdanau2014neural} and speech recognition \citep{hinton2012deep}. In healthcare domain, LVMs have also been applied to various analytic applications, including topic model for clinical notes analysis \citep{arnold2012topic,cohen2014redundancy}, restricted Boltzmann machine \citep{hinton2010practical} for patient profile modeling \citep{nguyen2013latent,tran2015learning}, distance metric learning \citep{xing2002distance} for patient similarity measure \citep{sun2012supervised,wang2011integrating}, tensor factorization \citep{cichocki2008nonnegative} for computational phenotyping \citep{ho2014marble,wang2015rubik}, to name a few.

\begin{figure}[t]
\begin{center}
\includegraphics[width=0.5\textwidth]{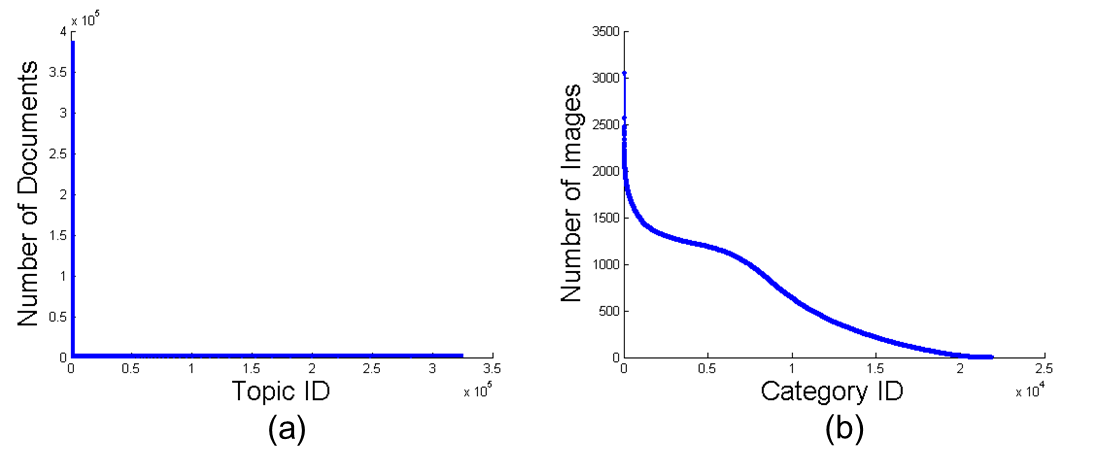}
\end{center}
\caption{(a) The number of documents belonging to each topic in the wikipedia dataset; (b) the number of images in each category of the ImageNet dataset.}
\label{fig:longtail}
\end{figure}

\begin{figure}
\begin{center}
\includegraphics[width=0.5\textwidth]{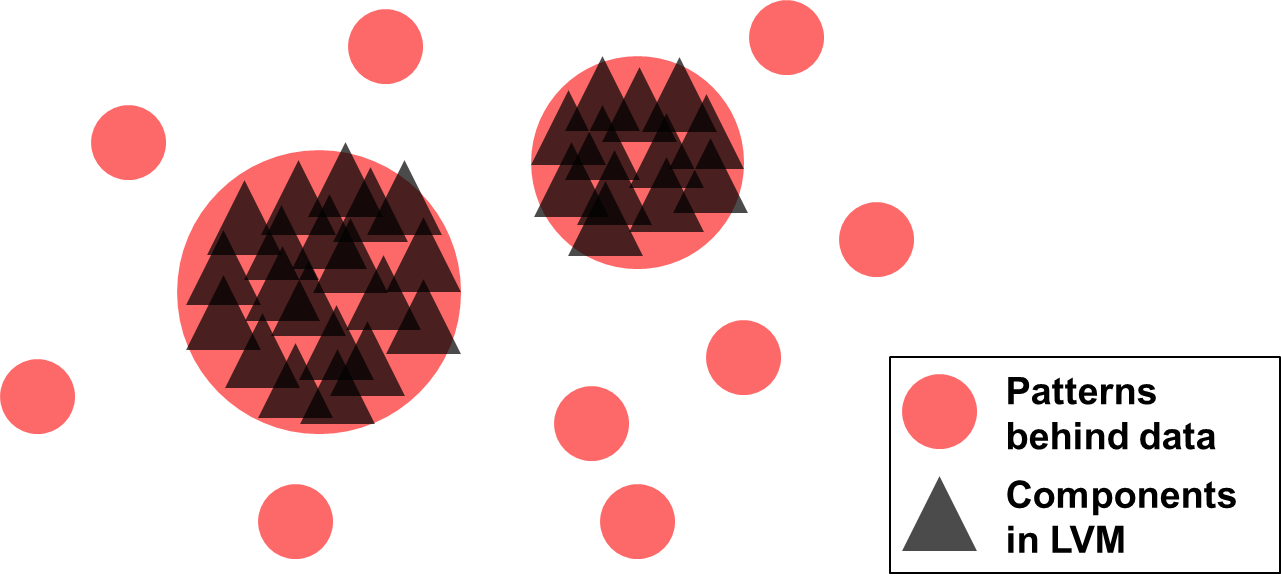}
\end{center}
\caption{Conventional LVMs: circles denote patterns in knowledge and triangles denote the components in an LVM. The size of the circle is proportional to the popularity of the pattern.}
\label{fig:lvm_fail}
\end{figure}

\begin{figure}[t]
\begin{center}
\includegraphics[width=0.5\textwidth]{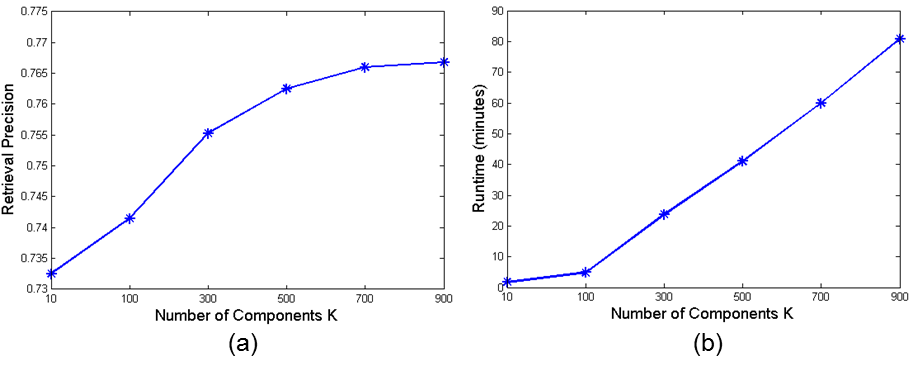}
\end{center}
\caption{ (a) Retrieval precision versus K; (b) Runtime versus K.}
\label{fig:dml}
\end{figure}

Although LVMs have now been widely used, several new challenges have emerged due to the dramatic growth of volume and complexity of data:
1) In the event that the popularity of patterns behind big data is distributed in a power-law fashion, where a few dominant patterns occur frequently whereas most patterns in the long-tail region are of low popularity \citep{wang2014peacock,xie2015diversifying}, standard LVMs are inadequate to capture the long-tail patterns, which can incur significant information loss \citep{wang2014peacock,xie2015diversifying}. For instance, Figure~\ref{fig:longtail}(a) shows the distribution of the number of documents belonging to each topic in the wikipedia \citep{partalas2015lshtc} document collection with 2.4M documents and 0.33M topics. 
Dominant topics such as politics, economics are of high frequency whereas long-tail topics such as symphony, painting are of low popularity. A possible reason for standard LVMs to be inadequate to capture the long-tail patterns may lie in the design of their objective function used for training. For example, a maximum likelihood estimator would reward itself by modeling the dominant patterns well as they are the major contributors of the likelihood function, as illustrated in Figure~\ref{fig:lvm_fail}.
Since dominant patterns denoted by these two large circles are the major contributors of the likelihood function, LVMs would allocate a number of triangles to cover the large circles as best as possible. On the other hand, the long-tail patterns denoted by the small circles contribute less to the likelihood function, thereby it is not very rewarding to model them well and LVMs tend to ignore them.
However, in practice, long-tail patterns are important and ignoring them would incur significant information loss, as we evidence below. First,
the volume of long-tail patterns can be quite large \citep{partalas2015lshtc,deng2009imagenet}. For example, in the Wikipedia dataset (Figure~\ref{fig:longtail}(a))
96.3\% topics are used by less than 100 documents.
The percentage of documents labeled with long-tail topics is 51.8\%.
Second, in some applications \citep{wang2014peacock}, the long-tail patterns can be more interesting and useful. For example, Tencent Inc applied topic models for advertisement, in one application \citep{Jin2014}, they showed that long-tail topics such as \textit{lose weight}, \textit{children nursing} improve the click-through rate by 40\%. 2) To cope with the rapidly growing complexity of patterns present in big data, ML practitioners typically increase the size and capacity of LVMs, which incurs great challenges for model training, inference, storage and maintenance \citep{xie2015learning} --- how to reduce model complexity without compromising expressivity? In LVMs the number of components $K$ incurs a tradeoff between model expressivity and complexity \citep{xie2015learning}. Under a small $K$, the model would have fewer parameters and hence of lower complexity and better computational and statistical efficiency. However, the downside is the expressivity of the model would be low. For a large $K$, the model would have high expressivity, but also high complexity and computational overhead. Figure~\ref{fig:dml} shows in a task \citep{xie2015learning} of document retrieval based on distance metric learning, how the retrieval precision and training time vary as the number of components K grows.
It is interesting to explore whether it is possible to simultaneously achieve high expressivity and low complexity under a small $K$. 3) There exist substantial redundancy and overlapping amongst patterns discovered by existing LVMs from massive data, making them hard to interpret \citep{wang2015rubik}. To better assist human to explore the data and make decisions, it is desirable to learn patterns that are interpretable \citep{wang2015rubik}.
Oftentimes, the patterns extracted by standard LVMs have a lot of redundancy and overlapping \citep{wang2015rubik}, which are ambiguous and difficult to interpret. Such a problem is especially severe in big data where the amount and complexity of both patterns and data are large. For example, in computational phenotyping from EHR, it is observed that the learned phenotypes by standard matrix and tensor factorization have much overlap, causing confusion such as two similar treatment plans are learned for the same type of disease~\citep{wang2015rubik}. It is necessary to control the latent space during learning to make the patterns distinct and interpretable.

\begin{figure}
\begin{center}
\includegraphics[width=0.5\textwidth]{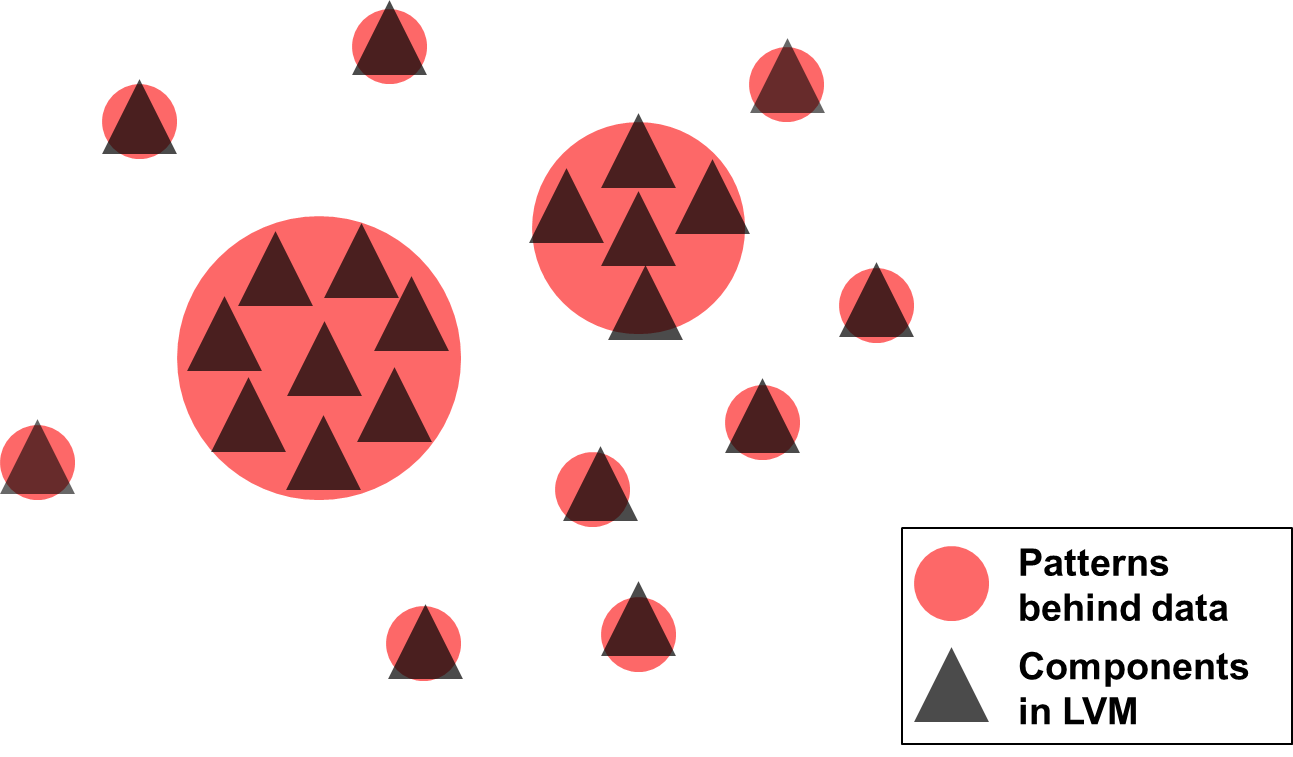}
\end{center}
\caption{ Long-tail LVMs: under diversification, the triangles (latent components) originally concentrated in the large circles (semantic patterns in data) now spread to different regions and cover the small circles.}
\label{fig:schematic_1}
\end{figure}

In this paper, we develop and investigate a novel regularization technique for LVMs, which controls the geometry of the latent space during learning, and simultaneously addresses the three challenges discussed above.
Our proposed methods enable the learned latent components of LVMs to be diverse in the sense that they are favored to be mutually "different" (in the sense to be made mathematically formal and explicit later in the proposal) from each other, to accomplish long-tail coverage, low redundancy, and better interpretability. First, concerning the long-tail phenomenon in extracting latent patterns (e.g., clusters, topics) from data: if the model components are biased to be far apart from each other, then one would expect that such components will tend to be less overlapping and less aggregated over dominant patterns (as one often experiences in standard clustering algorithms \citep{Zou_priorsfor}), and therefore more likely to capture the long-tail patterns, as illustrated in Figure~\ref{fig:schematic_1}.
Second, reducing model complexity without sacrificing expressivity: if the model components are preferred to be different from each other,
then the patterns captured by different components are likely to have less redundancy and hence complementary to each other. Consequently, it is possible to use a small set of components to sufficiently capture a large proportion of patterns, as illustrated in Figure~\ref{fig:schematic_2}. Third, improving the interpretability of the learned components: if model components are encouraged to be distinct from each other and non-overlapping, then it would be cognitively easy for human to associate each component to an object or concept in the physical world.

\begin{figure}
\begin{center}
\includegraphics[width=0.5\textwidth]{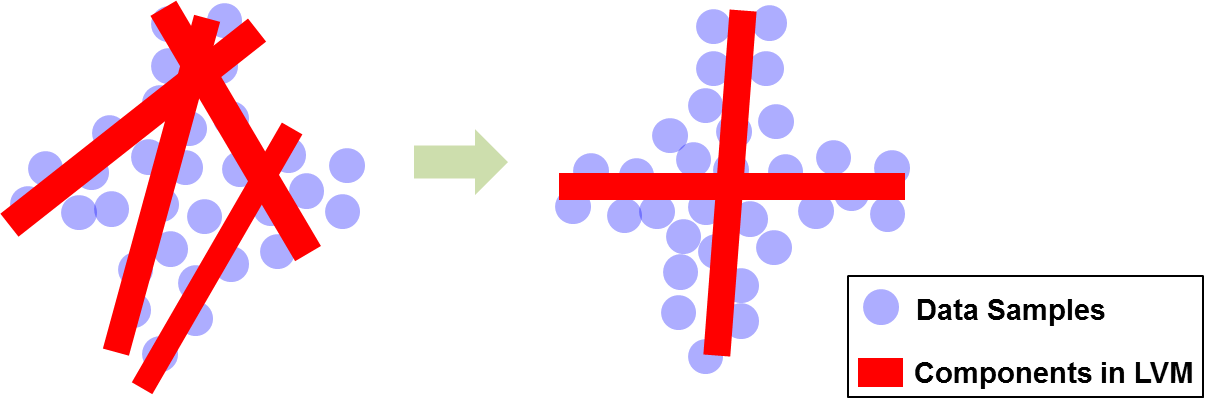}
\end{center}
\caption{ Low-Redundancy LVMs (blue circles denote data samples and red bars denote components): Without diversification, the red bars can be redundant and four bars are needed to cover the blue circles. With diversification, the red bars are forced to be different from each other and only two bars are needed to cover the circles well.
}
\label{fig:schematic_2}
\end{figure}

The major contributions of this paper are:
\begin{itemize}
\item We propose a diversity-promoting regularization approach to solve several key problems in latent variable modeling: capture long-tail patterns, reduce model complexity without compromising expressivity and improve interpretability.
\item We propose a mutual angular regularizer to encourage the components in LVMs to have larger mutual angles.
\item We develop optimization techniques to learn the mutual angle regularized LVMs, which are non-smooth and non-convex, hence are challenging.
\item Using neural network (NN) as a model instance, we analyze how the MAR affects the generalization error bounds of NN.
\item On restricted Boltzmann machine and distance metric learning, we empirically demonstrate that MAR can effectively capture long-tail patterns, reduce model complexity without sacrificing expressivity and improve interpretability.
\end{itemize}


The rest of the paper is organized as follows. In Section 2, we review related works. In Section 3, we propose the mutual angular regularizer and present the optimization techniques. Section 4 analyzes how the MAR affects the generalization errors of neural networks and Section 5 gives empirical evaluations of MAR on restricted Boltzmann machine and distance metric learning. Section 6 concludes the paper.

\section{Related Works}

\paragraph*{Latent Variable Models}
Latent Variable Models (LVMs) \citep{rabiner1989tutorial,bishop1998latent,knott1999latent,blei2003latent,hinton2006fast,
airoldi2009mixed,blei2014build} or more generally Latent Space Models (LSMs) \citep{
rumelhart1985learning,deerwester1990indexing,olshausen1997sparse,
lee1999learning,xing2002distance} are a large family of models in machine learning that are widely utilized for various application domains such as natural language processing \citep{blei2003latent,petrov2007discriminative}, computer vision \citep{fei2005bayesian,krizhevsky2012imagenet}, speech recognition \citep{rabiner1989tutorial,mohamed2011deep}, computational biology \citep{xing2007bayesian,song2009keller}, recommender systems \citep{mnih2007probabilistic,salakhutdinov2007restricted}, social network analysis \citep{airoldi2009mixed,ho2012triangular} and so on.
The utilities of LVMs/LSMs include but not limited to: 1) representation learning (Deep Neural Networks \citep{krizhevsky2012imagenet}, Deep Belief Network \citep{mohamed2011deep}, Restricted Boltzmann Machine \citep{hinton2010practical}); 2) semantic distillation (Latent Semantic Indexing \citep{deerwester1990indexing}, topic models \citep{blei2003latent}, Sparse Coding \citep{olshausen1997sparse}); 3) dimension reduction (Factor Analysis \citep{knott1999latent}, Principal Component Analysis \citep{jolliffe2002principal}, Canonical Component Analysis \citep{bishop2006pattern}, Independent Component Analysis \citep{hyvarinen2004independent}); 4) sequential data modeling (Hidden Markov Model \citep{rabiner1989tutorial}, Kalman Filtering \citep{grewal2014kalman}); 5) data grouping (Gaussian Mixture Model \citep{bishop2006pattern}, Mixture Membership Stochastic Block model \citep{airoldi2009mixed}); 6) latent factor discovery (Matrix Factorization \citep{koren2009matrix}, Tensor Factorization \citep{shashua2005non}). While existing latent variable models have demonstrated effectiveness on small to moderate scale data, they are inadequate to cope with new problems emerged in big data, such as the highly skewed distribution of pattern frequency \citep{wang2014peacock,xie2015diversifying}, the conflict between model complexity and effectiveness \citep{xie2015learning}, the interpretability of large amount of patterns discovered from massive data \citep{wang2015rubik}, as explained in Section \ref{sec:intro}.

\paragraph*{Regularization} Regularization \citep{hoerl1970ridge,tibshirani1996regression,recht2010guaranteed,
wainwright2014structured,srivastava2014dropout} is an important concept and technique in machine learning, which can help alleviate overfitting \citep{hoerl1970ridge,srivastava2014dropout}, reduce model complexity \citep{tibshirani1996regression,recht2010guaranteed}, achieve certain properties of parameters such as sparsity \citep{tibshirani1996regression,friedman2008sparse,jacob2009group,kim2010tree,jenatton2011structured}, low-rankness \citep{fazel2003log,recht2010guaranteed,candes2011robust}, stabilize an optimization problem \citep{wainwright2014structured} and lead to algorithmic speed-ups \citep{wainwright2014structured}.
Commonly used regularizers include squared $L_2$ norm \citep{hoerl1970ridge}, $L_1$ norm \citep{tibshirani1996regression}, group Lasso norm \citep{yuan2006model} for parameters represented by vectors and Frobenius norm \citep{mnih2007probabilistic} and trace norm \citep{recht2010guaranteed} for parameters represented by matrices, and structured Hilbert norms for functions in Hilbert spaces \citep{scholkopf2002learning}. Regularization approaches promoting diversity in the underlying solutions have been studied and applied in ensemble learning \citep{krogh1995neural,kuncheva2003measures,brown2005diversity,banfield2005ensemble,
tang2006analysis,partalas2008focused,yu2011diversity}, latent variable modeling \citep{Zou_priorsfor,xie2015diversifying,xie2015learning}, classification \citep{malkin2008ratio,jalali2015variational}, multitask learning \citep{jalali2015variational}.
Many works \citep{krogh1995neural,kuncheva2003measures,brown2005diversity,banfield2005ensemble,
tang2006analysis,partalas2008focused,yu2011diversity} explored how to select a diverse subset of base classifiers or regressors in ensemble learning, with the aim to improve generalization error and reduce computational complexity. Recently, \citep{Zou_priorsfor,xie2015diversifying,xie2015learning} studied the diversity-inducing regularization of latent variable models, which encourages the individual components in latent variable models to be different from each other, with the goal to capture long-tail knowledge and reduce model complexity. In a multi-class classification problem, \citet{malkin2008ratio} proposed to use the determinant of the covariance matrix to encourage classifiers to be different from each other. \citet{jalali2015variational} proposed a class of convex diversity regularizers and applied them for hierarchical classification and multi-task learning. While these works nicely demonstrate the effectiveness of diversity-promoting regularizers via empirical experiments, a rigorous theoretical analysis is missing. In this work, we aim to bridge this gap.

\paragraph*{Generalization Performance of Neural Networks}
The generalization performance of neural networks, in particular the approximation error and estimation error, has been widely studied in the past several decades. For the approximation error, \citet{cybenko1989approximation} demonstrated that finite linear combinations of compositions
of a fixed, univariate function and a set of affine functionals can uniformly
approximate any continuous function.
\citet{hornik1991approximation} showed that neural networks with a single hidden layer, sufficiently many hidden units and arbitrary bounded and nonconstant activation function are universal approximators. \citet{leshno1993multilayer} proved that multilayer feedforward networks with a non-polynomial activation function can approximate any function.
Various error rates have also been derived based on different assumptions of the target function. \citet{jones1992simple} showed that if the target function is in the hypothesis set formed by neural networks with one hidden layer of $m$ units, then the approximation error rate is $O(1/\sqrt{m})$.
\citet{barron1993universal} showed that neural networks with one layer of $m$ hidden units and sigmoid activation function can achieve approximation error of order $O(1/\sqrt{m})$, where the target function is assumed to have a bound on the first moment of the magnitude distribution of the Fourier transform. \citet{makovoz1998uniform} proved that if the target function is of the form $f(\mb{x})=\int_{Q}c(\mb{w},b)h(\mb{w}^{\mathsf{T}}\mb{x}+b)\mathrm{d}\mu$, where $c(\cdot,\cdot)\in L_{\infty}(Q,\mu)$, then neural networks with one layer of $m$ hidden units can approximate it with an error rate of $n^{-1/2-1/(2d)}\sqrt{\log n}$, where $d$ is the dimension of input $\mb{x}$.
As for the estimation error, please refer to \citep{anthony1999neural} for an extensive review, which introduces various estimation error bounds based on VC-dimension, flat-shattering dimension, pseudo dimension and so on.

\section{Latent Variable Modeling with Mutual Angular Regularization}
In this section, we begin with a review of latent variable models, then propose the mutual angular regularizer and utilize it to regularize LVMs. We present optimization techniques to learn mutual angle regularized LVMs.
\subsection{Latent Variable Models}
An LVM consists of two types of variables: the observed ones are utilized to model the observed data and the latent ones are used to characterize the hidden patterns. The interaction between observed and latent variables encodes the correlation between data and patterns. Under an LVM, extracting patterns from data corresponds to inferring the value of latent variables given the observed ones \citep{blei2014build}. For example, in topic models \citep{blei2003latent,hinton2009replicated} which are widely employed to extract topics from documents, the observed variables are used to model words and the latent variables are used to capture topics. The knowledge and structures hidden behind data are usually composed of multiple \textit{patterns}. For instance, the semantics underlying documents contains a set of \textit{themes} \citep{hofmann1999probabilistic,blei2003latent}, such as politics, economics and education.
Accordingly, latent variable models are parametrized by multiple \textit{components} where each component aims to capture one pattern in the knowledge and is represented with a parameter vector. For instance, the components in Latent Dirichlet Allocation \citep{blei2003latent} are called \textit{topics} and each topic is parametrized by a multinomial vector.

\subsection{Mutual Angular Regularizer}
\begin{figure}
\begin{center}
\includegraphics[width=0.4\textwidth]{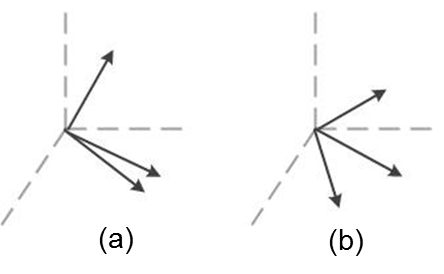}
\end{center}
\caption{The mean of pairwise angles between vectors in (a) is close to (b), but the variance of angles in (a) is much larger.
}
\label{fig:angle_variance}
\end{figure}
Motivated by the problems stated in Section \ref{sec:intro}, we propose to regularize the components in LVMs, to encourage them to diversely spread out. To measure the diversity of components, we define a mutual angular regularizer, which is a score bearing a larger value if the components are more diverse. Before quantifying the diversity of a set of components, we first measure the dissimilarity between two components (vectors). A good dissimilarity measure between two vectors would be invariant to scaling, translation, rotation and orientation of the two vectors. Commonly used metrics such as Euclidean distance, L1 distance, cosine similarity are not ideal since they are either sensitive to scaling or orientation. In this work, we utilize the non-obtuse angle $\theta$ as the dissimilarity measure of two vectors $\mb{x}$ and $\mb{y}$, which is defined
$\theta=\textrm{arccos}(\frac{|\mb{x}\cdot\mb{y}|}{\|\mb{x}\|\|\mb{y}\|})$. The non-obtuse angle differs from the ordinary definition of angle in that it is always acute or right, which is preferred due to its insensitiveness of vector orientation.
Given this pairwise dissimilarity measure, we can measure the diversity of a vector set. Let $\mb{A}\in\mathbb{R}^{K\times D}$ denote a set of $K$ components where each row of $\mb{A}$ is a $D$-dimensional component vector and the $K$ vectors are assumed to be linearly independent. 
We first take each pair of vectors $(\mb{a}_i,\mb{a}_j)$ from $\mb{A}$, compute their non-obtuse angle $\theta_{ij}$. Given all these pairwise angles, the mutual angular regularizer $\Omega(\mb{A})$ is defined as the mean of these angles minus the variance of these angles
\begin{equation}
\Omega(\mb{A})=\frac{1}{K(K-1)}\sum_{i=1}^{K}\sum_{j=1,j\neq i}^{K}\theta_{ij}-\gamma\frac{1}{K(K-1)}\sum_{i=1}^{K}\sum_{j=1,j\neq i}^{K}(\theta_{ij}-\frac{1}{K(K-1)}\sum_{p=1}^{K}\sum_{q=1,q\neq p}^{K}\theta_{pq})^{2}
\end{equation}
where $\gamma>0$ is a tradeoff parameter between mean and variance.
The mean term summarizes how these vectors are different from each other on the whole. A larger mean indicates these vectors share a larger angle in general, hence are more diverse. The variance term is utilized to encourage the vectors to evenly spread out to different directions. A smaller variance indicates that each vector is uniformly different from all other vectors. Encouraging the variance to be small can prevent the phenomenon that the vectors fall into several groups where vectors in the same group have small angles and vectors between groups have large angles. Such a phenomenon renders the vectors to be redundant and less diverse, and hence should be prohibited.
Figure \ref{fig:angle_variance} shows two set of vectors, where the mean of the pairwise angles of the first set (Figure \ref{fig:angle_variance}(a)) is roughly the same as that of the second set (Figure \ref{fig:angle_variance}(b)). But the variance of these angles are quite different in these two sets. In the first set, two vectors are very close to each other, while the third one is different from them. Hence the variance of the angles is large. In the second set, the vectors evenly point to different directions, hence the variance of angles is small. The first set has redundant vectors which is contradictory to diversity, hence it is desirable to prohibit such cases by encouraging the variance of the vectors to be small.

\subsection{Latent Variable Models with Mutual Angular Regularization}

Given this mutual angular regularizer, we can apply it to regularize latent variable models and control the geometry of the latent space during learning. Specifically it is employed to encourage the components to be diverse. Let $\mathcal{L}(\mb{A})$ denote model-specific objective function, such as likelihood (e.g., in topic models \citep{blei2003latent}, restricted Boltzmann machine \citep{hinton2010practical}), negative squared-loss (sparse coding \citep{olshausen1997sparse}), etc. Without loss of generality, we assume it is to be maximized. To diversify the components in the LVM, we augment the objective function with $\Omega(\mb{A})$ and define a mutual angle regularized LVM (MAR-LVM) problem as
\begin{equation}
\label{eq:p1}
\begin{array}{lll}
\mb{(P1)}&
\textrm{max}_{\mb{A}}\quad \mathcal{L}(\mb{A})+\lambda \Omega(\mb{A})
\end{array}
\end{equation}
where $\lambda$ is a tradeoff parameter. $\lambda$ plays an important role in balancing the diversity of model components and their fitness to data. Under a small $\lambda$, the components are learned to best fit data and their diversity is ignored. As discussed earlier, such components have high redundancy and may not be able to cover long-tail patterns effectively and are not amenable for optimization. Under a large $\lambda$, components are regularized to have high diversity, but may not fit well to data. 
To sum up, a proper $\lambda$ needs to be chosen to achieve the optimal balance.


\subsection{Optimization}
The mutual angular regularizer $\Omega(\mb{A})$ is non-smooth and non-convex, entailing great challenges for solving problem $\mb{(P1)}$ defined in Eq.(\ref{eq:p1}). In this section, we discuss how to address this issue. The basic strategy is to derive a smooth approximation $\Gamma(\mb{A})$ of $\Omega(\mb{A})$ and optimize $\Gamma(\mb{A})$ instead.
For reasons that will be clear later, we first reformulate $\mb{(P1)}$ by decomposing a component into its direction and magnitude. Let $\mb{A}=\textrm{diag}(\mb{g})\mb{\widetilde{A}}$, where $\mb{g}$ is a vector and $g_{i}$ denotes the $L_{2}$ norm of the $i$th row of $\mb{A}$, then the $L_{2}$ norm of each row vector $\mb{\tilde{a}}$ in $\mb{\widetilde{A}}$ is 1. Based on the definition of the mutual angular regularizer, we have $\Omega(\mb{A})=\Omega(\mb{\widetilde{A}})$. Accordingly, $\mb{(P1)}$ can be reformulated as
\begin{equation}
\label{eq:drbm2}
\begin{array}{lll}
\mb{(P2)}& \textrm{max}_{\mb{\widetilde{A}},\mb{g}}&\mathcal{L}(\textrm{diag}(\mb{g})\mb{\widetilde{A}})+\lambda\Omega(\mb{\widetilde{A}})\\
& s.t. & \forall i=1,\cdots, K, \|\mb{\tilde{a}}_{i}\|=1, g_i>0
\end{array}
\end{equation}
$\mb{(P2)}$ can be solved by alternating between $\mb{g}$ and $\mb{\widetilde{A}}$: optimizing $\mb{g}$ with $\mb{\widetilde{A}}$ fixed and optimizing $\mb{\widetilde{A}}$ with $\mb{g}$ fixed. With $\mb{\widetilde{A}}$ fixed, the problem defined over $\mb{g}$ is
\begin{equation}
\label{eq:drbm_g}
\begin{array}{ll}
\textrm{max}_{\mb{g}}&\mathcal{L}(\textrm{diag}(\mb{g})\mb{\widetilde{A}})\\
s.t.&\forall i=1,\cdots,K, g_i>0
\end{array}
\end{equation}
which can be efficiently solved with many optimization methods such as projected gradient descent \citep{boyd2004convex}, barrier method \citep{boyd2004convex}, etc. Fixing $\mb{g}$, the problem defined over $\mb{\widetilde{A}}$ is
\begin{equation}
\label{eq:drbma}
\begin{array}{ll}
\textrm{max}_{\mb{\widetilde{A}}}&\mathcal{L}(\textrm{diag}(\mb{g})\mb{\widetilde{A}})+\lambda\Omega(\mb{\widetilde{A}})\\
s.t. & \forall i=1,\cdots, K, \|\mb{\tilde{a}}_{i}\|=1
\end{array}
\end{equation}
which is still non-smooth and non-convex, entailing great obstacles for optimization. To address this issue, we derive a smooth lower bound $\Gamma(\mb{\widetilde{A}})$ of the regularizer and use the lower bound as a surrogate of $\Omega(\mb{\widetilde{A}})$ for optimization
\begin{equation}
\label{eq:drbma2}
\begin{array}{ll}
\textrm{max}_{\mb{\widetilde{A}}}&\mathcal{L}(\textrm{diag}(\mb{g})\mb{\widetilde{A}})+\lambda\Gamma(\mb{\widetilde{A}})\\
s.t. & \forall i=1,\cdots, K, \|\mb{\tilde{a}}_{i}\|=1
\end{array}
\end{equation}
Since $\Gamma(\mb{\widetilde{A}})$ is smooth, the optimization problem in Eq.(\ref{eq:drbma2}) is much easier than that in Eq.(\ref{eq:drbma}) and many algorithms can be applied, such as projected gradient descent \citep{boyd2004convex}. The overall algorithm is summarized in Algorithm \ref{opt}.
\begin{algorithm}[t]
\caption{Algorithm for solving MAR-LVM. }\label{opt} 
\begin{algorithmic}
\STATE {\bfseries Input:} $K$,$\lambda$
\REPEAT
\STATE Fixing $\mb{\widetilde{A}}$, learn $\mb{g}$ by solving the problem in Eq.(\ref{eq:drbm_g})
\STATE Fixing $\mb{g}$, learn $\mb{\widetilde{A}}$ by solving the problem in Eq.(\ref{eq:drbma2})
\UNTIL converge
\end{algorithmic}
\end{algorithm}
The lower bound is given in Theorem \ref{thm:lb}.

\begin{mythm}
\label{thm:lb}
Let $\det(\mb{\widetilde{A}}^{\mathsf{T}}\mb{\widetilde{A}})$ denote the determinant of the Gram matrix of $\mb{\widetilde{A}}$, then $0< \det(\mb{\widetilde{A}}^{\mathsf{T}}\mb{\widetilde{A}})\leq 1$. Let $\Gamma(\mb{\widetilde{A}})=\arcsin(\sqrt{\det(\mb{\widetilde{A}}^{\mathsf{T}}
\mb{\widetilde{A}})})-(\frac{\pi}{2}-\arcsin(\sqrt{\det(\mb{\widetilde{A}}^{\mathsf{T}}\mb{\widetilde{A}})}))^2$, then $\Gamma(\mb{\widetilde{A}})$ is a lower bound of $\Omega(\mb{\widetilde{A}})$. $\Gamma(\mb{\widetilde{A}})$ and $\Omega(\mb{\widetilde{A}})$ have the same global optimal.
\end{mythm}
\begin{proof}
To prove Theorem \ref{thm:lb}, the following lemma is needed\footnote{The proofs of lemmas utilized in this section are deferred to Appendix \ref{apd:opt}.}.
\begin{mylemma}
\label{lem:expand}
Let the parameter vector $\mb{\tilde{a}}_{i}$ of component $i$ be decomposed into $\mb{\tilde{a}}_{i}=\mb{x}_{i}+l_i\mb{e}_{i}$, where $\mb{x}_{i}=\sum_{j=1,j\neq i}^{K}\alpha_{j}\mb{\tilde{a}}_{j}$ lies in the subspace $L$ spanned by $\{\mb{\tilde{a}}_{1},\cdots,\mb{\tilde{a}}_{K}\}\backslash\{\mb{\tilde{a}}_{i}\}$, $\mb{e}_{i}$ is in the orthogonal complement of $L$, $\|\mb{e}_i\|=1$, $\mb{e}_{i}\cdot \mb{\tilde{a}}_{i}>0$, $l_i$ is a scalar. Then $\textrm{det}(\mathbf{\widetilde{A}}^\mathsf{T} \mathbf{\widetilde{A}}) = \textrm{det}(\mathbf{\widetilde{A}}_{-i}^\mathsf{T} \mathbf{\widetilde{A}}_{-i}) (l_i\mb{e}_{i}\cdot \mathbf{\tilde{a}}_i)$, where $\mb{\widetilde{A}}_{-i}=[\mb{\tilde{a}}_{1},\cdots,\mb{\tilde{a}}_{i-1},\mb{\tilde{a}}_{i+1},\cdots,\mb{\tilde{a}}_{K}]$ with $\mb{\tilde{a}}_{i}$ excluded..
\end{mylemma}

The mutual angular regularizer $\Omega(\mathbf{\widetilde{A}})$ comprises of two terms: $\Omega(\mathbf{\widetilde{A}}) = \Psi(\mathbf{\widetilde{A}}) - \Pi(\mathbf{\widetilde{A}})$, in which $\Psi(\mathbf{\widetilde{A}})$ and $\Pi(\mathbf{\widetilde{A}})$ measure the mean and variance of the pairwise angles respectively. We bound the two terms separately.
We first bound the mean $\Psi (\mathbf{\widetilde{A}})$.
Since the component vectors are assumed to be linearly independent, we have $\textrm{det}(\mathbf{\widetilde{A}}^\mathsf{T} \mathbf{\widetilde{A}})>0$.
As $l_i\mb{e}_{i}\cdot \mathbf{\tilde{a}}_i \le \|l_i\mb{e}_{i} \| \| \mathbf{\tilde{a}}_i\| \le 1$ and $\textrm{det}(\mathbf{\widetilde{A}}^\mathsf{T} \mathbf{\widetilde{A}}) = \textrm{det}(\mathbf{\widetilde{A}}_{-i}^\mathsf{T} \mathbf{\widetilde{A}}_{-i}) (l_i\mb{e}_{i}\cdot \mathbf{\tilde{a}}_i)$ (according to Lemma \ref{lem:expand}), we have
$\textrm{det}(\mathbf{\widetilde{A}}^\mathsf{T} \mathbf{\widetilde{A}}) \le \textrm{det}(\mathbf{\widetilde{A}}_{-i}^\mathsf{T} \mathbf{\widetilde{A}}_{-i})$. As $\forall j, \textrm{det}(\mathbf{\tilde{a}}_j^\mathsf{T} \mathbf{\tilde{a}}_j) = 1$, we can eliminate the columns of $\mathbf{\widetilde{A}}_{-i}$ and apply the inequality repeatedly to draw the conclusion that $\textrm{det}(\mathbf{\widetilde{A}}_{-i}^\mathsf{T} \mathbf{\widetilde{A}}_{-i}) \le 1$ (and $\textrm{det}(\mathbf{\widetilde{A}}^\mathsf{T} \mathbf{\widetilde{A}}) \le 1$). So $l_i\mb{e}_{i} \cdot \mathbf{\tilde{a}}_i =\|l_i\mb{e}_{i}\|^2 \ge \textrm{det}(\mathbf{\widetilde{A}}^\mathsf{T} \mathbf{\widetilde{A}})$.
For any $j \ne i$, the pairwise angle between $\mathbf{\tilde{a}}_i$ and $\mathbf{\tilde{a}}_j$ is:
\begin{equation}
\begin{array}{lll}
\theta(\mathbf{\tilde{a}}_i, \mathbf{\tilde{a}}_j)& =& \textrm{arccos} (|\mb{\tilde{a}}_i \cdot \mathbf{\tilde{a}}_j|) \\
&=& \textrm{arccos} (|\mathbf{x}_{i} \cdot \mathbf{\tilde{a}}_j|) \\
&\le& \textrm{arccos} (\| \mathbf{x}_{i} \| \| \mathbf{\tilde{a}}_j\|)\\
& = &\textrm{arccos} (\| \mathbf{x}_{i} \|)\\
&=& \textrm{arccos} (\sqrt{1 - \|l_i\mb{e}_{i}\|^2}) \\
&\le& \textrm{arccos} (\sqrt{1 - \textrm{det}(\mathbf{\widetilde{A}}^\mathsf{T} \mathbf{\widetilde{A}})}) \\
&=& \textrm{arcsin} (\sqrt{\textrm{det}(\mathbf{\widetilde{A}}^\mathsf{T} \mathbf{\widetilde{A}})})
\end{array}
\end{equation}
Thus $\Psi (\mathbf{\widetilde{A}}) \geq \textrm{arcsin} (\sqrt{\textrm{det}(\mathbf{\widetilde{A}}^\mathsf{T} \mathbf{\widetilde{A}})})$. Now we bound the variance $\Pi(\mathbf{\widetilde{A}})$. For any $i \ne j$, we have proved that $\theta(\mathbf{\tilde{a}}_i, \mathbf{\tilde{a}}_j) \ge \textrm{arcsin} (\sqrt{\textrm{det}(\mathbf{\widetilde{A}}^\mathsf{T} \mathbf{\widetilde{A}})})$. From the definition of $\theta(\mathbf{\tilde{a}}_i, \mathbf{\tilde{a}}_j)$, we also have $\theta(\mathbf{\tilde{a}}_i, \mathbf{\tilde{a}}_j) \le \frac{\pi}{2}$. As $\Psi (\mathbf{\widetilde{A}})$ is the mean value of all pairwise $\theta(\mathbf{\tilde{a}}_i, \mathbf{\tilde{a}}_j)$, we have $\textrm{arcsin} (\sqrt{\textrm{det}(\mathbf{\widetilde{A}}^\mathsf{T} \mathbf{\widetilde{A}})})
\leq \Psi (\mathbf{\widetilde{A}}) \leq \frac{\pi}{2}$. So $|\theta(\mathbf{\tilde{a}}_i, \mathbf{\tilde{a}}_j) - \Psi (\mathbf{\widetilde{A}})| \le \frac{\pi}{2} - \textrm{arcsin} (\sqrt{\textrm{det}(\mathbf{\widetilde{A}}^\mathsf{T} \mathbf{\widetilde{A}})})$. So $\Pi (\mathbf{\widetilde{A}}) \le (\frac{\pi}{2} - \textrm{arcsin} (\sqrt{\textrm{det}(\mathbf{\widetilde{A}}^\mathsf{T} \mathbf{\widetilde{A}})}))^2$.
Combining the lower bound of $\Psi (\mathbf{\widetilde{A}})$ and upper bound of $\Pi(\mathbf{\widetilde{A}})$, we have $\Omega(\mb{\widetilde{A}})\geq \Gamma(\mb{\widetilde{A}})= \textrm{arcsin} (\sqrt{\textrm{det}(\mathbf{\widetilde{A}}^\mathsf{T} \mathbf{\widetilde{A}})})-(\frac{\pi}{2}-\textrm{arcsin} (\sqrt{\textrm{det}(\mathbf{\widetilde{A}}^\mathsf{T} \mathbf{\widetilde{A}})}))^2$. Both $\Omega(\mb{\widetilde{A}})$ and $\Gamma(\mb{\widetilde{A}})$ obtain the optimal value of $\pi/2$ when vectors in $\mb{\widetilde{A}}$ are orthogonal to each other. The proof completes.
\end{proof}

Here we present an intuitive understanding of the lower bound. $\det(\mb{\widetilde{A}}^{\mathsf{T}}\mb{\widetilde{A}})$ is the volume of the parallelipiped formed by the vectors in $\mb{\widetilde{A}}$. The volume of a parallelipiped depends on both the length of vectors and the angles between vectors. Since vectors in $\mb{\widetilde{A}}$ are of unit-length, the angles determine the volume. The larger $\det(\mb{\widetilde{A}}^{\mathsf{T}}\mb{\widetilde{A}})$ is, the more likely\footnote{This is not for sure.} these vectors share larger angles. Let $g(x)=\arcsin(\sqrt{x})-(\frac{\pi}{2}-\arcsin(\sqrt{x}))^2$, which is increasing function, then $\Gamma(\mb{\widetilde{A}})=g(\det(\mb{\widetilde{A}}^{\mathsf{T}}\mb{\widetilde{A}}))$,
which is increasing w.r.t $\det(\mb{\widetilde{A}}^{\mathsf{T}}\mb{\widetilde{A}})$.
This implies that the larger $\Gamma(\mb{\widetilde{A}})$ is, the more likely the vectors in $\mb{\widetilde{A}}$ have larger angles, and accordingly the more likely that the mutual angular regularizer is larger.

While in general increasing a lower bound of a function cannot ensure the function itself is increased, in our case it can be proved that the monotonicity of $\Gamma(\mb{\widetilde{A}})$ is closely aligned with $\Omega(\mb{\widetilde{A}})$. Specifically, at each point, the gradient of $\Gamma(\mb{\widetilde{A}})$ is an ascent direction of $\Omega(\mb{\widetilde{A}})$, which is formally stated in Theorem \ref{thm:opt_main}. This property qualifies $\Gamma(\mb{\widetilde{A}})$ to be a desirable surrogate of $\Omega(\mb{\widetilde{A}})$.

\begin{mythm}
\label{thm:opt_main}
Let $\mb{G}$ be the gradient of $\Gamma(\mb{\widetilde{A}})$ w.r.t $\mb{\widetilde{A}}$. $\exists \tau>0$, such that $\forall \eta\in(0,\tau)$, $\Omega(\mathcal{P}(\mb{\widetilde{A}}+\eta \mb{G}))\geq \Omega(\mb{\widetilde{A}}) $ where $\mathcal{P}(\cdot)$ denotes the projection to the unit sphere.
\end{mythm}
To prove Theorem \ref{thm:opt_main}, we show that at each point the gradient of $\Gamma(\mb{\widetilde{A}})$ is an ascent direction of the mean of angles $\Psi(\mb{\widetilde{A}})$ and is a descent direction of the variance of angles $\Pi(\mb{\widetilde{A}})$, which are formally stated in Theorem \ref{thm:mean} and \ref{thm:variance}. Since $\Omega(\mb{\widetilde{A}})$ is the difference between $\Psi(\mb{\widetilde{A}})$ and $\Pi(\mb{\widetilde{A}})$, the gradient of $\Gamma(\mb{\widetilde{A}})$ is a descent direction of $\Omega(\mb{\widetilde{A}})$.

\begin{mythm}
\label{thm:mean}
Let $\mb{G}$ be the gradient of $\Gamma(\mb{\widetilde{A}})$ w.r.t $\mb{\widetilde{A}}$. $\exists \tau_1>0$, such that $\forall \eta\in(0,\tau_1)$, $\Psi(\mathcal{P}(\mb{\widetilde{A}}+\eta \mb{G}))\geq \Psi(\mb{\widetilde{A}}) $ where $\mathcal{P}(\cdot)$ denotes the projection to the unit sphere.
\end{mythm}

\begin{proof} 
We first introduce some notations.
Let $V = \{(i,j)|1\leq i,j\leq K, i\neq j, \mb{\tilde{a}}_{i}^{(t)} \cdot \mb{\tilde{a}}_{j}^{(t)} = 0\}$, $N = \{(i,j)|1\leq i,j\leq K, i\neq j,\mb{\tilde{a}}_{i}^{(t)} \cdot \mb{\tilde{a}}_{j}^{(t)} \neq 0 \}$, where $\mb{\tilde{a}}_{i}^{(t)}$ is the $i$th row of $\mb{A}^{(t)}$.
Let $\Delta V = \sum_{(i, j) \in V} (\theta(\mb{\tilde{a}}_{i}^{(t+1)}, \mb{\tilde{a}}_{j}^{(t+1)}) - \theta(\mb{\tilde{a}}_{i}^{(t)}, \mb{\tilde{a}}_{j}^{(t)}))$, $\Delta N = \sum_{(i, j) \in N} (\theta(\mb{\tilde{a}}_{i}^{(t+1)}, \mb{\tilde{a}}_{j}^{(t+1)}) - \theta(\mb{\tilde{a}}_{i}^{(t)}, \mb{\tilde{a}}_{j}^{(t)}))$, then $\Omega(\mb{\widetilde{A}}^{(t+1)}) - \Omega(\mb{\widetilde{A}}^{(t)}) = \Delta V + \Delta N$.
Let $x_{ij}^{(t)} = |\mb{\tilde{a}}_{i}^{(t)} \cdot \mb{\tilde{a}}_{j}^{(t)}|$, $x_{ij}^{(t+1)} =|\mb{\tilde{a}}_i^{(t+1)}\cdot \mb{\tilde{a}}_j^{(t+1)}|$. 

The following lemmas are needed to prove Theorem \ref{thm:mean}.
\begin{mylemma}
\label{lem:grad}
Let the parameter vector $\mb{\tilde{a}}_{i}$ of component $i$ be decomposed into $\mb{\tilde{a}}_{i}=\mb{x}_{i}+l_i\mb{e}_{i}$, where $\mb{x}_{i}=\sum_{j=1,j\neq i}^{K}\alpha_{j}\mb{\tilde{a}}_{j}$ lies in the subspace $L$ spanned by $\{\mb{\tilde{a}}_{1},\cdots,\mb{\tilde{a}}_{K}\}\backslash\{\mb{\tilde{a}}_{i}\}$, $\mb{e}_{i}$ is in the orthogonal complement of $L$, $\|\mb{e}_i\|=1$, $\mb{e}_{i}\cdot \mb{\tilde{a}}_{i}>0$, $l_i$ is a scalar. Then the gradient of $\Gamma(\mb{\widetilde{A}})$ w.r.t $\mb{a}_{i}$ is $k_{i}\mb{e}_{i}$, where $k_{i}$ is a positive scalar.
\end{mylemma}

\begin{mylemma}
\label{lem:m1}
$\forall (i,j) \in V$, we have $\theta(\mb{\tilde{a}}_{i}^{(t+1)}, \mb{\tilde{a}}_{j}^{(t+1)}) - \theta(\mb{\tilde{a}}_{i}^{(t)}, \mb{\tilde{a}}_{j}^{(t)}) = o(\eta)$, where $\lim\limits_{\eta \to 0} \frac{o(\eta)}{\eta} = 0$.
\end{mylemma}

\begin{mylemma}
\label{lem:m2}
$\forall (i,j) \in N$, $\exists c_{ij} > 0$, such that $\theta(\mb{\tilde{a}}_{i}^{(t+1)}, \mb{\tilde{a}}_{j}^{(t+1)}) - \theta(\mb{\tilde{a}}_{i}^{(t)}, \mb{\tilde{a}}_{j}^{(t)}) = c_{ij} \eta + o(\eta)$, where $\lim\limits_{\eta \to 0} \frac{o(\eta)}{\eta} = 0$.
\end{mylemma}

According to Lemma \ref{lem:grad}, $\mb{\tilde{a}}_{i}^{(t+1)}=\frac{\mb{\tilde{a}}^{(t)}_{i}+\eta k_{i}\mb{e}_{i}}{\|\mb{\tilde{a}}^{(t)}_{i}+\eta k_{i} \mb{e}_{i}\|} $, $\mb{\tilde{a}}_{j}^{(t+1)}=\frac{\mb{\tilde{a}}^{(t)}_{j}+\eta k_{j} \mb{e}_{j}}{\|\mb{\tilde{a}}^{(t)}_{j}+\eta k_{j} \mb{e}_{j}\|}$ and $\mb{e}_{i}\cdot \mb{\tilde{a}}_{j}^{(t)}=0$, $\mb{e}_{j}\cdot \mb{\tilde{a}}_{i}^{(t)}=0$. Since $\mb{\tilde{a}}_{i}=\mb{x}_{i}+l_i\mb{e}_{i}$, we have $\|\mb{\tilde{a}}^{(t)}_{i}+\eta k_{i} \mb{e}_{i}\|=\sqrt{1+2 l_i k_i\eta + k_i^2 \eta^2}$, and
$x_{ij}^{(t+1)}= \frac{| \mb{\tilde{a}}_i^{(t)}\cdot \mb{\tilde{a}}_j^{(t)} + \eta^2 k_i k_j \mb{e}_i^{(t)}\cdot \mb{e}_j^{(t)}|}{ \sqrt{1+2 l_i k_i\eta + k_i^2 \eta^2} \sqrt{1 + 2 l_j k_j\eta + k_j^2 \eta^2}}$. We can prove Theorem \ref{thm:mean} now.
\begin{equation}
\begin{array}{lll}
\Psi(\mb{\widetilde{A}}^{(t+1)}) - \Omega(\mb{\widetilde{A}}^{(t)})&=& \Delta V + \Delta N \\
&=& \sum_{(i,j)\in V} o(\eta) + \sum_{(i,j)\in N} (c_{ij}\eta + o(\eta)) \\
&=& o(\eta)+\sum_{(i,j)\in N} c_{ij}\eta
\end{array}
\end{equation}
$\lim\limits_{\eta \to 0} \frac{\Psi(\mb{\widetilde{A}}^{(t+1)}) - \Psi(\mb{\widetilde{A}}^{(t)})}{\eta} = \lim\limits_{\eta \to 0} \frac{o(\eta)+\sum_{(i,j)\in N} c_{ij}\eta }{\eta} = \sum_{(i,j)\in N} c_{ij}>0$.
So $\exists \tau > 0$ such that $\forall \eta \in (0,\tau)$ we have $\frac{\Psi(\mb{\widetilde{A}}^{(t+1)}) - \Psi(\mb{\widetilde{A}}^{(t)}) }{\eta} \ge \frac{1}{2}\sum_{(i,j)\in N} c_{ij}>0$. The proof completes.
\end{proof}

\begin{mythm}
\label{thm:variance}
Let $\mb{G}$ be the gradient of $\Gamma(\mb{\widetilde{A}})$ w.r.t $\mb{\widetilde{A}}$. $\exists \tau_2>0$, such that $\forall \eta\in(0,\tau_2)$, $\Pi(\mathcal{P}(\mb{\widetilde{A}}+\eta \mb{G}))\leq \Pi(\mb{\widetilde{A}}) $ where $\mathcal{P}(\cdot)$ denotes the projection to the unit sphere.
\end{mythm}


\begin{proof} To prove Theorem \ref{thm:variance}, we need the following lemma.
\begin{mylemma}
\label{lem:seq}
Given a nondecreasing sequence $b=(b_i)_{i = 1}^{n}$ and a strictly decreasing function $g(x)$ which satisfies
$
0 \le g(b_i) \le \textrm{min}\{b_{i+1} - b_i: i = 1, 2, \cdots, n-1, b_{i+1} \ne b_i\}
$,
we define a sequence $c=(c_i)_{i = 1}^{n}$ where $c_i = b_i + g(b_i)$. If $b_1 < b_n$, then $\textrm{var}(c) < \textrm{var}(b)$, where $var(\cdot)$ denotes the variance of a sequence. Furthermore, let $n' = \textrm{max}\{j: b_j \ne b_n\}$, we define a sequence $b'=(b'_i)_{i = 1}^{n}$ where $b_i' = b_i + g(b_n) + (g(b_{n'}) - g(b_n))\mathbb{I}(i \le n')$ and $\mathbb{I}(\cdot)$ is the indicator function, then $\textrm{var}(c) \le \textrm{var}(b') < \textrm{var}(b)$.
\end{mylemma}

The intuition behind the proof of Theorem \ref{thm:variance} is: when the stepsize $\eta$ is sufficiently small, we can make sure the changes of smaller angles (between consecutive iterations) are larger than the changes of larger angles, then Lemma \ref{lem:seq} can be used to prove that the variance decreases.
Let $\theta_{ij}^{(t)}$ denote $\theta(\mb{\tilde{a}}_{i}^{(t)}, \mb{\tilde{a}}_{j}^{(t)})$. We sort $\theta_{ij}^{(t)}$ in nondecreasing order and denote the resultant sequence as $\theta^{(t)}=(\theta^{(t)}_k)_{k=1}^{n}$, then $\textrm{var}((\theta_{ij}^{(t)})) = \textrm{var}(\theta^{(t)})$. We use the same order to index $\theta_{ij}^{(t+1)}$ and denote the resultant sequence as $\theta^{(t+1)}=(\theta_k^{(t+1)})_{k=1}^{n}$, then
$\textrm{var}((\theta_{ij}^{(t+1)}))=\textrm{var}(\theta^{(t+1)}) $. Let $g(\theta_{ij}^{(t)}) = \frac{2\textrm{cos}(\theta_{ij}^{(t)})}{\sqrt{1-{\textrm{cos}(\theta_{ij}^{(t)})}^2}}\eta$ if $\theta_{ij}^{(t)} < \frac{\pi}{2}$ and $0$ if $\theta_{ij}^{(t)} = \frac{\pi}{2}$, then $g(\theta_{ij}^{(t)})$ is a strictly decreasing function.
Let $\tilde{\theta}_{k}^{(t)} = \theta_{k}^{(t)}+ c_{k}\eta = \theta_{k}^{(t)}+g(\theta_{k}^{(t)})$. It is easy to see when $\eta$ is sufficiently small, $0 \le g(\theta_k^{(t)}) \le \textrm{min}\{\theta_{k+1}^{(t)}-\theta_k^{(t)}: k=1,2,\cdots,n-1, \theta_{k+1}^{(t)}\ne \theta_k^{(t)}\}$. We continue the proof from two complementary cases: (1) $\theta_1^{(t)} < \theta_n^{(t)}$; (2) $\theta_1^{(t)} = \theta_n^{(t)}$.
If $\theta_1^{(t)} < \theta_n^{(t)}$, then according to Lemma \ref{lem:seq}, we have $\textrm{var}(\tilde{\theta}^{(t)}) < \textrm{var}({\theta}^{(t)})$, where $\tilde{\theta}^{(t)}=(\tilde{\theta}^{(t)}_{k})_{k=1}^{n}$.
Furthermore, let $n'=\textrm{max}\{j:\theta_j^{(t)} \ne \theta_n^{(t)}\}$, $\theta_k'^{(t)} = \theta_k^{(t)} + g(\theta_n^{(t)})+(g(\theta_{n'}^{(t)})-g(\theta_n^{(t)}))\mathbb{I}(k\le n')$, then $\textrm{var}(\tilde{\theta}^{(t)}) \le \textrm{var}(\theta'^{(t)}) < \textrm{var}(\theta^{(t)})$, where $\theta'^{(t)}=(\theta'^{(t)}_{k})_{k=1}^{n}$. $\textrm{var}(\theta'^{(t)})$ can be written as:
\begin{equation}
\begin{array}{lll}
&=&\frac{1}{n} \sum_{i=1}^n (\theta_i'^{(t)} - \frac{1}{n} \sum_{j=1}^n \theta_j'^{(t)})^2 \\
&=& \frac{1}{n} \sum_{i=1}^n (\theta_i^{(t)} + (g(\theta_{n'}^{(t)}) - g(\theta_n^{(t)}))\mathbb{I}(i\le n') - \frac{1}{n} \sum_{j=1}^n \theta_j^{(t)}- \frac{n'}{n}(g(\theta_{n'}^{(t)}) - g(\theta_n^{(t)})) )^2 \\
&= &\textrm{var}(\theta^{(t)}) + 2(\frac{1}{n} \sum_{i=1}^n (\theta_i^{(t)} -\frac{1}{n} \sum_{j=1}^n \theta_j^{(t)})(g(\theta_{n'}^{(t)}) - g(\theta_n^{(t)}))(\mathbb{I}(i\le n') - \frac{n'}{n}) \\
&&+ \frac{1}{n} \sum_{j=1}^n (g(\theta_{n'}^{(t)})
- g(\theta_n^{(t)}))^2 (\mathbb{I}(i\le n') - \frac{n'}{n})^2
\end{array}
\end{equation}

Let
$\lambda = 2(\frac{1}{n} \sum_{i=1}^n (\theta_i^{(t)} - \frac{1}{n} \sum_{j=1}^n \theta_j^{(t)})
(\mathbb{I}(i\le n') - \frac{n'}{n})$, it can be further written as
\begin{equation}
\begin{array}{lll}
&=&\frac{2}{n}(\sum_{i=1}^{n'} (\theta_i^{(t)} - \frac{1}{n} \sum_{j=1}^n \theta_j^{(t)})(1 - \frac{n'}{n}) + \frac{2}{n}(\sum_{i=n'+1}^n (\theta_i^{(t)} - \frac{1}{n} \sum_{j=1}^n \theta_j^{(t)})
( - \frac{n'}{n}) \\
&=& \frac{2}{n}((\sum_{i=1}^{n'}\theta_i^{(t)} - \frac{n'}{n} \sum_{j=1}^n
\theta_j^{(t)})(1 - \frac{n'}{n}) +
\frac{2}{n}
( (\sum_{i=n'+1}^n
\theta_i^{(t)} - \frac{n-n'}{n} \sum_{j=1}^n \theta_j^{(t)})( - \frac{n'}{n}) \\
&=& \frac{2n'(n-n')}{n^2}(\frac{1}{n'}\sum_{i=1}^{n'}\theta_i^{(t)}
-\frac{1}{n-n'}
\sum_{i=n'+1}^n \theta_i^{(t)})
\end{array}
\end{equation}
As $\theta_k^{(t)}$ is nondecreasing and $\theta_n^{(t)} \ne \theta_{n'}^{(t)}$, we have $\lambda < 0$.
Let $\mu = \frac{2\textrm{cos}(\theta_{n'}^{(t)})}{\sqrt{1-{\textrm{cos}(\theta_{n'}^{(t)})}^2}}
- \frac{2\textrm{cos}(\theta_n^{(t)})}{\sqrt{1-{\textrm{cos}(\theta_n^{(t)})}^2}}$ when $\theta_n^{(t)} < \frac{\pi}{2}$ and $\mu = \frac{2\textrm{cos}(\theta_{n'}^{(t)})}{\sqrt{1-{\textrm{cos}(\theta_{n'}^{(t)})}^2}}$ when $\theta_n^{(t)}=\frac{\pi}{2}$, then $g(\theta_{n'}^{(t)}) - g(\theta_n^{(t)}) = \mu \eta$ and $\mu > 0$. Substituting $\lambda$ and $\mu$ into $\textrm{var}(\theta'^{(t)})$, we can obtain:
\begin{equation*}
\begin{array}{lll}
\textrm{var}(\theta'^{(t)}) &=& \textrm{var} (\theta^{(t)}) + \lambda \mu \eta + \frac{1}{n} \sum_{j=1}^n (\mathbb{I}(i\le n') - \frac{n'}{n})^2 \mu^2 \eta^2 \\
&=& \textrm{var}(\theta^{(t)}) + \lambda \mu \eta + o(\eta)
\end{array}
\end{equation*}
Note that $\lambda < 0$ and $\mu > 0$, so $\exists \delta_1$, such that $\eta < \delta_1 \Rightarrow \textrm{var}(\theta'^{(t)})< \textrm{var}(\theta^{(t)}) + \frac{\lambda \mu}{2} \eta$. As $\textrm{var}(\tilde{\theta}^{(t)}) < \textrm{var} (\theta'^{(t)})$, we can draw the conclusion that $\textrm{var}(\tilde{\theta}^{(t)}) < \textrm{var}(\theta^{(t)}) + \frac{\lambda \mu}{2} \eta$. On the other hand,
\begin{equation*}
\begin{array}{lll}
\textrm{var}(\theta^{(t+1)}) &=& \frac{1}{n}\sum_{i=1}^n(\theta_i^{(t+1)} - \frac{1}{n}\sum_{j=1}^n \theta_j^{(t+1)})^2 \\
&=& \frac{1}{n}\sum_{i=1}^n(\tilde{\theta}_i^{(t)} + o(\eta) - \frac{1}{n}\sum_{j=1}^n \tilde{\theta}_j^{(t)} + o(\eta))^2 \\
&=& \frac{1}{n}\sum_{i=1}^n(\tilde{\theta}_i^{(t)} - \frac{1}{n}\sum_{j=1}^n \tilde{\theta}_j^{(t)})^2 + o(\eta)\\
&=&\textrm{var}(\tilde{\theta}^{(t)}) + o(\eta)
\end{array}
\end{equation*}
So $\exists \delta_2 >0$ such that
$\eta < \delta_2 \Rightarrow \textrm{var}(\theta^{(t+1)})< \textrm{var}(\tilde{\theta}^{(t)}) - \frac{\lambda \mu}{4} \eta
$.
Let $\delta = \textrm{min}\{\delta_1,\delta_2\}$, then
\begin{equation*}
\begin{array}{lll}
\eta<\delta
&\Rightarrow& \textrm{var}(\theta^{(t+1)})
<\textrm{var}(\theta^{(t)}) +\frac{\lambda \mu}{4} \eta
< \textrm{var}(\theta^{(t)})\\
&\Rightarrow& \textrm{var}((\theta^{(t+1)})) < \textrm{var}(\theta^{(t)})
\end{array}
\end{equation*}
For the second case $\theta_1^{(t)} = \theta_n^{(t)}$, i.e.,
$\forall (i_1,j_1), (i_2,j_2)\in N \cup V$, $\theta_{i_1 j_1}^{(t)} = \theta_{i_2 j_2}^{(t)}$, we prove that $\textrm{var}(\theta^{(t+1)}) = \textrm{var}(\theta^{(t)})$. In this case, $\forall (i_1,j_1), (i_2,j_2) \in N \cup V$, $((\mb{\widetilde{A}}^{(t)})^\mathsf{T} \mb{\widetilde{A}}^{(t)})_{i_1 j_1} = \mb{\tilde{a}}^{(t)}_{i_1} \cdot\mb{\tilde{a}}^{(t)}_{j_1} = \mb{\tilde{a}}^{(t)}_{i_2} \cdot\mb{\tilde{a}}^{(t)}_{j_2} = ((\mb{\widetilde{A}}^{(t)})^\mathsf{T} \mb{\widetilde{A}}^{(t)})_{i_2 j_2}$.
Denote $p_1 = \mb{\tilde{a}}^{(t)}_{i} \cdot\mb{\tilde{a}}^{(t)}_{j}$ for $i \ne j$ and $p_2 = \mb{\tilde{a}}^{(t)}_{i} \cdot\mb{\tilde{a}}^{(t)}_{j}$ for $i=j$.
As $\mb{\widetilde{A}}^{(t+1)} = \mb{\widetilde{A}}^{(t)} + c\mb{\widetilde{A}}^{(t)} ((\mb{\widetilde{A}}^{(t)})^\mathsf{T} \mb{\widetilde{A}}^{(t)})^{-1}$, where $c=2\eta g'(\det((\mb{\widetilde{A}}^{(t)})^\mathsf{T} \mb{\widetilde{A}}^{(t)}))
\det((\mb{\widetilde{A}}^{(t)})^\mathsf{T} \mb{\widetilde{A}}^{(t)})$ and $g(x)=\arcsin(\sqrt{x})-(\frac{\pi}{2}-\arcsin(\sqrt{x}))^2$,
we have $(\mb{\widetilde{A}}^{(t+1)})^\mathsf{T} \mb{\widetilde{A}}^{(t+1)} = (\mb{\widetilde{A}}^{(t)})^\mathsf{T} \mb{\widetilde{A}}^{(t)} + 2c \mb{I} + c^2 ((\mb{\widetilde{A}}^{(t)})^\mathsf{T} \mb{\widetilde{A}}^{(t)})^{-1}$.
It is clear that $\forall (i_1,j_1), (i_2,j_2) \in N \cup V$, $((\mb{\widetilde{A}}^{(t)})^\mathsf{T} \mb{\widetilde{A}}^{(t)} + 2c \mb{I} )_{i_1 j_1} = ((\mb{\widetilde{A}}^{(t)})^\mathsf{T} \mb{\widetilde{A}}^{(t)} + 2c \mb{I} )_{i_2 j_2}$. For $c^2 ((\mb{\widetilde{A}}^{(t)})^\mathsf{T} \mb{\widetilde{A}}^{(t)})^{-1}$, write it as $c^2 ((p_2 - p_1)\mb{I}_K + p_1 \mb{1}_K \mb{1}_K^\mathsf{T})^{-1}$, where $\mb{I}_K$ is the identity matrix and $\mb{1}_K$ is a vector of $1$s whose length is $K$. Applying Sherman-Morrison formula, we can obtain that $((\mb{\widetilde{A}}^{(t)})^\mathsf{T} \mb{\widetilde{A}}^{(t)})^{-1}=
((p_2-p_1)^{-1}\mb{I}_K - \frac{(p_2-p_1)^{-1}\mb{1}_K \mb{1}_K^T}{1+K(p_2-p_1)})$
which implies that $\forall (i_1,j_1), (i_2,j_2) \in N \cup V$, $((\mb{\widetilde{A}}^{(t)})^\mathsf{T} \mb{\widetilde{A}}^{(t)})^{-1}_{i_1 j_1} = ((\mb{\widetilde{A}}^{(t)})^\mathsf{T} \mb{\widetilde{A}}^{(t)})^{-1}_{i_2 j_2}$, so $((\mb{\widetilde{A}}^{(t+1)})^\mathsf{T} \mb{\widetilde{A}}^{(t+1)})_{i_1 j_1} = ((\mb{\widetilde{A}}^{(t+1)})^\mathsf{T} \mb{\widetilde{A}}^{(t+1)})_{i_2 j_2}$, so $\textrm{var}(\theta^{(t+1)}) = 0 = \textrm{var}(\theta^{(t)})$.

Putting these two cases together, we conclude that $\exists \tau_2>0$, such that $\forall \eta\in(0,\tau_2)$, $\Pi(\mb{\widetilde{A}}^{(t+1)})\leq \Pi(\mb{\widetilde{A}}^{(t)})$.
\end{proof}

\section{Case Studies}
In this section, we instantiate the general framework of mutual angle regularized LVMs (MAR-LVMs) to three specific latent variable models: restricted Boltzmann machine, distance metric learning and neural networks, which are utilized to carry out the theoretical and empirical study later on.
\subsection{Restricted Boltzmann Machine}
\begin{figure}[t]
\begin{center}
\centerline{\includegraphics[width=0.5\columnwidth]{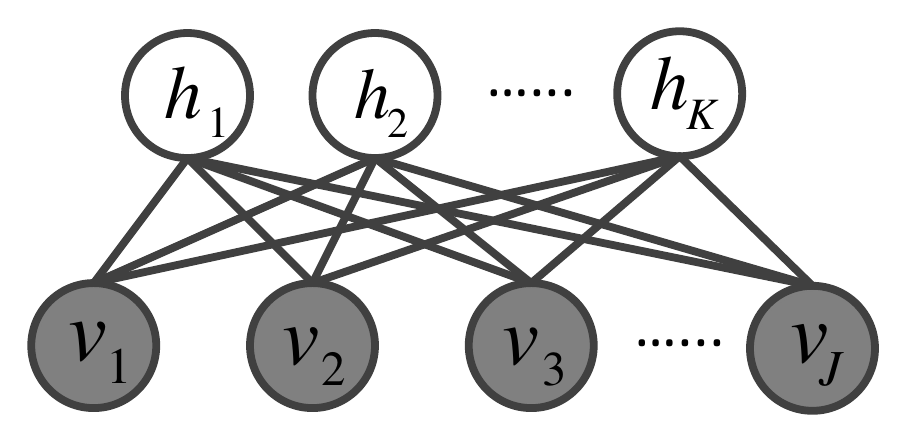}
}
\caption{Restricted Boltzmann Machine}
\label{fig:rbm}
\end{center}
\end{figure}
Restricted Boltzmann Machine (RBM) \citep{smolensky1986information,hinton2009replicated,srivastava2013modeling} is a primary model for representation learning and has been widely utilized in document modeling, information retrieval, collaborative filtering, etc. RBM is a two-layer undirected graphical model (Figure \ref{fig:rbm}) consisting of a layer of hidden units $\mb{h}=\{h_{k}\}_{k=1}^{K}$ used to capture latent features and a layer of visible units $\mb{v}=\{v_{j}\}_{j=1}^{D}$ utilized to represent observed data. Each hidden unit is connected with all visible units with undirected and weighted edges. The energy function defined over $\mb{h}$ and $\mb{v}$ is
\begin{equation}
E(\mb{h},\mb{v})=-\sum_{j=1}^{D}\alpha_{j}v_{j}-\sum_{k=1}^{K}
\beta_{k}h_{k}-\sum\limits_{k=1}^{K}\sum\limits_{j=1}^{D}A_{kj}h_{k}v_{j}
\end{equation}
where $\boldsymbol\alpha=\{\alpha_{j}\}_{j=1}^{J}$ and $\boldsymbol\beta=\{\beta_{j}\}_{j=1}^{D}$ are the biases associated with the visible and hidden units respectively. $\mb{A}\in \mathbb{R}^{K\times D}$ contain the weights on the edges connecting two set of units. The $k$-th row of $\mb{A}$ is the weight vector corresponding to hidden unit $k$. The energy function induces a probability distribution $p(\mb{h},\mb{v})=\exp(-E(\mb{h},\mb{v}))/Z$ where $Z$ is the partition function. The model parameters $\boldsymbol\alpha$, $\boldsymbol\beta$, $\mb{A}$ can be learned by maximizing data likelihood using gradient ascent method where the gradient is approximately computed via the contrastive divergence \citep{hinton2002training} method.

As stated early, the popularity of latent patterns underlying data is in a power-law distribution: a few dominant patterns appear very frequently while those in the long-tail region are of low frequency. The standard RBM tends to learn many redundant hidden units to cover the dominant patterns and ignore the long-tail patterns. To address this issue, we use the mutual angular regularizer to diversify the hidden units in RBM, encouraging them to diversely spread out so that the long-tail patterns get a better chance to be covered. A mutual angle regularized RBM (MAR-RBM) can be defined as follows
\begin{equation}
\begin{array}{ll}
\textrm{max}& \sum_{i=1}^{N}\log p(\mb{v}_i;\boldsymbol\alpha,\boldsymbol\beta,\mb{A})+\lambda \Omega(\mb{A})
\end{array}
\end{equation}
where $\sum_{i=1}^{N}\log p(\mb{v}_i;\boldsymbol\alpha,\boldsymbol\beta,\mb{A})$ is the log-likelihood of RBM and $\Omega(\mb{A})$ is the mutual angular regularizer encouraging the weight vectors of hidden units to be different from each other.

\subsection{Distance Metric Learning}
A proper distance measure is of vital importance for many distance based tasks and applications, such as retrieval \citep{zhang2014supervised}, clustering \citep{xing2002distance} and classification \citep{weinberger2005distance}. Distance metric learning (DML) \citep{xing2002distance,weinberger2005distance,davis2007information} takes pairs of data points which are labeled either as similar or dissimilar and learns a distance metric such that similar data pairs will be placed close to each other while dissimilar pairs will be separated apart. DML can be formulated as a latent variable model, which learns a projection matrix $\mb{A}\in\mathbb{R}^{K\times D}$ to project the data pair $(\mb{x},\mb{y})$ from the original $D$-dimensional feature space to a $K$-dimensional latent space and measure the squared Euclidean distance $\|\mb{A}\mb{x}-\mb{A}\mb{y}\|_2^2$ therein. Each row of $\mb{A}$ corresponds to a latent component (or one dimension of the latent space). Given data pairs labeled as similar $\mathcal{S}=\{(\mb{x}_{i},\mb{y}_{i})\}_{i=1}^{|\mathcal{S}|}$ and dissimilar $\mathcal{D}=\{(\mb{x}_{i},\mb{y}_{i})\}_{i=1}^{|\mathcal{D}|}$, DML learns the projection matrix $\mb{A}$ by minimizing the distance of similar pairs while separating dissimilar pairs apart with a certain margin
\begin{equation}
\label{eq:dml_opt_2}
\begin{array}{ll}
\textrm{min}_{\mb{A}}&\frac{1}{|\mathcal{S}|}\sum\limits_{(\mb{x},\mb{y})\in \mathcal{S}}\|\mb{Ax}-\mb{Ay}\|^{2}\\
s.t.&\|\mb{Ax}-\mb{Ay}\|^{2}\geq1, \forall (\mb{x},\mb{y})\in\mathcal{D}\\
\end{array}
\end{equation}

In choosing the number $K$ of latent components, there is an inherent tradeoff between the effectiveness of the distance matrix $\mb{A}$ and computational efficiency. A larger $K$ would bestow $\mb{A}$ more expressiveness and power in measuring distances. However, the resultant latent representations would be of high dimensionality, which incurs high computational complexity and inefficiency. This is especially true for retrieval where performing nearest neighbor search on high-dimensional representations is largely difficult. On the other hand, while a smaller $K$ can reduce the computational cost, it would render the distance matrix less effective.

To address this dilemma, we utilize the mutual angular regularizer to learn a compact but effective distance matrix: $K$ is reduced to facilitate computational efficiency without sacrificing the effectiveness of measuring distances. The MAR encourages the components in $\mb{A}$ to have larger mutual angles, which renders the components to be less correlated and each component captures information that cannot be captured by others. As a result, a small number of components are sufficient to capture a large proportion of information. The mutual angle regularized DML (MAR-DML) problem can be defined as follows
\begin{equation}
\label{eq:mar_dml}
\begin{array}{ll}
\textrm{min}_{\mb{A}}&\frac{1}{|\mathcal{S}|}\sum\limits_{(\mb{x},\mb{y})\in \mathcal{S}}\|\mb{Ax}-\mb{Ay}\|^{2}-\lambda\Omega(\mb{A})\\
s.t.&\|\mb{Ax}-\mb{Ay}\|^{2}\geq1, \forall (\mb{x},\mb{y})\in\mathcal{D}\\
\end{array}
\end{equation}

\subsection{Neural Network with Mutual Angular Regularization}

Neural networks (NNs) have shown great success in many applications, such as speech recognition \citep{hinton2012deep}, image classification \citep{krizhevsky2012imagenet}, machine translation \citep{bahdanau2014neural}, etc. NNs are composed of multiple layers of computing units and units in adjacent layers are connected with weighted edges. NNs are a typical type of LVMs where each hidden unit is a component aiming to capture the latent features underlying data and is characterized by a vector of weights connecting to units in the lower layer.

We instantiate the general framework of diversity-regularized LVM to neural network and utilize the mutual angular regularizer to encourage the hidden units (precisely their weight vectors) to be different from each other, which could lead to several benefits: (1) better capturing of long-tail latent features; (2) reducing the size of NN without compromising modeling power.
Let $\mathcal{L}(\{\mb{A}_{i}\}_{i=0}^{l-1})$ be the loss function of a neural network with $l$ layers where $\mb{A}_{i}$ are the weights between layer $i$ and layer $i+1$, and each row of $\mb{A}_{i}$ corresponds to a unit. A diversified NN with mutual angular regularization (MAR-NN) can be defined as
\begin{equation}
\begin{array}{ll}
\textrm{min}_{\{\mb{A}_{i}\}_{i=0}^{l-1}}&\mathcal{L}(\{\mb{A}_{i}\}_{i=0}^{l-1})-\lambda \sum_{i=0}^{l-2}\Omega(\mb{A}_{i})
\end{array}
\end{equation}
where $\Omega(\mb{A}_{i})$ is the mutual angular regularizer and $\lambda>0$ is a tradeoff parameter. Note that the regularizer is not applied to $\mb{A}_{l-1}$ since in the last layer are output units which are not latent components.

\section{Analysis}
In this section, we aim to formally understand why and how introducing diversity into LVMs can lead to better modeling effects. Specifically, we analyze how the mutual angular regularizer affects the generalization performance of supervised latent variable models. We choose neural network (NN) as a model instance to carry out the analysis while noting that the analysis could be extended to other LVMs such as restricted Boltzmann machine and Distance Metric Learning.

Let $L(f)=\mathbb{E}_{(\mb{x},y)\sim p^*}[\ell(f(\mb{x}),y)]$ denote the generalization error of hypothesis $f$, where $p^*$ is the distribution of input-output pair $(\mb{x},y)$ and $\ell(\cdot)$ is the loss function. Let $f^*\in \textrm{argmin}_{f\in \mathcal{F}}L(f)$ be the expected risk minimizer. Let $\hat{L}(f)=\frac{1}{n}\sum_{i=1}^{n}\ell(f(\mb{x}^{(i)}),y^{(i)})$ be the training error and $\hat{f}\in \textrm{argmin}_{f\in \mathcal{F}}\hat{L}(f)$ be the empirical risk minimizer. We are interested in the generalization error $L(\hat{f})$ of the empirical risk minimizer $\hat{f}$, which can be decomposed into two parts $L(\hat{f})=L(\hat{f})-L(f^*)+L(f^*)$, where $L(\hat{f})-L(f^*)$ is the estimation error (or excess risk) and $L(f^*)$ is the approximation error. The estimation error represents how well the algorithm is able to learn and usually depends on the complexity of the hypothesis and the number of training samples. A lower hypothesis complexity and a larger amount of training data incur lower estimation error bound. The approximation error indicates how expressive the hypothesis set is to effectively approximate the target function.

Our analysis below shows that the mutual angular regularizer can reduce the generalization error of neural networks. We assume with high probability $\tau$, the angle between each pair of hidden units is lower bounded by $\theta$. $\theta$ is a formal characterization of diversity. The larger $\theta$ is, the more diverse these hidden units are. The analysis in the following sections suggests that $\theta$ incurs a tradeoff between estimation error and approximation error: the larger $\theta$ is, the smaller the estimation error bound is and the larger the approximation error bound is. Since the generalization error is the sum of estimation error and approximation error, $\theta$ has an optimal value to yield the minimal generalization error. In addition, we can show that under the same probability $\tau$, increasing the mutual angular regularizer can increase $\theta$. Given a set of hidden units $\mb{A}$ learned by the MAR-NN, we assume their pairwise angles $\{\theta_{ij}\}$ are $i.i.d$ samples drawn from a distribution $p(X)$ where the expectation and variance of random variable $X$ is $\mu$ and $\sigma$ respectively. Lemma \ref{lem:mar_bd} states that $\theta$ is an increasing function of $\mu$ and decreasing function of $\sigma$. By the definition of MAR, it encourages larger mean and smaller variance. Thereby, the larger the MAR is, the larger $\theta$ is. Hence properly controlling the MAR can generate a desired $\theta$ that produces the lowest generalization error.

\begin{mylemma}
\label{lem:mar_bd} With probability at least $\tau$, we have
$X\geq \theta=\mu-\sqrt{\frac{\sigma}{1-\tau}}$
\begin{proof}
According to Chebyshev inequality \citep{wasserman2013all},
\begin{equation}
\frac{\sigma}{t^2}\geq p(|X-\mu|>t)\geq p(X<\mu-t)
\end{equation}
Let $\theta=\mu-t$, then $p(X<\theta)\leq \frac{\sigma}{(\mu-\theta)^2}$. Hence $p(X\geq\theta)\geq 1-\frac{\sigma}{(\mu-\theta)^2}$. Let $\tau=1-\frac{\sigma}{(\mu-\theta)^2}$, then $\theta=\mu-\sqrt{\frac{\sigma}{1-\tau}}$.
\end{proof}
\end{mylemma}

\subsection{Setup}
\label{sec:setup}
For the ease of presentation, we first consider a simple neural network whose setup is described below. Later on we extend the analysis to more complicated neural networks.
\begin{itemize}
\item Network structure: one input layer, one hidden layer and one output layer
\item Activation function: Lipschitz continuous function $h(t)$ with constant $L$. Examples: rectified linear $h(t)=\textrm{max}(0,t)$, $L=1$; tanh $h(t)=\textrm{tanh}(t)$, $L=1$; sigmoid $h(t)=\textrm{sigmoid}(t)$, $L=0.25$.
\item Task: univariate regression
\item Let $\mb{x}\in \mathbb{R}^d$ be the input vector with $\|\mb{x}\|_2\leq C_1$
\item Let $y$ be the response value with $|y|\leq C_2$
\item Let $\mb{w_j}\in \mathbb{R}^d$ be the weights connecting to the $j$-th hidden unit, $j=1,\cdots,m$, with $\|\mb{w_j}\|_2\leq C_3$. Further, we assume with high probability $\tau$, the angle $\rho(\mb{w_i},\mb{w_j})=\textrm{arccos}(\frac{|\mb{w_i}\cdot\mb{w_j}|}{\|\mb{w_i}\|_2\|\mb{w_j}\|_2})$ between $\mb{w_i}$ and $\mb{w_j}$ is lower bounded by a constant $\theta$ for all $i\neq j$.
\item Let $\alpha_j$ be the weight connecting the hidden unit $j$ to the output with $\|\mb{\alpha}\|_2\leq C_4$
\item Hypothesis set: $\mathcal{F}=\{f|f(\mb{x})=\sum\limits_{j=1}^{m}\alpha_j h(\mb{w_j}^{\mathsf{T}}\mb{x})\}$
\item Loss function set: $\mathcal{A}=\{\ell|\ell(f(\mb{x}),y)=(f(\mb{x})-y)^2\}$
\end{itemize}

\subsection{Estimation Error}
We first analyze the estimation error bound of MAR-NN and are interested in how the upper bound is related with the diversity (measured by $\theta$) of the hidden units. The major result is presented in Theorem \ref{thm:est_err}.
\begin{mythm}
\label{thm:est_err}
With probability at least $(1-\delta)\tau$
\begin{equation}
\label{eq:est_err}
\begin{array}{lll}
&&L(\hat{f})-L(f^*)\leq 8(\sqrt{\mathcal{J}}+C_2)(2LC_1C_3C_4+C_4|h(0)|)\frac{\sqrt{m}}{\sqrt{n}}+ (\sqrt{\mathcal{J}}+C_2)^2\sqrt{\frac{2\log(2/\delta)}{n}}
\end{array}
\end{equation}
where $\mathcal{J}=mC_4^2h^2(0)+L^2C_1^2C_3^2C_4^2((m-1)\cos\theta+1) + 2\sqrt{m}C_1C_3C_4^2L|h(0)|\sqrt{(m-1)\cos\theta+1}$.
\end{mythm}

Note that the right hand side is a decreasing function w.r.t $\theta$. A larger $\theta$ (denoting the hidden units are more diverse) would induce a lower estimation error bound.

\subsubsection{Proof}
\label{sec:thm1_proof}
A well established result in learning theory is that the estimation error can be upper bounded by the Rademacher complexity. We start from the Rademacher complexity, seek a further upper bound of it and show how the diversity of the hidden units affects this upper bound.
The Rademacher complexity $\mathcal{R}_n(\mathcal{A})$ of the loss function set $\mathcal{A}$ is defined as
\begin{equation}
\begin{array}{lll}
\mathcal{R}_n(\mathcal{A})&=&\mathbb{E}[\textrm{sup}_{\ell\in \mathcal{A}}\frac{1}{n}\sum_{i=1}^{n}\sigma_i \ell(f(\mb{x}^{(i)}), y^{(i)})]\\
\end{array}
\end{equation}
where $\sigma_i$ is uniform over $\{-1,1\}$ and $\{(\mb{x}^{(i)}, y^{(i)})\}_{i=1}^{n}$ are i.i.d samples drawn from $p^*$. The Rademacher complexity can be utilized to upper bound the estimation error, as shown in Lemma \ref{lem:rc_bd}\footnote{The proofs of lemmas utilized in this section are deferred to Appendix \ref{apd:est}.}.
\begin{mylemma}\citep{anthony1999neural,bartlett2003rademacher,liang2015lecture}
\label{lem:rc_bd}
With probability at least $1-\delta$
\begin{equation}
L(\hat{f})-L(f^*)\leq 4\mathcal{R}_n(\mathcal{A})+B\sqrt{\frac{2\log(2/\delta)}{n}}
\end{equation}
for $B \ge \sup_{\mb{x}, y, f}|\ell(f(\mb{x}), y)|$
\end{mylemma}
Our analysis starts from this lemma and we seek further upper bound of $\mathcal{R}_n(\mathcal{A})$. The analysis needs an upper bound of the Rademacher complexity of the hypothesis set $\mathcal{F}$, which is given in Lemma \ref{lem:rc_f}.
\begin{mylemma}
\label{lem:rc_f}
Let $\mathcal{R}_n(\mathcal{F})$ denote the Rademacher complexity of the hypothesis set $\mathcal{F}=\{f|f(\mb{x})=\sum\limits_{j=1}^{m}\alpha_j h(\mb{w_j}^{\mathsf{T}}\mb{x})\}$, then
\begin{equation}
\mathcal{R}_n(\mathcal{F})\leq \frac{2 L C_{1}C_{3}C_{4}\sqrt{m}}{\sqrt{n}} + \frac{C_{4}|h(0)|\sqrt{m}}{\sqrt{n}}
\end{equation}
\end{mylemma}

In addition, we need the following bound of $|f(\mb{x})|$.
\begin{mylemma}
With probability at least $\tau$
\begin{equation}
\sup_{\mb{x},f}|f(\mb{x})| \le \sqrt{\mathcal{J}}
\end{equation}
where $\mathcal{J}=mC_4^2h^2(0)+L^2C_1^2C_3^2C_4^2((m-1)\cos\theta+1) + 2\sqrt{m}C_1C_3C_4^2L|h(0)|\sqrt{(m-1)\cos\theta+1}$.
\end{mylemma}

Given these lemmas, we proceed to prove Theorem \ref{thm:est_err}.
The Rademacher complexity $\mathcal{R}_n(\mathcal{A})$ of $\mathcal{A}$ is
\begin{equation}
\label{eq:ra}
\begin{array}{lll}
\mathcal{R}_n(\mathcal{A})=\mathbb{E}[\textrm{sup}_{f\in \mathcal{F}}\frac{1}{n}\sum_{i=1}^{n}\sigma_i \ell(f(\mb{x}^{(i)}),y_i)]
\end{array}
\end{equation}
$\ell(\cdot,y)$ is Lipschitz continuous with respect to the first argument, and the constant $L$ is $\sup_{\mb{x},y,f}2|f(\mb{x})-y|\le2\sup_{\mb{x},y,f}(|f(\mb{x})|+|y|) \le 2(\sqrt{\mathcal{J}}+C_2)$. Applying the composition property of Rademacher complexity, we have
\begin{equation}
\mathcal{R}_n(\mathcal{A}) \le 2(\sqrt{\mathcal{J}}+C_2)\mathcal{R}_n(\mathcal{F})
\end{equation}
Using Lemma \ref{lem:rc_f}, we have
\begin{equation}
\label{eq:R_A}
\mathcal{R}_n(\mathcal{A}) \le 2(\sqrt{\mathcal{J}}+C_2)(\frac{2 L C_{1}C_{3}C_{4}\sqrt{m}}{\sqrt{n}} + \frac{C_{4}|h(0)|\sqrt{m}}{\sqrt{n}})
\end{equation}
Note that $\sup_{\mb{x},y,f}|\ell(f(\mb{x}),y)|\le (\sqrt{\mathcal{J}}+C_2)^2$, and plugging Eq.(\ref{eq:R_A}) into Lemma \ref{lem:rc_bd} completes the proof.

\subsubsection{Extensions}
In the above analysis, we consider a simple neural network described in Section \ref{sec:setup}. In this section, we present how to extend the analysis to more complicated cases, such as neural networks with multiple hidden layers, other loss functions and multiple outputs.

\paragraph{Multiple Hidden Layers}
The analysis can be extended to multiple hidden layers by recursively applying the composition property of Rademacher complexity to the hypothesis set.

We define the hypothesis set $\mathcal{F}^P$ for neural network with $P$ hidden layers in a recursive manner:
\begin{equation}
\begin{array}{lll}
\mathcal{F}^0 &=& \{f^0|f^0(\mb{x})=\sum_{j=1}^d w_j^0 x_j = \mb{w}^0\cdot\mb{x}\}\\
\mathcal{F}^1 &=& \mathcal{F}=\{f^1|f^1(\mb{x})=\sum_{j=1}^{m^0}{w_j}^1h(f_j^0(\mb{x})),f_j^0\in\mathcal{F}^0\}\\
\mathcal{F}^{p} &=&\{f^p|f^p(\mb{x})=\sum_{j=1}^{m^{p-1}}{w_j}^ph(f_j^{p-1}(\mb{x})),f_j^{p-1}\in\mathcal{F}^{p-1}\}(p=2,\cdots,P)
\end{array}
\end{equation}
where we assume there are $m^p$ units in hidden layer $p$ and $\mb{w}^p=[w_{1}^p, \cdots,w_{m^{p-1}}^p]$ is the connecting weights from hidden layer $p-1$ to one unit in hidden layer $p$. (we index hidden layers from 0, $\mb{w}^0$ is the connecting weight from input to one unit in hidden layer 0). When $P=1$ the above definition recovers the one-hidden-layer case in Section \ref{sec:setup}. We make similar assumptions as Section \ref{sec:setup}: $h(\cdot)$ is L-Lipschitz, $\|\mb{x}\|_2\le C_1$, $\|\mb{w}^p\|_2\le C_3^p$. We also assume that the pairwise angles of the connecting weights $\rho(\mb{w_j}^p, \mb{w_k}^p)$ are lower bounded by $\theta^p$ with probability at least $\tau^p$ for $j\neq k$, where $\mb{w_j}^p$ denotes the connecting weights from previous layer to unit $j$ in current layer, and $\mb{w_k}^p$ denotes the connecting weights from previous layer to unit $k$ in current layer. Under these assumptions, we have the following result:
\begin{mythm}
\label{thm:est_err_gen}
For a neural network with $P$ hidden layers, with probability at least $(1-\delta)\prod_{p=0}^{P-1}\tau^p$
\begin{equation}
\label{eq:est_err_gen}
\begin{array}{lll}
L(\hat{f})-L(f^*)&\leq& 8(\sqrt{\mathcal{J}^P}+C_2)(\frac{(2L)^PC_1C_3^P}{\sqrt{n}}\prod_{p=0}^{P-1}\sqrt{m^p}C_3^p\\
&+&\frac{|h(0)|}{\sqrt{n}}\sum_{p=0}^{P-1}(2L)^{P-1-p}\prod_{j=p}^{P-1}\sqrt{m^j}C_3^{j+1})+ (\sqrt{\mathcal{J}^P}+C_2)^2\sqrt{\frac{2\log(2/\delta)}{n}}
\end{array}
\end{equation}
where
\begin{equation}
\begin{array}{lll}
\mathcal{J}^0 &=& C_1^2(C_3^0)^2\\
\mathcal{J}^{p} &=& m^{p-1}(C_3^p)^2h^2(0)+L^2(C_3^p)^2((m^{p-1}-1)\cos\theta^{p-1}+1)\mathcal{J}^{p-1}\\
&&+2\sqrt{m^{p-1}}(C_3^p)^2L|h(0)|\sqrt{((m^{p-1}-1)\cos\theta^{p-1}+1)\mathcal{J}^{p-1}} (p=1,\cdots,P)
\end{array}
\end{equation}
\end{mythm}

When $P=1$, Eq.(\ref{eq:est_err_gen}) reduces to the estimation error bound of neural network with one hidden layer. Note that the right hand side is a decreasing function w.r.t $\theta^p (p=0,\cdots,P-1)$, hence making the hidden units in each hidden layer to be diverse can reduce the estimation error bound of neural networks with multiple hidden layers.

In order to prove Theorem \ref{thm:est_err_gen}, we first bound the Rademacher complexity of the hypothesis set $\mathcal{F}^{P}$:
\begin{mylemma}
\label{lem:rc_f_gen}
Let $\mathcal{R}_n(\mathcal{F}^P)$ denote the Rademacher complexity of the hypothesis set $\mathcal{F}^P$, then
\begin{multline}
\mathcal{R}_n(\mathcal{F}^P) \le \frac{(2L)^PC_1C_3^P}{\sqrt{n}}\prod_{p=0}^{P-1}\sqrt{m^p}C_3^p
+\frac{|h(0)|}{\sqrt{n}}\sum_{p=0}^{P-1}(2L)^{P-1-p}\prod_{j=p}^{P-1}\sqrt{m^j}C_3^{j+1}
\end{multline}
\end{mylemma}

In addition, we need the following bound.
\begin{mylemma}
With probability at least $\prod_{p=0}^{P-1}\tau^p$,
$\sup_{\mb{x}, f^P\in\mathcal{F}^p}|f^P(\mb{x})| \le \sqrt{\mathcal{J}^P}$, where
\begin{equation}
\begin{array}{lll}
\mathcal{J}^0 & =& C_1^2(C_3^0)^2\\
\mathcal{J}^{p} &=& m^{p-1}(C_3^p)^2h^2(0)+L^2(C_3^p)^2((m^{p-1}-1)\cos\theta^{p-1}+1)\mathcal{J}^{p-1}\\
&&+2\sqrt{m^{p-1}}(C_3^p)^2L|h(0)|\sqrt{((m^{p-1}-1)\cos\theta^{p-1}+1)\mathcal{J}^{p-1}} (p=1,\cdots,P)
\end{array}
\end{equation}
\end{mylemma}

Given these lemmas, we proceed to prove Theorem \ref{thm:est_err_gen}.
For the neural network with $P$ hidden layers, the Rademacher complexity $\mathcal{R}_n(\mathcal{A})$ of $\mathcal{A}$ is
\begin{equation}
\label{eq:ra}
\begin{array}{l}
\mathcal{R}_n(\mathcal{A})=\mathbb{E}[\textrm{sup}_{f\in \mathcal{F}^P}\frac{1}{n}\sum_{i=1}^{n}\sigma_i \ell(f(\mb{x}^{(i)}),y)]
\end{array}
\end{equation}
$\ell(\cdot,y)$ is Lipschitz continuous with respect to the first argument, and the constant $L$ is $\sup_{\mb{x},y,f}2|f(\mb{x})-y|\le2\sup_{\mb{x},y,f}(|f(\mb{x})|+|y|) \le 2(\sqrt{\mathcal{J}^P}+C_2)$. Applying the composition property of Rademacher complexity, we have
\begin{equation}
\mathcal{R}_n(\mathcal{A}) \le 2(\sqrt{\mathcal{J}^P}+C_2)\mathcal{R}_n(\mathcal{F})
\end{equation}
Using Lemma \ref{lem:rc_f_gen}, we have
\begin{multline}
\label{eq:R_A_gen}
\mathcal{R}_n(\mathcal{A}) \le 2(\sqrt{\mathcal{J}^P}+C_2)(\frac{(2L)^PC_1C_3^P}{\sqrt{n}}\prod_{p=0}^{P-1}\sqrt{m^p}C_3^p
+\frac{|h(0)|}{\sqrt{n}}\sum_{p=0}^{P-1}(2L)^{P-1-p}\prod_{j=p}^{P-1}\sqrt{m^j}C_3^{j+1})
\end{multline}
Note that $\sup_{\mb{x},y,f}|\ell(f(\mb{x}),y)|\le (\sqrt{\mathcal{J}^P}+C_2)^2$, and plugging Eq.(\ref{eq:R_A_gen}) into Lemma \ref{lem:rc_bd} completes the proof.
\paragraph{Other Loss Functions} Other than regression, a more popular application of neural network is classification. For binary classification, the most widely used loss functions are logistic loss and hinge loss. Estimation error bounds similar to that in Theorem \ref{thm:est_err} can also be derived for these two loss functions.
\begin{mylemma}
\label{lem:logistic_ub}
Let the loss function $\ell(f(\mb{x}),y)=\log(1+\exp(-yf(\mb{x})))$ be the logistic loss where $y\in\{-1,1\}$, then with probability at least $(1-\delta)\tau$
\begin{equation}
\begin{array}{lll}
L(\hat{f})-L(f^*)
\leq \frac{4}{1+\exp(-\sqrt{\mathcal{J}})}(2LC_1C_3C_4+C_4|h(0)|)\frac{\sqrt{m}}{\sqrt{n}}+ \log (1+\exp(\sqrt{J}))\sqrt{\frac{2\log(2/\delta)}{n}}
\end{array}
\end{equation}
\end{mylemma}

\begin{mylemma}
\label{lem:angle}
Let $\ell(f(\mb{x}),y)=\max(0,1-yf(\mb{x}))$ be the hinge loss where $y\in\{-1,1\}$, then with probability at least $(1-\delta)\tau$
\begin{equation}
\begin{array}{lll}
L(\hat{f})-L(f^*)\leq 4(2LC_1C_3C_4+C_4|h(0)|)\frac{\sqrt{m}}{\sqrt{n}}+ (1+\sqrt{J})\sqrt{\frac{2\log(2/\delta)}{n}}
\end{array}
\end{equation}
\end{mylemma}


\paragraph{Multiple Outputs} The analysis can be also extended to neural networks with multiple outputs, provided the loss function factorizes over the dimensions of the output vector. Let $\mb{y}\in \mathbb{R}^K$ denote the target output vector, $\mb{x}$ be the input feature vector and $\ell(f(\mb{x}),\mb{y})$ be the loss function. If $\ell(f(\mb{x}),\mb{y})$ factorizes over $k$, i.e., $\ell(f(\mb{x}),\mb{y})=\sum_{k=1}^{K}\ell'(f(\mb{x})_k,y_k)$, then we can perform the analysis for each $\ell'(f(\mb{x})_k,y_k)$ as that in Section \ref{sec:thm1_proof} separately and sums the estimation error bounds up to get the error bound for $\ell(f(\mb{x}),\mb{y})$. Here we present two examples. For multivariate regression, the loss function $\ell(f(\mb{x}),\mb{y})$ is a squared loss: $\ell(f(\mb{x}),\mb{y})=\|f(\mb{x})-\mb{y}\|_2^2$, where $f(\cdot)$ is the prediction function. This squared loss can be factorized as $\|f(\mb{x})-\mb{y}\|_2^2=\sum_{k=1}^{K}(f(\mb{x})_k-y_k)^2$. We can obtain an estimation error bound for each $(f(\mb{x})_k-y_k)^2$ according to Theorem \ref{thm:est_err}, then sum these bounds together to get the bound for $\|f(\mb{x})-\mb{y}\|_2^2$.

For multiclass classification, the commonly used loss function is cross-entropy loss: $\ell(f(\mb{x}),\mb{y})=-\sum_{k=1}^{K}y_k\log a_k$, where $a_k=\frac{\exp(f(\mb{x})_k)}{\sum_{j=1}^{K}\exp(f(\mb{x})_j)}$. We can also derive error bounds similar to that in Theorem \ref{thm:est_err} by using the composition property of Rademacher complexity. First we need to find the Lipschitz constant:
\begin{mylemma}
\label{lem:cross_ent}
Let $\ell(f(\mb{x}),\mb{y})$ be the cross-entropy loss, then with probability at least $\tau$, for any $f$, $f'$
\begin{equation}
|\ell(f(\mb{x}),y)-\ell(f'(\mb{x}),y)|\le \frac{K-1}{K-1+\exp(-2\sqrt{\mathcal{J}})}\sum_{k=1}^K|f(\mb{x})_k-f'(\mb{x})_k|
\end{equation}
\end{mylemma}

With Lemma \ref{lem:cross_ent}, we can get the Rademacher complexity of cross entropy loss by performing the Rademacher complexity analysis for each $f(\mb{x})_k$ as that in Section \ref{sec:thm1_proof} separately, and multiplying the sum of them by $\frac{K-1}{K-1+\exp(-2\sqrt{\mathcal{J}})}$ to get the Rademacher complexity of $\ell(f(\mb{x}),\mb{y})$. And as the loss function can be bounded by
\begin{equation}
|\ell(f(\mb{x}),\mb{y})|\le \log (1+(K-1)\exp(2\sqrt{\mathcal{J}}))
\end{equation}
we can use similar techniques as the proof of Theorem \ref{thm:est_err} to get the estimation error bound.

\subsection{Approximation Error}

Now we proceed to investigate how the diversity of weight vectors affects the approximation error bound. For the ease of analysis, following \citep{barron1993universal}, we assume the target function $g$ belongs to a function class with smoothness expressed in the first moment of its Fourier representation: we define function class $\Gamma_C$ as the set of functions $g$ satisfying
\begin{equation}
\int_{\|\mb{x}\|_2\le C_1}|\mb{w}||\tilde{g}(\mb{w})|d\mb{w}\le C
\end{equation}
where $\tilde{g}(\mb{w})$ is the Fourier representation of $f(\mb{x})$ and we assume $\|\mb{x}\|_2 \le C_1$ throughout this paper. We use function $f$ in $\mathcal{F} = \{f|f(\mb{x})=\sum_{j=1}^m \alpha_j h(\mb{w_j}^T \mb{x})\}$ which is the NN function class defined in Section \ref{sec:setup}, to approximate $g\in \Gamma_C$. Recall the following conditions of $\mathcal{F}$:
\begin{align}
&\forall j \in \{1,\cdots,m\}, \|\mb{w_j}\|_2\le C_3\\
&\|\mb{\alpha}\|_2\le C_4\\
&\forall j\neq k, \rho(\mb{w_j},\mb{w_k})\ge\theta (\text{with probability at least $\tau$})
\end{align}
where the activation function $h(t)$ is the sigmoid function and we assume $\|\mb{x}\|_2\le C_1$.
The following theorem states the approximation error.
\begin{mythm}
\label{thm:appro2}
Given $C>0$, for every function $g\in \Gamma_C$ with $g(0)=0$, for any measure $P$, if
\begin{align}
&C_1C_3\ge 1\\
&C_4 \ge 2\sqrt{m} C\\
&m\le2(\lfloor\frac{\frac{\pi}{2}-\theta}{\theta}\rfloor+1)
\end{align}
then with probability at least $\tau$, there is a function $f \in \mathcal{F}$ such that
\begin{equation}
\label{eq:app_eb}
\|g - f\|_L \le 2C(\frac{1}{\sqrt{n}} + \frac{1+ 2\ln C_1C_3}{C_1C_3}) + 4mCC_1C_3\sin(\frac{\theta'}{2})
\end{equation}
where $\|f\|_{L} = \sqrt{\int_{\mb{x}}f^2(\mb{x})dP(\mb{x})}$, $\theta'=\min(3m\theta, \pi)$.
\end{mythm}

Note that the approximation error bound in Eq.(\ref{eq:app_eb}) is an increasing function of $\theta$. Hence increasing the diversity of hidden units would hurt the approximation capability of neural networks.

\subsubsection{Proof}
Before proving Theorem \ref{thm:appro2}, we need the following lemma\footnote{The proofs of lemmas utilized in this section are deferred to Appendix \ref{apd:app}.}:
\begin{mylemma}
\label{lem:theta_sum_bound}
For any three nonzero vectors $\mb{u_1}$, $\mb{u_2}$, $\mb{u_3}$, let $\theta_{12} = \arccos (\frac{\mb{u_1} \cdot \mb{u_2}}{\|\mb{u_1}\|_2 \|\mb{u_2}\|_2})$, $\theta_{23} = \arccos (\frac{\mb{u_2} \cdot \mb{u_3}}{\|\mb{u_2}\|_2 \|\mb{u_3}\|_2})$, $\theta_{13} = \arccos (\frac{\mb{u_1} \cdot \mb{u_3}}{\|\mb{u_1}\|_2 \|\mb{u_3}\|_2})$. We have $\theta_{13}\le\theta_{12}+\theta_{23}$.
\end{mylemma}
In order to approximate the function class $\Gamma_C$, we first remove the constraints $\rho(\mb{w_j},\mb{w_k})\ge\theta$ and obtain an approximation error:
\begin{mylemma}
\label{lem:appro2}
Let $\mathcal{F'} = \{f|f(\mb{x})=\sum_{j=1}^m \alpha_j h(\mb{w_j}^T \mb{x})\}$ be the function class satisfying the following constraints:
\begin{itemize}
\item $|\alpha_j| \le 2C$
\item $\|\mb{w_j}\|_2 \le C_3$
\end{itemize}
Then for every $g\in \Gamma_C$ with $g(0)=0$, $\exists f' \in \mathcal{F'}$ such that
\begin{equation}
\|g(\mb{x}) - f'(\mb{x})\|_L \le 2C(\frac{1}{\sqrt{n}} + \frac{1+ 2\ln C_1C_3}{C_1C_3})
\end{equation}
\end{mylemma}

We also need the following lemma, which intends to approximate a set of vectors $(\mb{w_j}')_{j=1}^m$ by a set of vectors satisfying the constraints $\rho(\mb{w_j},\mb{w_k})\ge\theta$:
\begin{mylemma}
\label{lem:theta_appro}
For any $0 < \theta < \frac{\pi}{2}$, $m\le2(\lfloor\frac{\frac{\pi}{2}-\theta}{\theta}\rfloor+1)$, $(\mb{w_j}')_{j=1}^m$, $\exists (\mb{w_j})_{j=1}^m$ such that
\begin{align}
&\forall j\neq k\in\{1,\cdots,m\}, \rho(\mb{w_j},\mb{w_k})\ge \theta\\
&\forall j\in\{1,\cdots,m\}, \|\mb{w_j}\|_2 = \|\mb{w_j}'\|_2\\
&\forall j \in \{1,\cdots,m\},\arccos (\frac{\mb{w_j} \cdot \mb{w_j}'}{\|\mb{w_j}\|_2\|\mb{w_j}'\|_2}) \le \theta'
\end{align}
where $\theta' = \min(3m\theta,{\pi})$.
\end{mylemma}

Finally, the following lemma is needed:

\begin{mylemma}
\label{lem:fprimefbound}
For any $f' \in \mathcal{F}'$, $\exists f \in \mathcal{F}''$ such that
\begin{equation}
\|f' - f\|_L \le 4mCC_1C_3\sin(\frac{\theta'}{2})
\end{equation}
where $\theta' = \min(3m\theta,{\pi})$.
\end{mylemma}
Now we proceed to prove Theorem \ref{thm:appro2}. For every $g\in \Gamma_C$ with $g(0)=0$, according to Lemma \ref{lem:appro2}, $\exists f' \in \mathcal{F'}$ such that
\begin{equation}
\|g-f'\|_L \le 2C(\frac{1}{\sqrt{n}} + \frac{1+ 2\ln C_1C_4}{C_1C_4})
\end{equation}
According to Lemma \ref{lem:fprimefbound}, we can find $f\in \mathcal{F}$ such that
\begin{equation}
\label{eq:appro2}
\|f-f'\|_L \le 4mCC_1C_3\sin(\frac{\theta'}{2})
\end{equation}
The proof is completed by noting
\begin{equation}
\|g-f\|_L \le \|g-f'\|_L+\|f'-f\|_L
\end{equation}

\begin{table}[t]
\centering
\begin{tabular}{|c|c|c|c|c|c|c|}
\hline
& \#categories&\#samples&vocab. size\\
\hline
TDT& 30&9394&5000\\
\hline
20-News& 20&18846&5000\\
\hline
Reuters& 9&7195&5000\\
\hline
\end{tabular}
\caption{Statistics of Datasets}\label{table:stat_dataset}
\end{table}

\begin{table}[t]
\centering
\begin{tabular}{|c|c|c|c|c|c|c|}
\hline
K & 25&50&100&200&500\\
\hline
RBM&11.2&11.4&11.9&12.1&47.4\\
\hline
MAR-RBM&\textbf{78.4}&\textbf{84.2}&\textbf{78.6}&\textbf{79.9}&\textbf{77.6}\\
\hline
\end{tabular}
\caption{Precision$@$100 (\%) on TDT dataset}
\label{tb:prec_tdt}
\end{table}

\begin{table}[t]
\centering
\begin{tabular}{|c|c|c|c|c|c|c|}
\hline
K & 25&50&100&200&500\\
\hline
RBM&6&6.1&5.7&9.2&\textbf{22.3}\\
\hline
MAR-RBM&\textbf{14.5}&\textbf{24.9}&\textbf{15.4}&\textbf{20.3}&21.1\\
\hline
\end{tabular}
\caption{Precision$@$100 (\%) on 20-News dataset}
\label{tb:prec_20news}
\end{table}

\section{Experiments}
We present empirical results for mutual angle regularized restricted Boltzmann machine (RBM) and distance metric learning (DML), which demonstrate the merits of the mutual angular regularizer (MAR) in capturing long-tail patterns, improving interpretability and reducing model complexity without sacrificing expressivity.
\subsection{Restricted Boltzmann Machine with Mutual Angular Regularization}
We apply a variant of RBM --- Replicated Softmax RBM \citep{hinton2009replicated} ---  which is introduced in Appendix \ref{apd:rsr}, for document representation learning and topic modeling, and investigate whether the MAR can help capture long-tail topics, enhance the interpretability of topics, improve the effectiveness of document representations in retrieval and clustering while reducing their dimensionalities and computational cost.
\subsubsection{Experimental Setup}

\begin{table}[t]
\centering
\begin{tabular}{|c|c|c|c|c|c|c|}
\hline
K &25 &50&100&200&500\\
\hline
RBM&37.7&38.1&50.1&64.0&\textbf{70.1}\\
\hline
MAR-RBM&\textbf{67.8}&\textbf{73.3}&\textbf{75.9}&\textbf{70.3}&66.2\\
\hline
\end{tabular}
\caption{Precision$@$100 (\%) on Reuters dataset}
\label{tb:prec_rcv1}
\end{table}

\begin{table*}[t]
\centering
\begin{tabular}{|c|c|c|c|c|c|c|c|c|c|}
\hline
Category ID&1 &2 &3 &4 &5 &6 &7 &8 & 9\\
\hline
Num. of Docs & 3713 & 2055 & 321 & 298 &245 & 197 & 142 & 114 &110\\
\hline
RBM Precision (\%)& 68.5& 44.4 & 9.1 & 10.1 & 6.3 &4.4 &3.8 &3.0 &2.6\\
\hline
\parbox{2.5cm}{MAR-RBM \\ Precision (\%)} & 89.7 &80.2 & 31.3 &39.5 & 26.5 & 22.7&9.4&14.0&12.9\\
\hline
Improvement & 31\% & 81\% & 245\% & 289\% & 324\% & 421\% &148\% &366\% & 397\%\\
\hline
\end{tabular}
\caption{Precision$@$100 on each category in Reuters dataset}
\label{tb:prec_cat_rcv1}
\end{table*}

\begin{table}[h]
\centering
\begin{tabular}{|c|c|c|c|}
\hline
& TDT&20-News&Reuters\\
\hline
BOW&40.9&7.4&69.3\\
LDA \citep{blei2003latent}&79.4&19.6&68.5\\
DPP-LDA \citep{NIPS2012_4660}&81.9&18.2&69.9 \\
PYTM \citep{sato2010topic}& 78.7 & 20.1 & 70.6 \\
LIDA \citep{archambeau2015latent}& 77.9 & 21.8 & 71.4 \\
DocNADE \citep{larochelle2012neural}&80.3&16.8&72.6\\
PV \citep{le2014distributed}&81.7&19.1&\tb{76.9}\\
RBM \citep{hinton2009replicated}&47.4&22.3&70.1\\
MAR-RBM &\tb{84.2}&\tb{24.9}&\tb{75.9}\\

\hline
\end{tabular}
\caption{Precision$@$100 (\%) on three datasets}\label{table:cmp_ap}
\end{table}

\begin{table}[h]
\centering
\begin{tabular}{|c|c|c|c|c|c|c|}
\hline
K &25&50&100&200&500\\
\hline
RBM&19.7&19.1&14.4&13.0&23.3\\
\hline
MAR-RBM& \textbf{52.4}&\textbf{46.2}&\textbf{46.5}&\textbf{41.4}&\textbf{39.5}\\
\hline
\end{tabular}
\caption{Clustering accuracy (\%) on TDT test set}
\label{tb:ac_tdt}
\end{table}

\begin{table}[h]
\centering
\begin{tabular}{|c|c|c|c|c|c|c|c|}
\hline
K &25&50&100&200&500\\
\hline
RBM&6.4&6.8&\textbf{21.5}&12.7&22.7\\
\hline
MAR-RBM& \textbf{18.2}&\textbf{29.4}&19.8&\textbf{25.9}&\textbf{25.6}\\
\hline
\end{tabular}
\caption{Clustering accuracy (\%) on 20-News test set}
\label{tb:ac_20news}
\end{table}

\begin{table}[h]
\centering
\begin{tabular}{|c|c|c|c|c|c|c|}
\hline
K &25 &50&100&200&500\\
\hline
RBM&45.0&41.7&38.4&46.8&47.6\\
\hline
MAR-RBM&\textbf{51.4}&\textbf{58.6}&\textbf{60.9}&\textbf{53.4}&\textbf{48.5}\\
\hline
\end{tabular}
\caption{Clustering accuracy (\%) on Reuters test set}
\label{tb:ac_rcv1}
\end{table}

\begin{table}[h]
\centering
\begin{tabular}{|c|c|c|c|}
\hline
& TDT&20-News&Reuters\\
\hline
BOW&51.3&21.3&49.7\\
LDA \citep{blei2003latent}&45.2&21.9&51.2\\
DPP-LDA \citep{NIPS2012_4660}&46.3&10.9&49.3\\
PYTM \citep{sato2010topic}& 46.9& 21.5& 51.7\\
LIDA \citep{archambeau2015latent} &47.3 &17.4 &53.1 \\
DocNADE \citep{larochelle2012neural}&45.7&18.7&48.7\\
PV \citep{le2014distributed}&48.2&24.3&52.8\\
RBM \citep{hinton2009replicated}&23.3&22.7&47.6\\
MAR-RBM&\textbf{52.4}&\textbf{29.4}&\textbf{60.9}\\
\hline
\end{tabular}
\caption{Clustering accuracy (\%) on three datasets}\label{table:cmp_ac}
\end{table}

\begin{table}[t]
\centering
\begin{tabular}{|c|c|c|c|c|c|}
\hline
K &25&50&100&200&500\\
\hline
RBM&2602&2603&2606&2609&2350\\
\hline
MAR-RBM&\textbf{1699}&\textbf{1391}&\textbf{1658}&\textbf{1085}&\textbf{859}\\
\hline
\end{tabular}
\caption{Perplexity on TDT test set}
\label{tb:per_tdt}
\end{table}

\begin{table}[t]
\centering
\begin{tabular}{|c|c|c|c|c|c|}
\hline
K &25&50&100&200&500\\
\hline
RBM&764.9&765.1&765&741&633\\
\hline
MAR-RBM&\textbf{713}&\textbf{623}&\textbf{659}&\textbf{558}&\textbf{497}\\
\hline
\end{tabular}
\caption{Perplexity on 20-News test set}
\label{tb:per_20news}
\end{table}

\begin{table}[t]
\centering
\begin{tabular}{|c|c|c|c|c|c|}
\hline
K &25&50&100&200&500\\
\hline
RBM&1147&1129&1130&881&849\\
\hline
MAR-RBM&\textbf{1028}&\textbf{859}&\textbf{746}&\textbf{734}&\textbf{848}\\
\hline
\end{tabular}
\caption{Perplexity on Reuters test set}
\label{tb:per_rcv1}
\end{table}

Three datasets were used in the experiments. The first one \citep{cai2005document} is a subset of the Nist Topic Detection and Tracking (TDT) corpus which contains 9394 documents from the largest 30 categories. 70\% documents were used for training and 30\% were used for testing.
The second dataset is the 20 Newsgroups (20-News), which contains 18846 documents from 20 categories. 60\% documents were used for training and 40\% were used for testing. The third dataset \citep{cai2012manifold} is the Reuters-21578 (Reuters) dataset. Categories with less than 100 documents were removed, which left us 9 categories and 7195 documents. 70\% documents were used for training and 30\% were used for testing. Each dataset used a vocabulary of 5000 words with the largest document frequency. Table \ref{table:stat_dataset} summarizes the statistics of three datasets.

We used gradient methods to train RBM \citep{hinton2009replicated} and MAR-RBM. The mini-batch size was set to 100. The learning rate was set to 1e-4. The number of gradient ascent iterations was set to 1000 and the number of Gibbs sampling iterations in contrastive divergence was fixed to 1. The tradeoff parameter $\lambda$ in MAR-RBM was tuned with 5-fold cross validation. We compared with the following baselines methods: (1) bag-of-words (BOW); (2) Latent Dirichlet Allocation (LDA) \citep{blei2003latent}; (3) LDA regularized with Determinantal Point Process prior (DPP-LDA) \citep{NIPS2012_4660};
(4) Pitman-Yor Process Topic Model (PYTM) \citep{sato2010topic}; (5) Latent IBP Compound Dirichlet Allocation (LIDA) \citep{archambeau2015latent}; (6) Neural Autoregressive Topic Model (DocNADE) \citep{larochelle2012neural}; (7) Paragraph Vector (PV) \citep{le2014distributed}; (8) Replicated Softmax RBM \citep{hinton2009replicated}. The parameters in baseline methods were tuned using 5-fold cross validation.

\subsubsection{Document Retrieval}
\label{sec:retrieval}

In this section, we evaluate the effectiveness of the learned representations on retrieval. Precision$@$100 is used to assess the retrieval performance. For each test document, we retrieve 100 documents from the training set that have the smallest Euclidean distance with the query document. The distance is computed on the learned representations. Precision$@$100 is defined as $n/100$, where $n$ is the number of retrieved documents that share the same class label with the query document.

Table \ref{tb:prec_tdt}, \ref{tb:prec_20news} and \ref{tb:prec_rcv1} show the precision$@$100 under different number $K$ of hidden units on TDT, 20-News and Reuters dataset respectively.
As can be seen from these tables, MAR-RBM with mutual angular regularization largely outperforms unregularized RBM under various choices of $K$, which demonstrates that diversifying the hidden units can greatly improve the effectiveness of document representation learning. The improvement is especially significant when $K$ is small. For example, on TDT dataset, under $K=25$, MAR-RBM improves the precision$@$100 from 11.2\% to 78.4\%.
Under a small $K$, RBM allocates most (if not all) hidden units to cover dominant topics, thus long-tail topics have little chance to be modeled effectively. MAR-RBM solves this problem by increasing diversity of these hidden units to enforce them to cover not only the dominant topics, but also the long-tail topics. Thereby, the learned representations are more effective in capturing the long-tail semantics and the retrieval performance is greatly improved. As $K$ increases, the performance of RBM increases. This is because under a larger $K$, some hidden units can be spared to model topics in the long-tail region. In this case, enforcing diversity still improves performance, though the significance of improvement diminishes as $K$ increases.

To further examine whether MAR-RBM can effectively capture the long-tail semantics, we show the precision$@$100 on each of the 9 categories in the Reuters dataset in Table \ref{tb:prec_cat_rcv1}. The 2nd row shows the number of documents in each category. The distribution of document frequency is in a power-law fashion, where dominant categories (such as 1 and 2) have a lot of documents while most categories (called long-tail categories) have a small amount of documents. The 3rd and 4th row show the precision@100 achieved by RBM and MAR-RBM on each category. The 5th row shows the relative improvement of MAR-RBM over RBM. The relative improvement is defined as $\frac{P_{drbm}-P_{rbm}}{P_{rbm}}$, where $P_{drbm}$ and $P_{rbm}$ denote the precision@100 achieved by MAR-RBM and RBM respectively. While MAR-RBM improves RBM over all the categories, the improvements on long-tail categories are much more significant than dominant categories. For example, the relative improvements on category 8 and 9 are 366\% and 397\% while the improvements are 31\% and 81\% on category 1 and 2. This indicates that MAR-RBM can effectively capture the long-tail topics, thereby improve the representations learned for long-tail categories significantly.

One great merit of MAR-RBM is that it can achieve notable performance under a small $K$, which is of key importance for fast retrieval. On the TDT dataset, MAR-RBM can achieve a precision$@$100 of 78.4\% with $K=25$, which cannot be achieved by RBM even when $K$ is raised to 500. This indicates that with MAR-RBM, one can perform retrieval on low-dimensional representations, which is usually much easier than on high-dimensional representations. For example, in KD tree \citep{friedman1977algorithm} based nearest neighbor search, while building a KD tree on feature vectors with hundreds of dimensions is extremely hard, feature vectors whose dimension is less than one hundred are much easier to handle.

Table \ref{table:cmp_ap} presents the comparison of MAR-RBM with the state of the art document representation learning methods. As can be seen from this table, our method achieves the best performances on the TDT and 20-News datasets and achieves the second best performance on the Reuters dataset. The bag-of-word (BOW) representation cannot capture the underlying semantics of documents, thus its performance is inferior.
LDA, RBM and neural network based approaches including DocNADE \citep{larochelle2012neural} and PV \citep{le2014distributed} can represent documents into the latent topic space where document retrieval can be performed more accurately. However, they lack the mechanisms to cover long-tail topics, hence the resultant representations are less effective. PYTM \citep{sato2010topic} and LIDA \citep{archambeau2015latent} use power-law priors to encourage new topics to be generated, however, the newly generated topics may still be used to model the dominant semantics rather than those in the long-tail region.
DPP-LDA \citep{NIPS2012_4660} uses a correlation kernel Determinantal Point Process (DPP) prior to diversify the topics in LDA. However, it does not improve LDA too much.

\subsubsection{Clustering}

Another task we study is to perform k-means clustering on the learned representations. In our experiments, the input cluster number of k-means was set to the ground truth
number of categories in each dataset. In each run, k-means was repeated 10 times with different random initializations and the solution with lowest loss value was returned.
Following \citep{cai2011locally}, we used accuracy to measure the clustering performance.
Please refer to \citep{cai2011locally} for the definition of accuracy. Table \ref{tb:ac_tdt}, \ref{tb:ac_20news} and \ref{tb:ac_rcv1} show the clustering accuracy on TDT, 20-News and Reuters test set respectively under different number $K$ of hidden units.

As can be seen from these tables, with diversification, MAR-RBM achieves significantly better clustering accuracy than the standard RBM. On TDT test data, the best accuracy of RBM is 23.3\%. MAR-RBM dramatically improves the accuracy to 52.4\%. On Reuters test set, the best accuracy achieved by MAR-RBM is 60.9\%, which is largely better than the 47.6\% accuracy achieved by RBM.
The great performance gain achieved by MAR-RBM attributes to the improved effectiveness of the learned representations. RBM lacks the ability to learn hidden units to cover long-tail topics, which largely inhibits its ability to learn rich and expressive representations. MAR-RBM uses the MAR to regularize the hidden units to enhance their diversity. The learned hidden units under MAR-RBM can not only represent dominant topics effectively, but also cover long-tail topics properly. Accordingly, the resultant representations can cover diverse semantics and dramatically improve the performance of k-means clustering.

MAR-RBM can achieve a high accuracy with a fairly small number of hidden units, which greatly facilitates computational efficiency. For example, on TDT dataset, with 25 hidden units, MAR-RBM can achieve an accuracy of 52.4\%, which cannot be achieved by RBM with even 500 hidden units. The computational complexity of k-means is linear to the feature dimension. Thus, on this dataset, with the latent representations learned by MAR-RBM, k-means can achieve a significant speed up. Similar observations can be seen from the other two datasets.

Table \ref{table:cmp_ac} presents the comparison of MAR-RBM with the baseline methods on clustering accuracy. As can be seen from this table, our method consistently outperforms the baselines across all three datasets. The analysis of why MAR-RBM is better than the baseline methods follows that presented in Section \ref{sec:retrieval}.

\subsubsection{Perplexity on Testing Data}
Following \citep{hinton2009replicated}, we computed perplexity on the held out test set to assess the document modeling power of RBM and MAR-RBM. Table \ref{tb:per_tdt}, \ref{tb:per_20news} and \ref{tb:per_rcv1} compare the perplexity of RBM and MAR-RBM computed on TDT, 20-News and Reuters dataset respectively. As can be seen from these tables, MAR-RBM can achieve significant lower (better) perplexity than RBM, which corroborates that by diversifying the hidden units, the document modeling power of RBM can be dramatically improved.

\subsubsection{Sensitivity to Parameters}

\begin{figure*}[t]

\begin{center}
\centerline{\includegraphics[width=0.3\columnwidth]{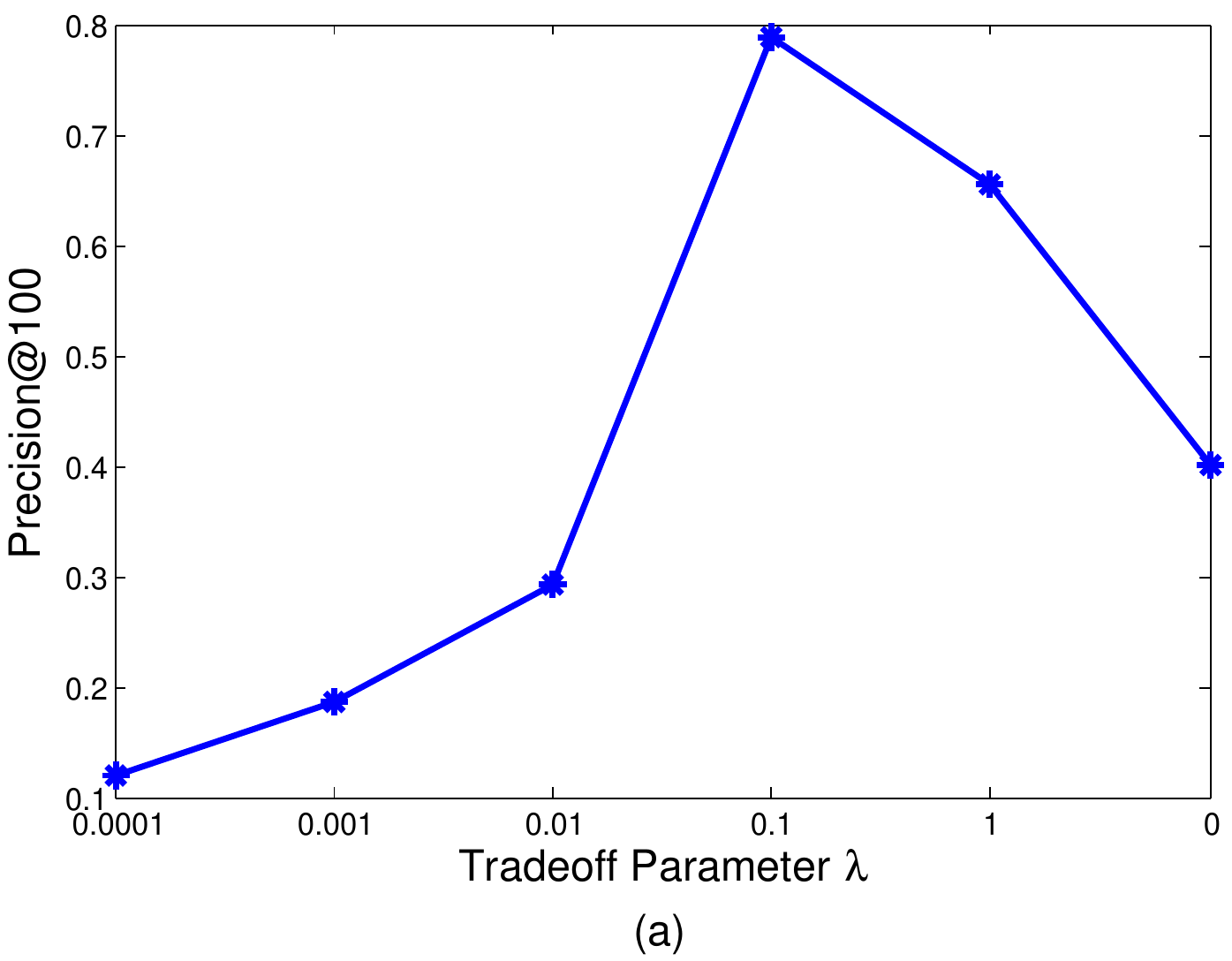}
\includegraphics[width=0.3\columnwidth]{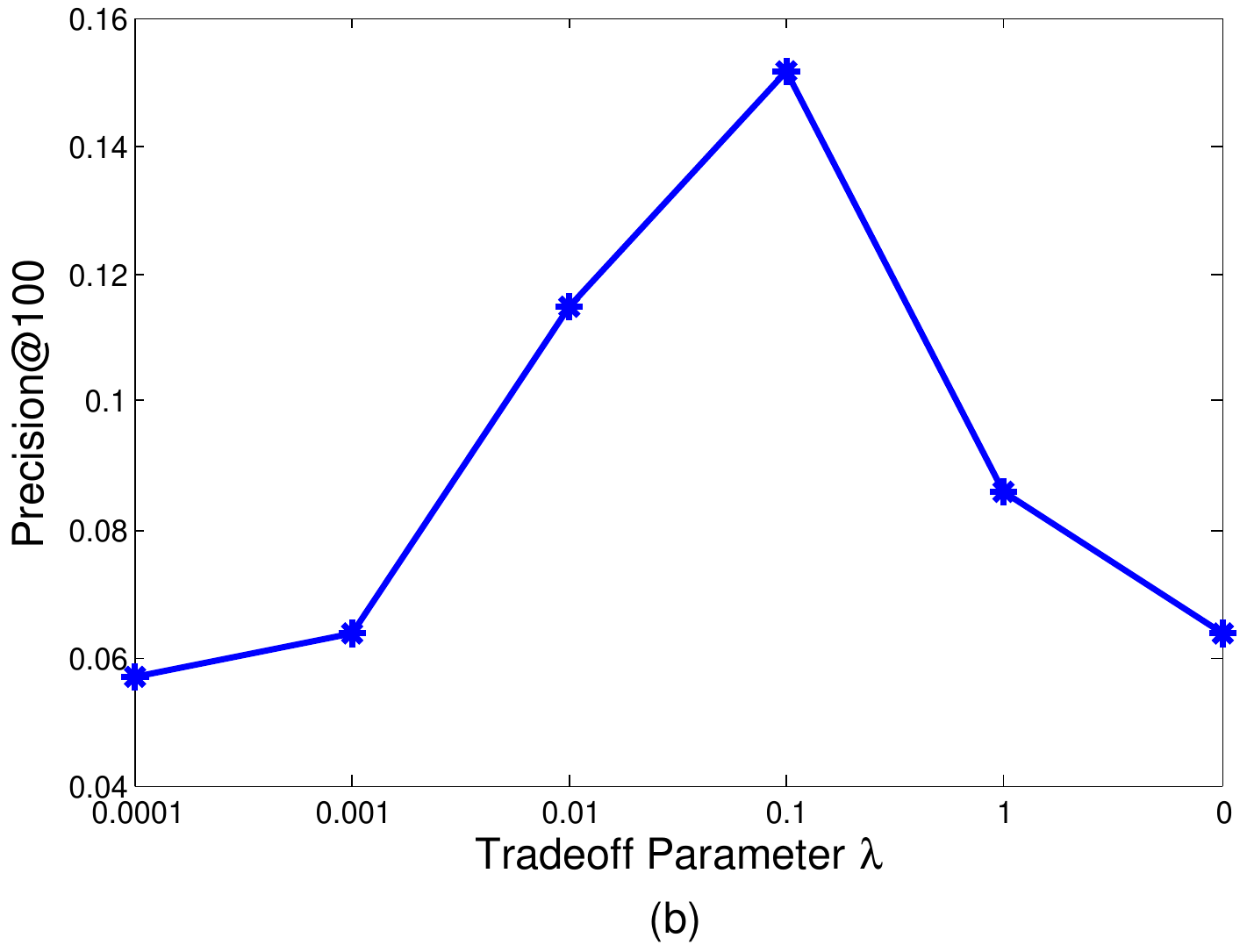}
\includegraphics[width=0.3\columnwidth]{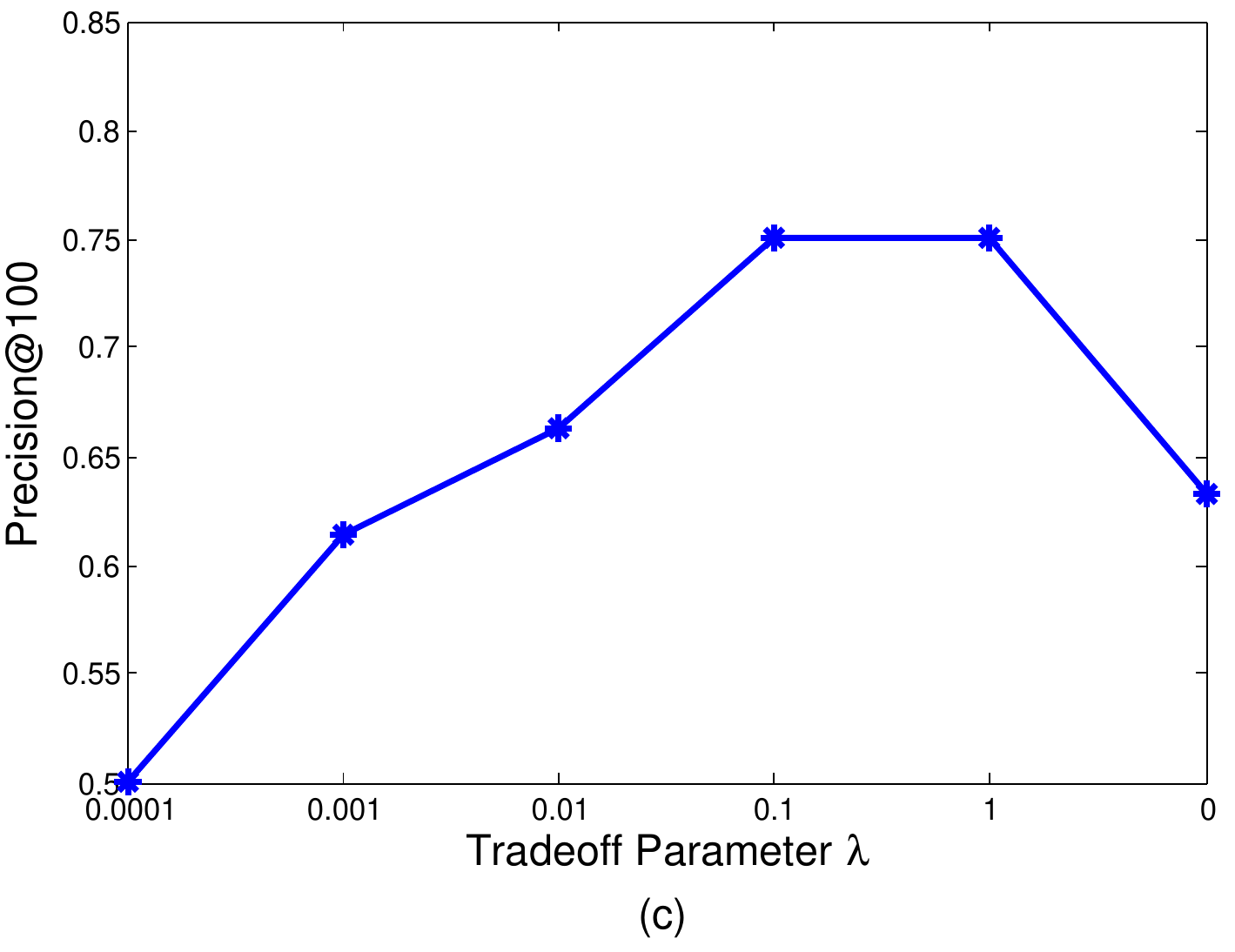}
}
\caption{Sensitivity of MAR-RBM to tradeoff parameter $\lambda$ on (a) TDT dataset (b) 20-News dataset (c) Reuters dataset}
\label{fig:paratune_lambda}
\end{center}
\end{figure*}


\begin{table*}[t]
\caption{Exemplar topics learned By RBM and MAR-RBM}
\label{tb:20news_topics}
\begin{center}
\begin{small}
\begin{tabular}{lllll}
\hline
\multicolumn{5}{c}{RBM} \\
\hline
Topic 1 &Topic 2 &Topic 3&Topic 4 & Topic 5\\
\hline
president&iraq&iraq&olympic&spkr\\
clinton&united&un&games&voice\\
iraq&un&iraqi&nagano&tobacco\\
united&weapons&lewinsky&olympics&olympic\\
spkr&iraqi& saddam& game&games\\
house&nuclear&clinton&team&people\\
people&india&baghdad&gold&olympics\\
lewinsky&minister&inspectors&japan&nagano\\
government&saddam&weapons&medal&game \\
white&military&white &hockey&gold\\
\hline
\multicolumn{5}{c}{MAR-RBM} \\
\hline
Topic 1 &Topic 2 &Topic 3&Topic 4 & Topic 5\\
\hline
president&olympic&iraq&lawyers&students\\
iraq&games&united&kaczynski &japanese\\
clinton&olympics&un&ms&japan\\
united&nagano&weapons&defense&school\\
million&team&iraqi&trial&ms\\
lewinsky&gold&baghdad&judge&united\\
thailand&game&council&people&yen\\
spkr&hockey&inspectors&prosecutors& gm\\
government&medal&nations&kaczynskis& tokyo \\
jones &winter&military&government&south\\
\hline
\end{tabular}
\end{small}
\end{center}
\end{table*}

We study the sensitivity of MAR-RBM to the tradeoff parameter $\lambda$. Figure \ref{fig:paratune_lambda} shows how precision$@$100 in retrieval varies as $\lambda$ increases on the TDT, 20-News and Reuters dataset respectively. The number of hidden units was fixed to 100.
As can be seen from the figures, starting from 0, increasing $\lambda$ improves precision$@$100. That is because a larger $\lambda$ induces more diversity of the hidden units, enabling them to better cover long-tail topics. However, further increasing $\lambda$ causes the precision to drop. This is because, if $\lambda$ is too large, too much emphasis is paid to the diversify regularizer and the data likelihood of RBM is ignored.

\subsubsection{Topic Visualization}
Other than evaluating the learned hidden units of MAR-RBM quantitatively, we also visualize and evaluate them in a qualitative way. We can interpret each hidden unit as a topic and its weight vector as a pseudo distribution over the vocabulary. To visualize each hidden unit, we pick up the top 10 representative words which correspond to the ten largest values in the weight vector. Table \ref{tb:20news_topics} shows 5 topics learned by RBM and 5 topics learned by MAR-RBM. As can be seen from the table, topics learned by RBM have many near-duplicates and are very redundant. In contrast, the topics learned by MAR-RBM are much more diverse, with a broad coverage of various topics including American politics, sports, Iraq war, law and Japanese education. The enhanced diversity makes these topics more interpretable and distinguishable.




%
%

\subsection{Distance Metric Learning with Mutual Angular Regularization}
In this section, on three tasks --- retrieval, clustering and classification --- we corroborate that through diversification it is possible to learn distance metrics that are both compact and effective.
\label{sec:exp}
\begin{table}[t]
\centering
\begin{tabular}{|c|c|c|c|c|c|c|}
\hline
& Feature Dim.&\#training data&\#data pairs\\
\hline
20-News& 5000&11.3K&200K\\
\hline
15-Scenes& 1000&3.2K&200K\\
\hline
6-Activities& 561&7.4K&200K\\
\hline
\end{tabular}
\caption{Statistics of datasets}\label{table:dml_dataset}
\end{table}

\begin{table}[t]
\centering
\begin{tabular}{|c|c|c|c|c|c|c|}
\hline
K & 10&100&300&500&700&900\\
\hline
DML& 72.4&74.0&74.9&75.4&75.8&76.2\\
\hline
MAR-DML& \textbf{76.7}&\textbf{81.0}&\textbf{81.1}&\textbf{79.2}&\textbf{78.3}&\textbf{77.8}\\
\hline
\end{tabular}
\caption{Retrieval average precision (\%) on 20-News dataset}\label{table:retrieval_20news}
\end{table}

\begin{table}[t]
\centering
\begin{tabular}{|c|c|c|c|c|c|}
\hline
K & 10&50&100&150&200\\
\hline
DML& 79.5&80.2&80.7&80.7&80.8\\
\hline
MAR-DML& \textbf{82.4}&\textbf{83.6}&\textbf{83.3}&\textbf{83.1}&\textbf{82.8}\\
\hline
\end{tabular}
\caption{Retrieval average precision (\%) on 15-Scenes dataset}\label{table:retrieval_15scenes}
\end{table}

\begin{table}[t]
\centering
\begin{tabular}{|c|c|c|c|c|c|}
\hline
K & 10&50&100&150&200\\
\hline
DML& 93.2&94.3&94.5&94.5&94.5\\
\hline
MAR-DML& \textbf{96.2}&\textbf{95.5}&\textbf{95.9}&\textbf{95.3}&\textbf{95.1}\\
\hline
\end{tabular}
\caption{Retrieval average precision (\%) on 6-Activities dataset}\label{table:retrieval_6activities}
\label{tb:retrievl_compare}
\end{table}

\begin{table}[t]
\centering
\begin{tabular}{|c|c|c|c|}
\hline
& 20-News&15-Scenes&6-Activities\\
\hline
EUC &62.8&65.3&85.0\\
DML \citep{xing2002distance} &76.2&80.8&94.5\\
LMNN \citep{weinberger2005distance} &67.0&70.3&71.5\\
ITML \citep{davis2007information} &74.7&79.1&94.2\\
DML-eig \citep{ying2012distance} &71.2&71.3&86.7\\
Seraph \citep{niu2012information} &75.8&82.0&89.2\\
MAR-DML &\tb{81.1}&\tb{83.6}&\tb{96.2}\\
\hline
\end{tabular}
\caption{Retrieval average precision (\%) on three datasets}
\label{table:cmp_ap}
\end{table}

\begin{table}[t]
\centering
\begin{tabular}{|c|c|c|c|c|c|c|}
\hline
K & 10&100&300&500&700&900\\
\hline
DML& 23.7	&25.1	&26.2&	26.9&	28.1&	28.4
\\
\hline
MAR-DML& \tb{33.4}&	\tb{42.7}	&\tb{44.6}&	\tb{39.5}&	\tb{40.6}	&\tb{41.3}
\\
\hline
\end{tabular}
\caption{Clustering accuracy (\%) on 20-News dataset}
\label{table:clus_ac_20news_test}
\end{table}

\begin{table}[t]
\centering
\begin{tabular}{|c|c|c|c|c|c|c|}
\hline
K & 10&100&300&500&700&900\\
\hline
DML& 34.1&	35.4&	36.8	&36.9&	38.0&	38.2
\\
\hline
MAR-DML&\tb{ 42.5}&\tb{	49.7}&	\tb{51.1}&	\tb{47.2}	&\tb{47.8}&	\tb{48.1}

\\
\hline
\end{tabular}
\caption{Normalized mutual information (\%) on 20-News dataset}\label{table:nmi_20news_test}
\end{table}

\begin{table}[t]
\centering
\begin{tabular}{|c|c|c|c|c|c|}
\hline
K & 10&50&100&150&200\\
\hline
DML&33.9&36.5&40.1&37.0&	37.8
\\
\hline
MAR-DML &\tb{46.9}&\tb{51.3}	&\tb{46.2}&\tb{46.5}&\tb{49.6}

\\
\hline
\end{tabular}
\caption{Clustering accuracy (\%) on 15-Scenes dataset}\label{table:ac_15scenes_test}
\end{table}

\begin{table}[t]
\centering
\begin{tabular}{|c|c|c|c|c|c|}
\hline
K & 10&50&100&150&200\\
\hline
DML&41.4&	41.0&	42.0	&41.4&41.6
\\
\hline
MAR-DML&\tb{46.7}&\tb{	48.9}	&\tb{47.3}&	\tb{48.8}&	\tb{47.1}
\\
\hline
\end{tabular}
\caption{Normalized mutual information (\%) on 15-Scenes dataset}
\label{table:nmi_15scenes_test}
\end{table}

\begin{table}[t]
\centering
\begin{tabular}{|c|c|c|c|c|c|}
\hline

K & 10&50&100&150&200\\
\hline
DML&75.0	&75.6	&76.1&75.6&		75.7\\
\hline
MAR-DML &\tb{94.9}&\tb{96.3}	&\tb{96.6}&\tb{95.1}&\tb{95.7}
\\
\hline
\end{tabular}
\caption{Clustering accuracy (\%) on 6-Activities dataset}\label{table:ac_6activities_test}
\end{table}

\begin{table}[t]
\centering
\begin{tabular}{|c|c|c|c|c|c|}
\hline
K & 10&50&100&150&200\\
\hline
DML&83.6&	83.5	&84.0&83.5&		83.5\\
\hline
MAR-DML &\tb{90.3}&\tb{91.9}&\tb{91.3}&	\tb{91.4}&\tb{91.1}
\\
\hline
\end{tabular}
\caption{Normalized mutual information (\%) on 6-Activities dataset}\label{table:nmi_6activities_test}
\end{table}

\begin{table}[t]
\centering
\begin{tabular}{|c|c|c|c|}
\hline
& 20-News&15-Scenes&6-Activities\\
\hline
EUC &36.5&29.0&61.6\\
DML \citep{xing2002distance}&28.4&40.1&76.1\\
LMNN \citep{weinberger2005distance}&32.9&33.6&56.9\\
ITML \citep{davis2007information}&34.5&38.2&93.4\\
DML-eig \citep{ying2012distance}&27.3&26.6&63.3\\
Seraph \citep{niu2012information}&\tb{48.1}&48.2&74.8\\
MAR-DML &44.6&\tb{51.3}&\tb{96.6}\\
\hline
\end{tabular}
\caption{Clustering accuracy (\%) on three datasets}\label{table:ac_cmp}
\end{table}

\begin{table}[t]
\centering
\begin{tabular}{|c|c|c|c|}
\hline
& 20-News&15-Scenes&6-Activities\\
\hline
EUC&37.9&28.7&59.9\\
DML \citep{xing2002distance}&38.2&42.0&83.6\\
LMNN \citep{weinberger2005distance}&33.3&34.3&58.2\\
ITML \citep{davis2007information}&39.2&41.5&87.0\\
DML-eig \citep{ying2012distance}&34.0&31.8&58.6\\
Seraph \citep{niu2012information}&49.7&47.5&71.1\\
MAR-DML &\tb{51.1}&\tb{48.9}&\tb{91.9}\\
\hline
\end{tabular}
\caption{Normalized mutual information (\%) on three datasets}\label{table:nmi_cmp}
\end{table}

\subsubsection{Datasets}
We used three datasets in the experiments: 20 Newsgroups\footnote{http://qwone.com/~jason/20Newsgroups/} (20-News), 15-Scenes \citep{lazebnik2006beyond} and 6-Activities \citep{anguita2012human}. The 20-News dataset has 18846 documents from 20 categories, where 60\% of the documents were for training and the rest were for testing. Documents were represented with \textit{tf-idf} (term frequency, inverse document frequency) vectors whose dimensionality is 5000. We randomly generated 100K similar pairs and 100K dissimilar pairs from the training set to learn distance metrics. Two documents were labeled as similar if they belong to the same group and dissimilar otherwise.
The 15-Scenes dataset contains 4485 images belonging to 15 scene classes. 70\% of the images were used for training and 30\% were for testing. Images were represented with bag-of-words vectors whose dimensionality is 1000. Similar to 20-News, we generated 100K similar and 100K dissimilar data pairs for distance learning according to whether two images are from the same scene class or not. The 6-Activities dataset is built from recordings of 30 subjects performing six activities of daily living while carrying a waist-mounted smart phone with embedded inertial sensors. The features are 561-dimensional sensory signals. There are 7352 training instances and 2947 testing instances. Similarly, 100K similar pairs and 100K dissimilar pairs were generated to learn distance metrics. Table \ref{table:dml_dataset} summarizes the statistics of these three datasets.

\subsubsection{Experimental Setup}
Our method MAR-DML contains two key parameters --- the number $K$ of components and the tradeoff parameter $\lambda$ --- both of which were tuned using 5-fold cross validation.
We compared with 6 baseline methods, which were selected according to their popularity and the state of the art performance. They are: (1) Euclidean distance (EUC); (2) Distance Metric Learning (DML) \citep{xing2002distance}; (3) Large Margin Nearest Neighbor (LMNN) metric learning \citep{weinberger2005distance}; (4) Information Theoretical Metric Learning (ITML) \citep{davis2007information}; (5) Distance Metric Learning with Eigenvalue Optimization (DML-eig) \citep{ying2012distance}; (6) Information-theoretic Semi-supervised Metric Learning via Entropy Regularization (Seraph) \citep{niu2012information}. Parameters of the baseline methods were tuned using 5-fold cross validation. Some methods, such as ITML, achieve better performance on lower-dimensional representations which are obtained via Principal Component Analysis. The number of leading principal components were selected via 5-fold cross validation.

\subsubsection{Retrieval}

\begin{table}[t]
\centering
\begin{tabular}{|c|c|c|c|c|c|c|}
\hline
K& 10&100&300&500&700&900\\
\hline
DML& 39.1&	48.0&	53.0	&55.0&	56.4	&57.5
\\
\hline
MAR-DML& \tb{51.3}&\tb{	64.1}	&\tb{64.5}	&\tb{63.3}&	\tb{62.9}&\tb{	61.4}
\\
\hline
\end{tabular}
\caption{3-NN accuracy (\%) on 20-News dataset}\label{table:3nn_20news}
\end{table}


We first applied the learned distance metrics for retrieval. To evaluate the effectiveness of the learned metrics, we randomly sampled 100K similar pairs and 100K dissimilar pairs from 20-News test set, 50K similar pairs and 50K dissimilar pairs from 15-Scenes test set, 100K similar pairs and 100K dissimilar pairs from 6-Activities test set and used the learned metrics to judge whether these pairs were similar or dissimilar. If the distance was greater than some threshold $t$, the pair was regarded as similar. Otherwise, the pair was regarded as dissimilar. We used average precision (AP) to evaluate the retrieval performance.

Table \ref{table:retrieval_20news}, \ref{table:retrieval_15scenes} and \ref{table:retrieval_6activities} show the average precision under different number $K$ of components on 20-News, 15-Scenes and 6-Activities dataset respectively.
As shown in these tables, MAR-DML with a small $K$ can achieve retrieval precision that is comparable to DML with a large $K$. For example, on the 20-News dataset (Table \ref{table:retrieval_20news}), with 10 components, MAR-DML is able to achieve a precision of $76.7\%$, which cannot be achieved by DML with even 900 components. As another example, on the 15-Scenes dataset (Table \ref{table:retrieval_15scenes}), the precision obtained by MAR-DML with $K=10$ is $82.4\%$, which is largely better than the $80.8\%$ precision achieved by DML with $k=200$. Similar behavior is observed on the 6-Activities dataset (Table \ref{table:retrieval_6activities}). This demonstrates that, with diversification, MAR-DML is able to learn a distance metric that is as effective as (if not more effective than) DML, but is much more compact than DML. Such a compact distance metric greatly facilitates retrieval efficiency. Performing retrieval on 10-dimensional latent representations is much easier than on representations with hundreds of dimensions. It is worth noting that the retrieval efficiency gain comes without sacrificing the precision, which allows one to perform fast and accurate retrieval. For DML, increasing $K$ consistently increases the precision, which corroborates that a larger $K$ would make the distance metric to be more expressive and powerful. However, $K$ cannot be arbitrarily large, otherwise the distance matrix would have too many parameters that lead to overfitting. This is evidenced by how the precision of MAR-DML varies as $K$ increases.

Table \ref{table:cmp_ap} presents the comparison with the state of the art distance metric learning methods. As can be seen from this table, our method achieves the best performances across all three datasets. The Euclidean distance does not incorporate distance supervision provided by human, thus its performance is inferior. DML-eig imposes no regularization over the distance metric, which is thus prone to overfitting. To avoid overfitting, ITML utilized a Bregman divergence regularizer and Seraph used a sparsity regularizer. But the performances of both regularizers are inferior to the diversity-promoting MAR utilized by MAR-DML. LMNN is specifically designed for k-NN classification, thus the learned distance metrics cannot guarantee to be effective in retrieval tasks.

\subsubsection{Clustering}
The second task we study is to apply the learned distance metrics for k-means clustering, where the number of clusters was set to the number of categories in each dataset and k-means was run 10 times with random initialization of the centroids. Following \citep{cai2011locally}, we used two metrics to measure the clustering performance: accuracy (AC) and normalized mutual information (NMI). Please refer to \citep{cai2011locally} for their definitions.

Table \ref{table:clus_ac_20news_test},\ref{table:ac_15scenes_test} and \ref{table:ac_6activities_test} show the clustering accuracy on 20-News, 15-Scenes and 6-Activity test set respectively under various number of components $K$. Table \ref{table:nmi_20news_test}, \ref{table:nmi_15scenes_test} and \ref{table:nmi_6activities_test} show the normalized mutual information on 20-News, 15-Scenes and 6-Activity test set respectively.
These tables show that the clustering performance achieved by MAR-DML under a small $K$ is much better than DML under a much larger $K$. For instance, MAR-DML can achieve $33.4\%$ accuracy on the 20-News dataset (Table \ref{table:clus_ac_20news_test}) with 10 components, which is much better than the $28.4\%$ accuracy obtained by DML with 900 components. As another example, the NMI obtained by MAR-DML on the 15-Scenes dataset (Table \ref{table:nmi_15scenes_test}) with $K=10$ is $46.7\%$, which is largely better than the $41.6\%$ NMI achieved by DML with $K=200$.
This again corroborates that the mutual angular regularizer can enable MAR-DML to learn compact and effective distance metrics, which significantly reduce computational complexity while preserving the clustering performance.

Table \ref{table:ac_cmp} and \ref{table:nmi_cmp} present the comparison of MAR-DML with the state of the art methods on clustering accuracy and normalized mutual information. As can be seen from these two tables, our method outperforms the baselines in most cases except that the accuracy on 20-News dataset is worse than the Seraph method. Seraph performs very well on 20-News and 15-Scenes dataset, but its performance is bad on the 6-Activities dataset. MAR-DML achieves consistently good performances across all three datasets.

\begin{table}[t]
\centering
\begin{tabular}{|c|c|c|c|c|c|}
\hline
K & 10&50&100&150&200\\
\hline
DML&47.7	&	47.7	&	50.8	&	51.7	&	51.1\\
\hline
MAR-DML &\tb{57.4}	&\tb{57.5}&\tb{57.9}&\tb{58.8}	&\tb{57.3}\\
\hline
\end{tabular}
\caption{3-NN accuracy (\%) on 15-Scenes dataset}\label{table:3nn_15scenes}
\end{table}


\begin{table}[t]
\centering
\begin{tabular}{|c|c|c|c|c|c|}
\hline
K & 10&50&100&150&200\\
\hline
DML&94.9&	94.8&94.6&95.1&		95.0\\

\hline
MAR-DML &\tb{94.3}	&\tb{96.2}&\tb{96.5}&\tb{95.5}	&\tb{95.9}\\
\hline
\end{tabular}
\caption{3-NN accuracy (\%) on 6-Activities dataset}\label{table:3nn_6activities}
\end{table}

\subsubsection{Classification}
We also apply the learned metrics for k-nearest neighbor classification, which is also an algorithm that largely depends on a good distance measure. For each testing sample, we find its k-nearest neighbors in the training set and use the class labels of the nearest neighbors to classify the test sample. Table \ref{table:3nn_20news}, \ref{table:3nn_15scenes} and \ref{table:3nn_6activities} show the 3-NN classification accuracy on the 20-News, 15-Scenes and 6-Activities dataset.
Similar to retrieval and clustering, MAR-DML with a small $K$ can achieve classification accuracy that is comparable to or better than DML with a large $K$.
Table \ref{table:3nn_cmp} presents the comparison of MAR-DML with the state of the art methods on 3-NN classification accuracy. As can be seen from these two tables, our method outperforms the baselines in most cases except that the accuracy on 20-News dataset is worse than the Seraph method.

\subsubsection{Sensitivity to Parameters}

We study the sensitivity of MAR-DML to the two key parameters: tradeoff parameter $\lambda$ and the number of components $K$. Figure \ref{fig:paratune_lambda} shows how the retrieval average precision (AP) varies as $\lambda$ increases on the 20-News, 15-Scenes and 6-Activities dataset respectively.
The curves correspond to different $K$. As can be seen from the figure, initially increasing $\lambda$ improves AP. The reason is that a larger $\lambda$ encourages the latent factors to be more uncorrelated, thus different aspects of the information can be captured more comprehensively. However, continuing to increase $\lambda$ degrades the precision. This is because if $\lambda$ is too large, the diversify regularizer dominates the distance loss and the resultant distance metric is not tailored to the distance supervision and loses effectiveness in measuring distances.

Figure \ref{fig:paratune_k} shows how AP varies as $K$ increases on the 20-News, 15-Scenes and 6-Activities dataset respectively. The curves correspond to different $\lambda$. When $K$ is small, the average precision is low. This is because a small amount of latent factors are insufficient to capture the inherent complex pattern behind data, hence lacking the capability to effectively measure distances. As $K$ increases, the model capacity increases and the AP increases accordingly. However, further increasing $K$ causes performance to drop. This is because a larger $K$ incurs higher risk of overfitting to training data.

\begin{table}[t]
\centering
\begin{tabular}{|c|c|c|c|}
\hline
& 20-News&15-Scenes&6-Activities\\
\hline
EUC&42.6&42.5&88.7\\
DML \citep{xing2002distance}&57.5&51.7&95.1\\
LMNN \citep{weinberger2005distance}&60.6&53.5&91.5\\
ITML \citep{davis2007information}&50.9&51.9&93.5\\
DML-eig \citep{ying2012distance}&39.2&33.1&82.3\\
Seraph \citep{niu2012information}&\tb{67.9}&55.2&91.4\\
MAR-DML &64.5&\tb{58.8}&\tb{96.5}\\
\hline
\end{tabular}
\caption{3-NN accuracy (\%) on three datasets}\label{table:3nn_cmp}
\end{table}

\begin{figure*}[t]
\begin{center}
\centerline{\includegraphics[width=0.32\columnwidth]{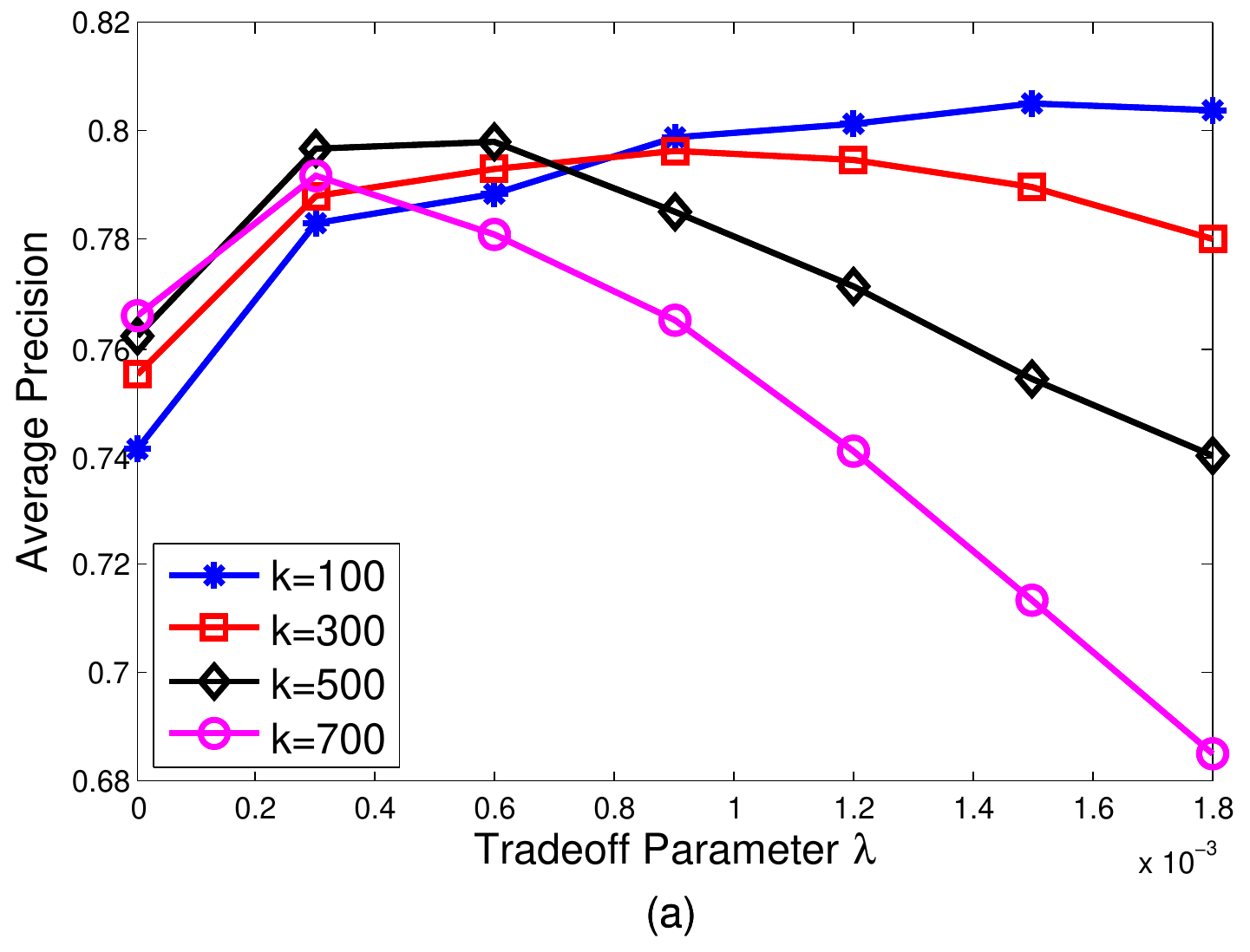}
\includegraphics[width=0.32\columnwidth]{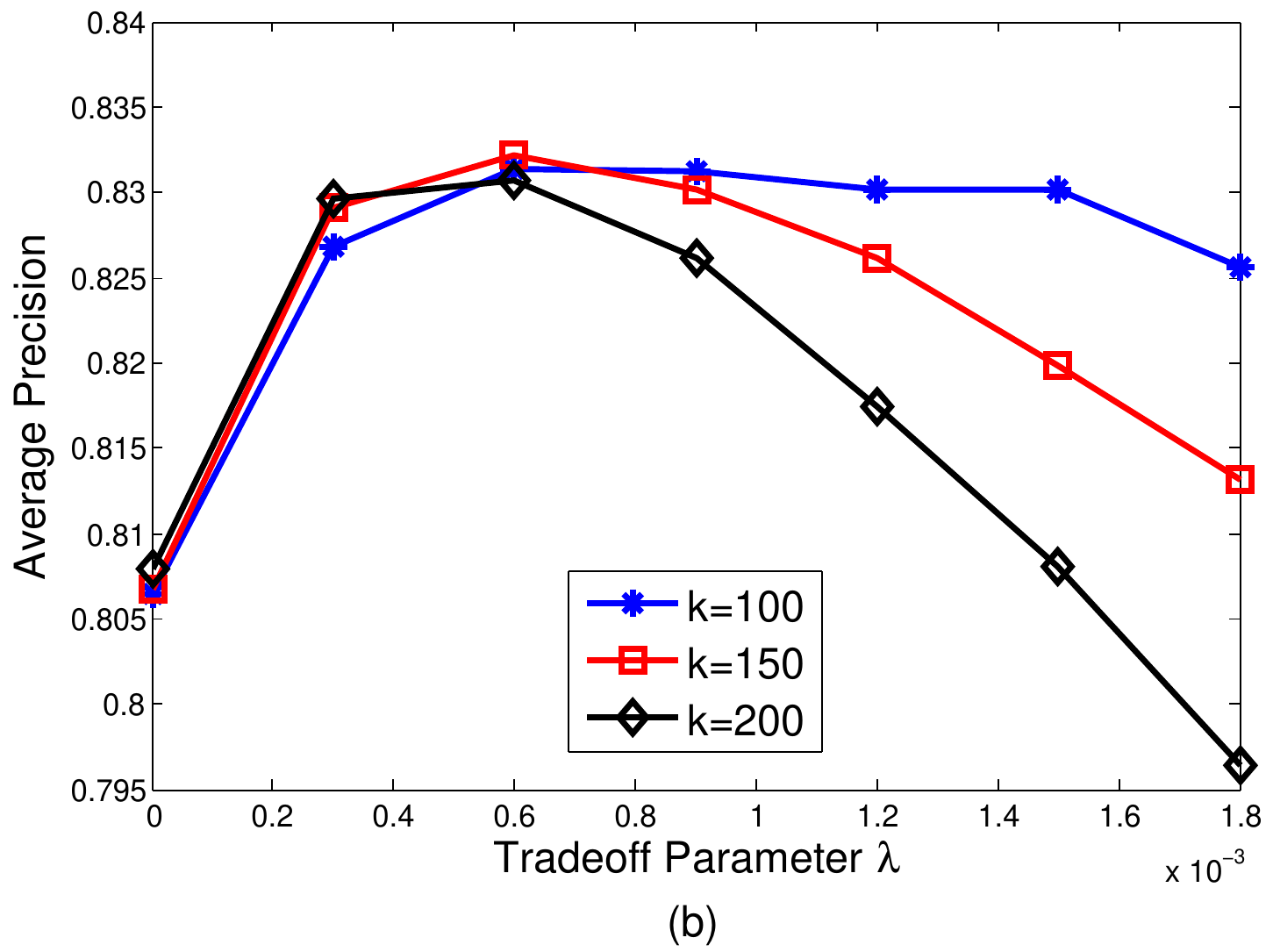}
\includegraphics[width=0.32\columnwidth]{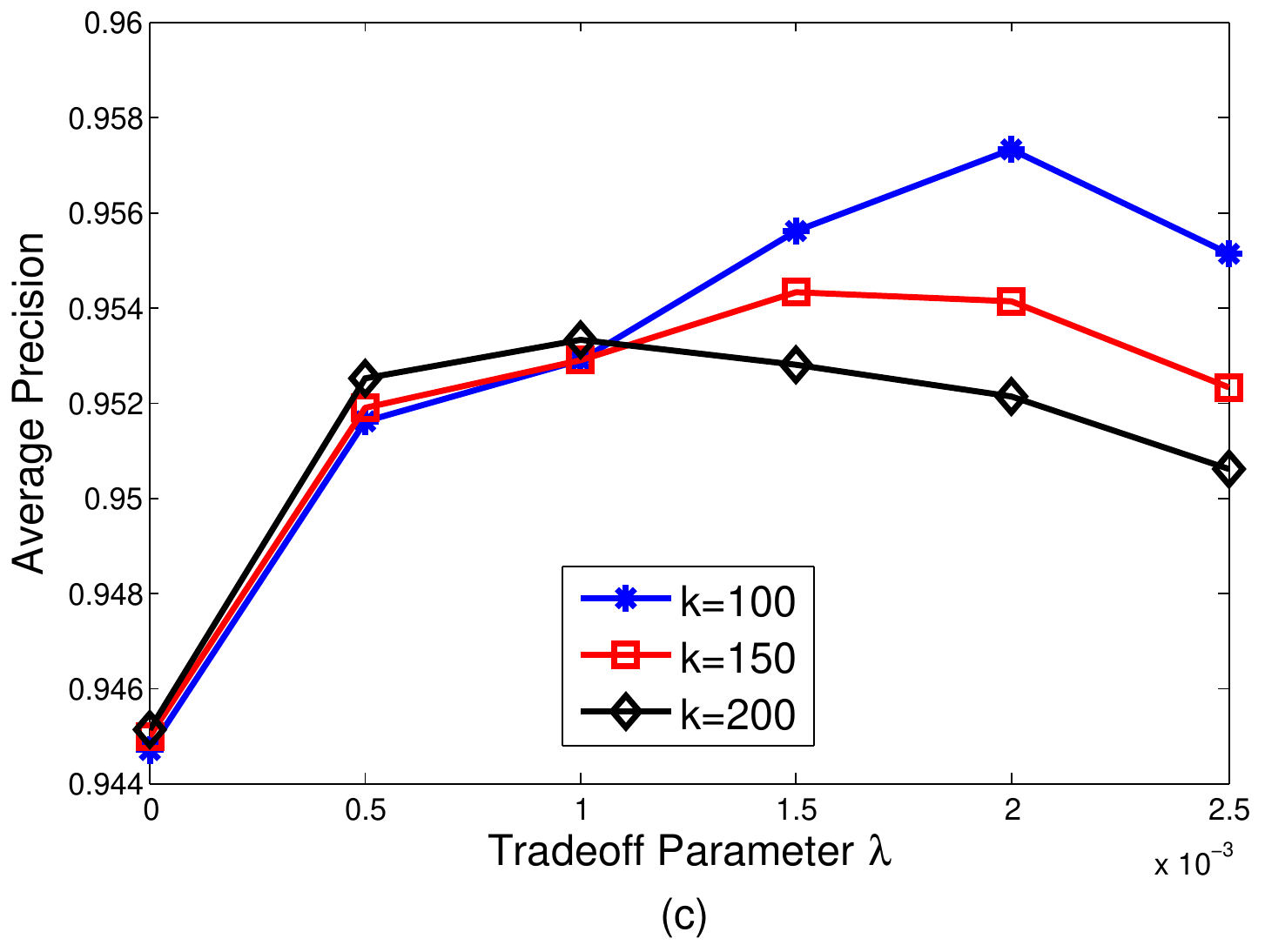}
}
\caption{Sensitivity of MAR-DML to the tradeoff parameter $\lambda$ on (a) 20-News dataset (b) 15-Scenes dataset (c) 6-Activities dataset }
\label{fig:paratune_lambda}
\end{center}
\end{figure*}

\begin{figure*}[t]
\begin{center}
\centerline{\includegraphics[width=0.32\columnwidth]{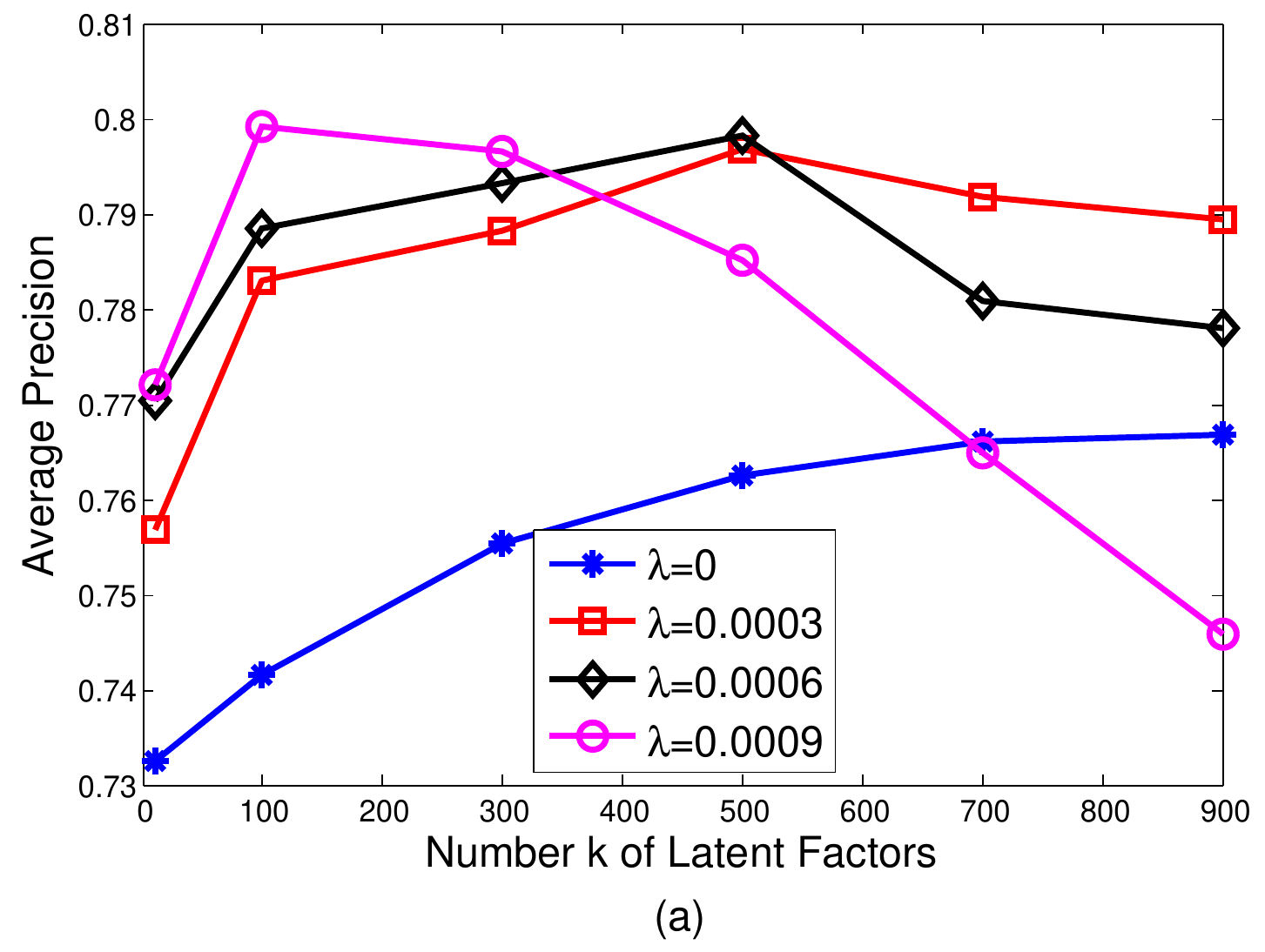}
\includegraphics[width=0.32\columnwidth]{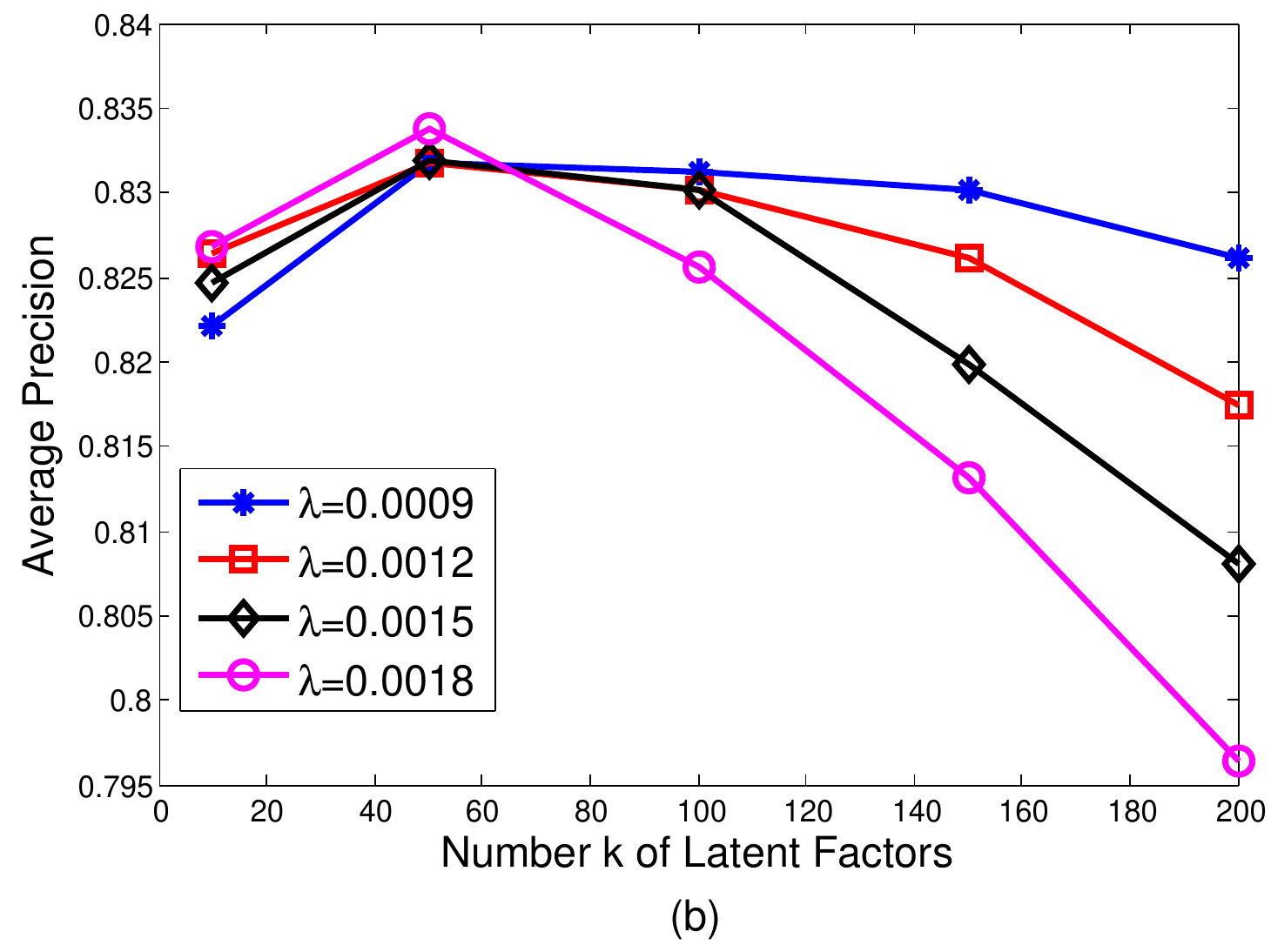}
\includegraphics[width=0.32\columnwidth]{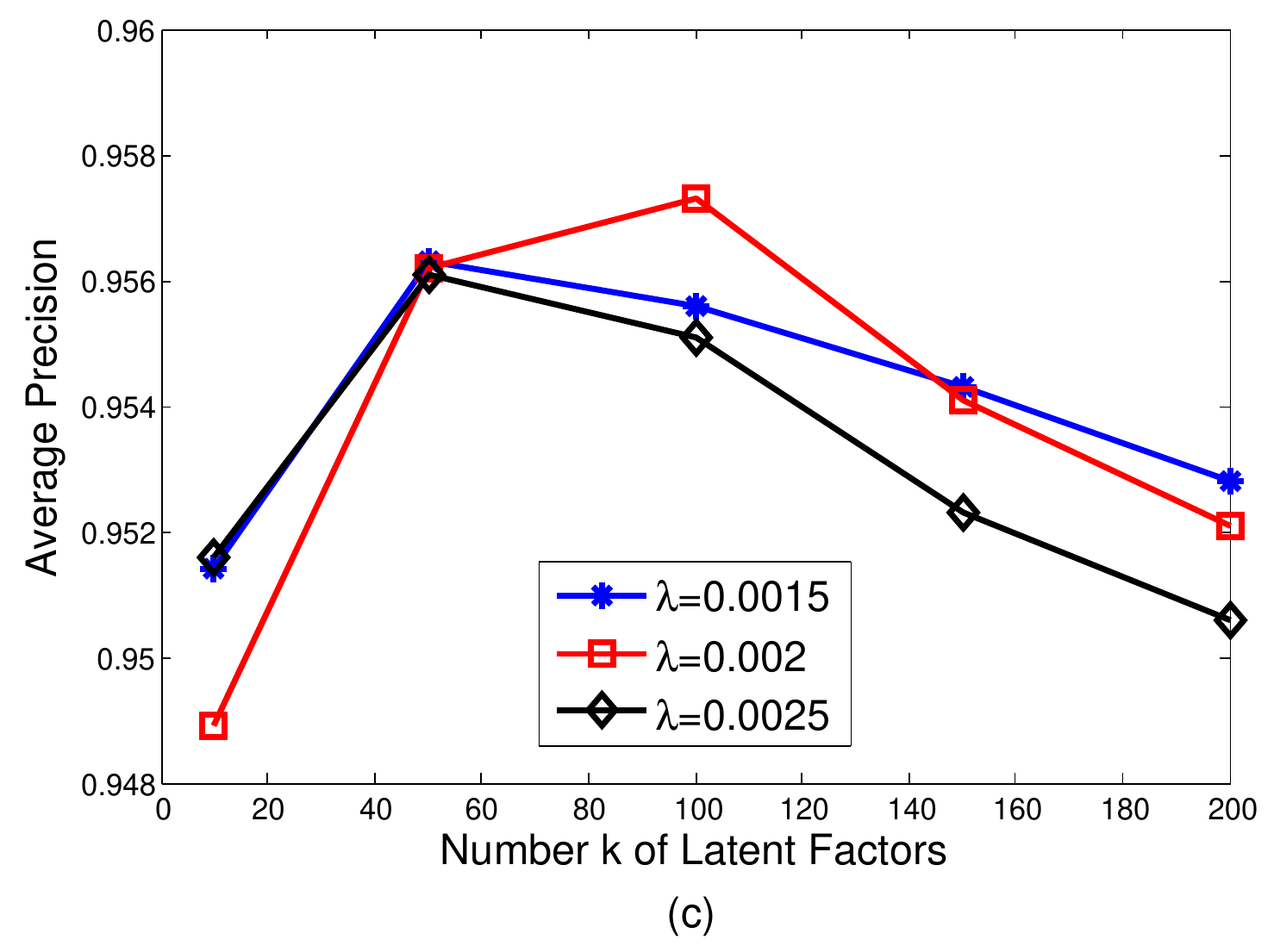}
}
\caption{Sensitivity of MAR-DML to the number of latent factors $k$ on (a) 20-News dataset (b) 15-Scenes dataset (c) 6-Activities dataset}
\label{fig:paratune_k}
\end{center}
\end{figure*}


\begin{figure*}[t]
\begin{center}
\includegraphics[width=0.24\columnwidth]{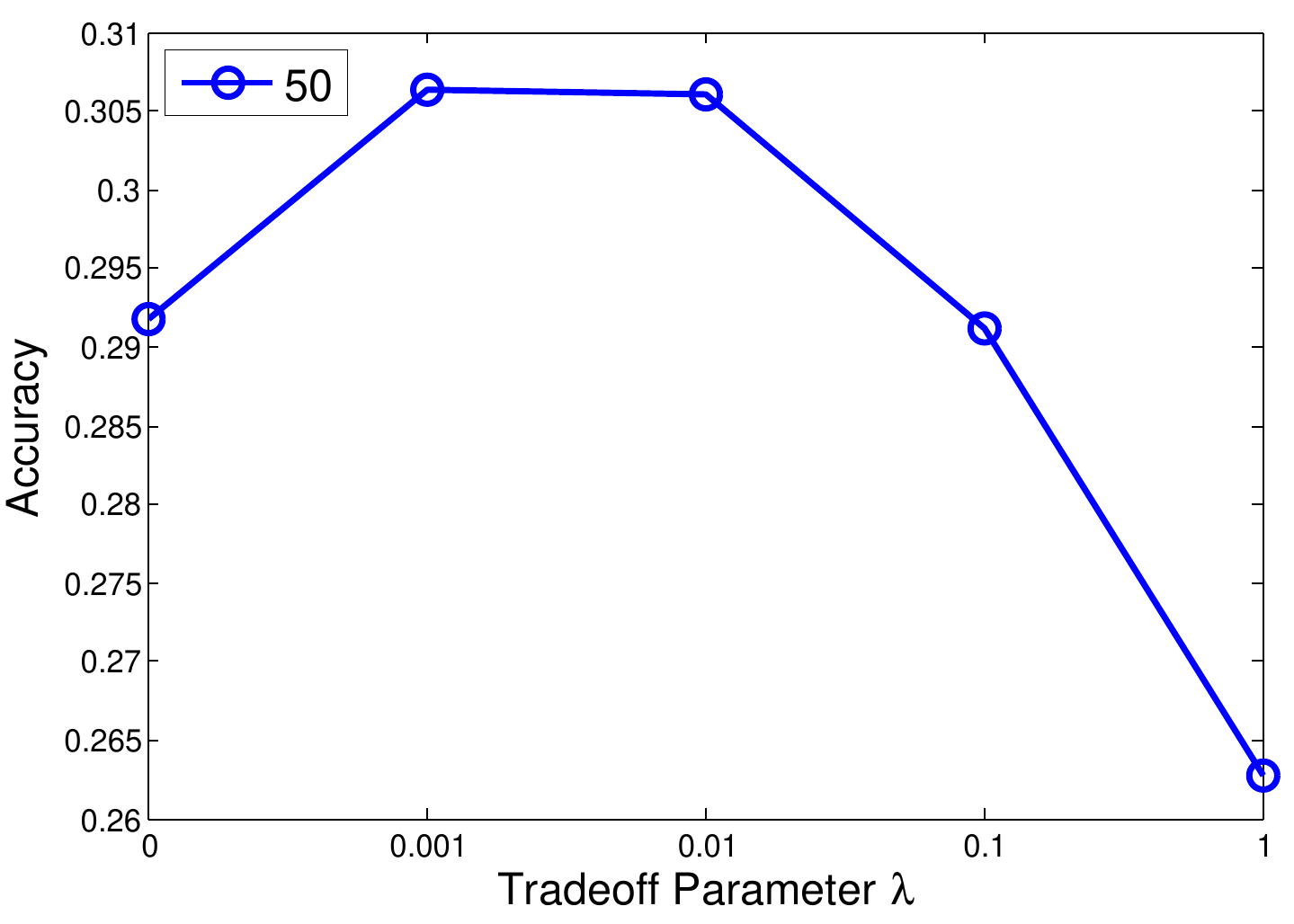}
\includegraphics[width=0.24\columnwidth]{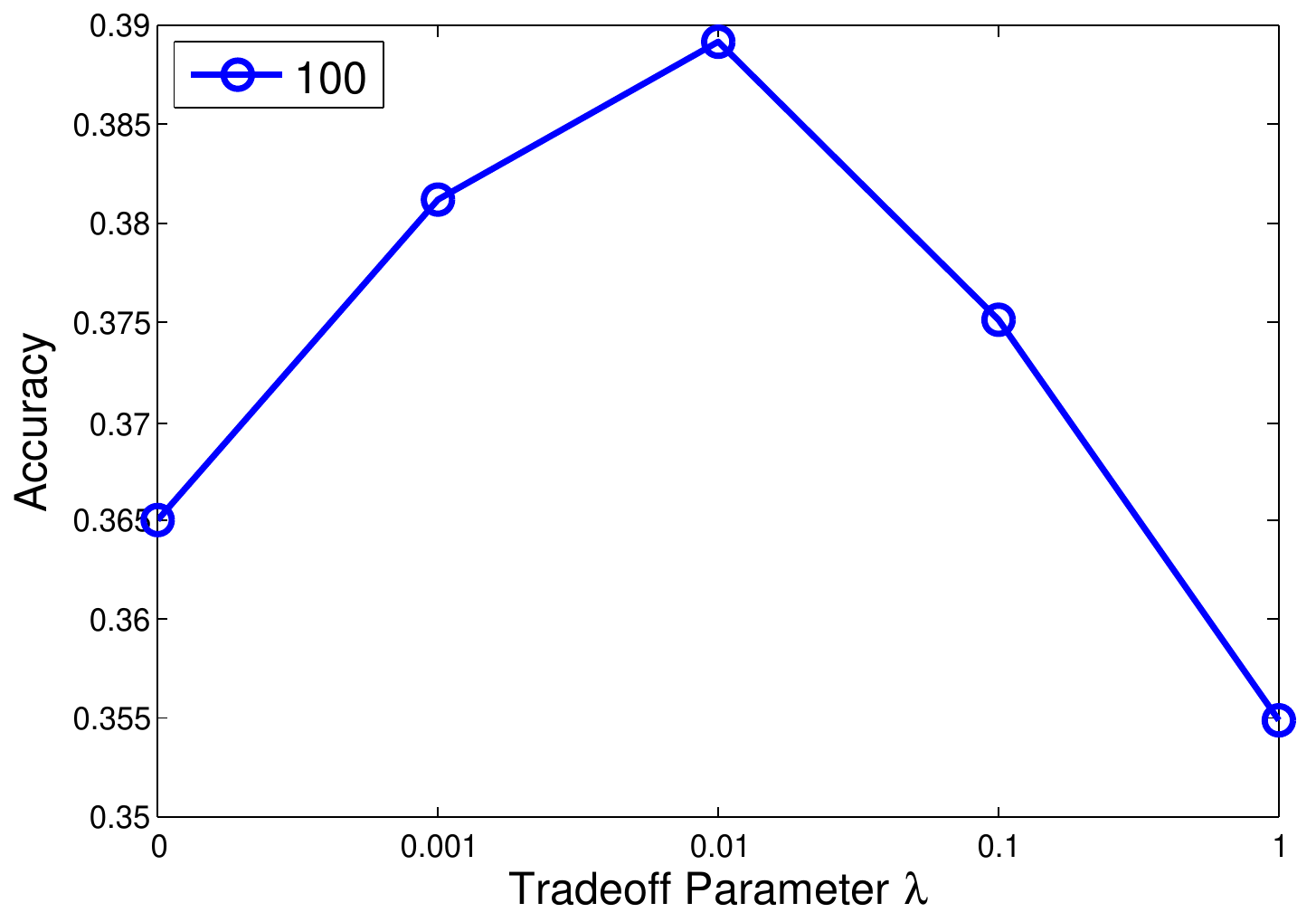}
\includegraphics[width=0.24\columnwidth]{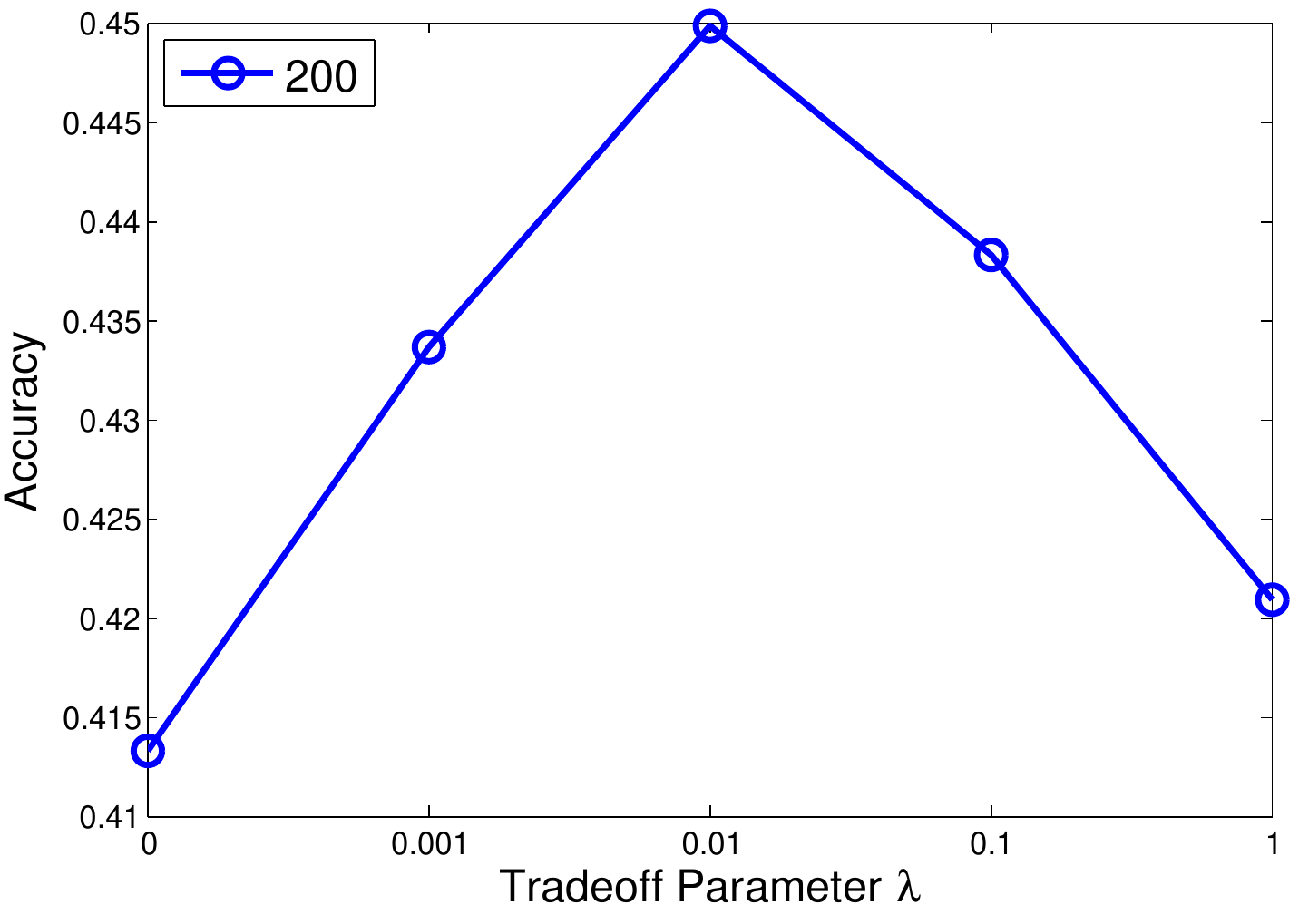}
\includegraphics[width=0.24\columnwidth]{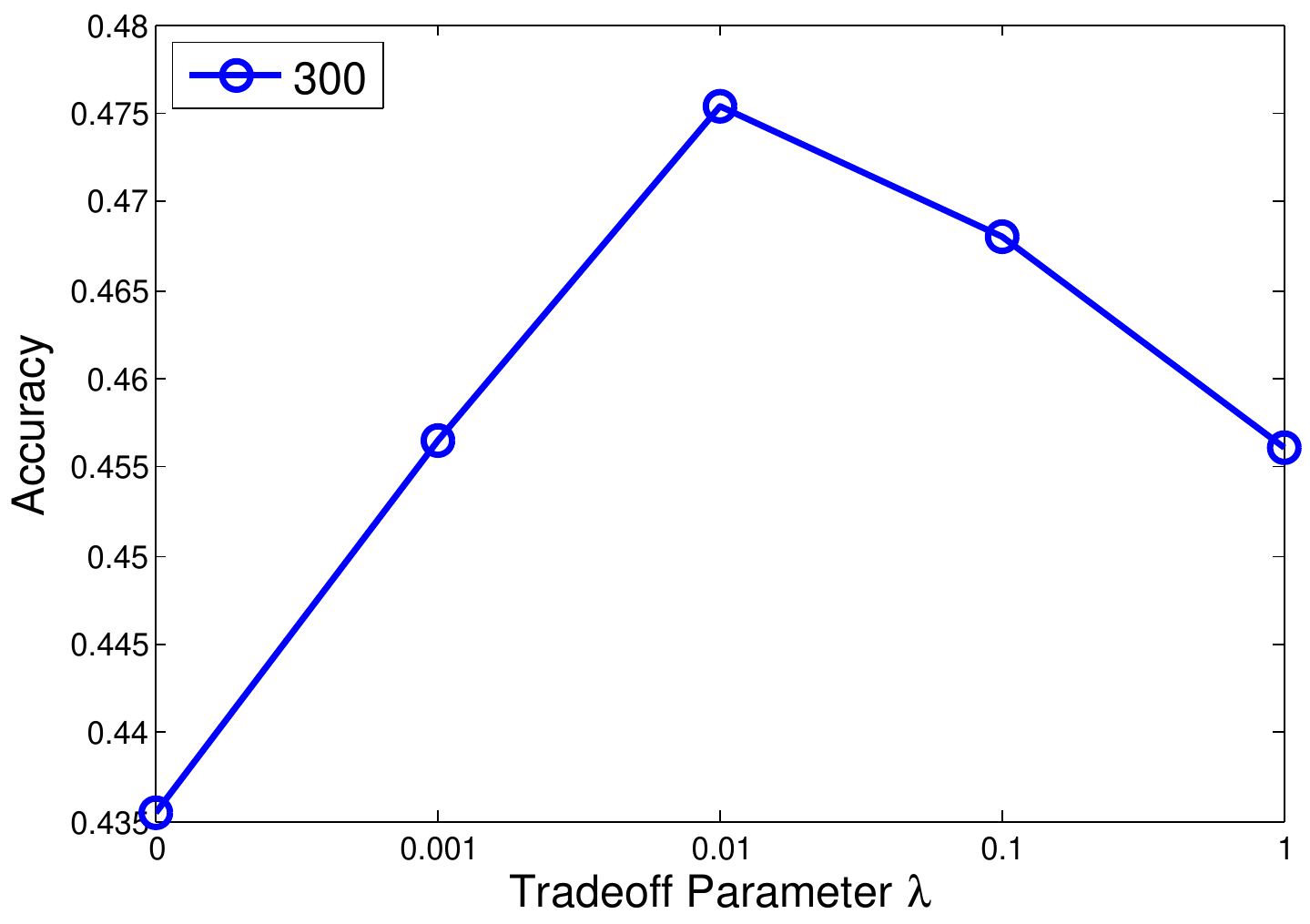}
\caption{Test accuracy versus $\lambda$ for neural networks with one hidden layer.}
\label{fig:acc_lambda_1l_timit}
\end{center}
\end{figure*}

\subsection{Neural Networks with Mutual Angular Regularization}
We conduct empirical study of MAR-NN, to check whether the empirical results are aligned with the theoretical analysis. 
Specifically, we are interested in how the performance of neural networks varies as the tradeoff parameter $\lambda$ in MAR-NN increases. A larger $\lambda$ would induce a stronger regularization, which generates a larger angle lower bound $\theta$. We apply MAR-NN for phoneme classification \citep{mohamed2011deep} on the TIMIT\footnote{\url{https://catalog.ldc.upenn.edu/LDC93S1}} speech dataset.
The inputs are MFCC features extracted with context windows and the outputs are class labels generated by the HMM-GMM model through forced alignment \citep{mohamed2011deep}. The feature dimension is 360 and the number of classes is 2001. There are 1.1 million data instances in total. We use 70\% data for training and 30\% for testing. The activation function is sigmoid and loss function is cross-entropy. The networks are trained with stochastic gradient descent and the minibatch size is 100.

Figure \ref{fig:acc_lambda_1l_timit} shows the testing accuracy versus the tradeoff parameter $\lambda$ achieved by four neural networks with one hidden layer. The number of hidden units varies in $\{50,100,200,300\}$.
As can be seen from these figures, under various network architectures, the best accuracy is achieved under a properly chosen $\lambda$. For example, for the neural network with 100 hidden units, the best accuracy is achieved when $\lambda=0.01$. These empirical observations are aligned with our theoretical analysis that the best generalization performance is achieved under a proper diversity level. Adding this regularizer greatly improves the performance of neural networks, compared with unregularized NNs. For example, in a NN with 200 hidden units, the mutual angular regularizer improves the accuracy from $\sim$0.415 (without regularization) to 0.45.

\section{Conclusions}
In this paper, we propose a diversity-inducing regularization method for latent variable models (LVMs), which encourage the components in LVMs to be diverse to address several issues involved in latent variable modeling: (1) how to capture long-tail patterns underlying data; (2) how to reduce model complexity without sacrificing expressivity; (3) how to improve the interpretability of learned patterns. We begin with defining a mutual angular regularizer (MAR) which encourages the components in LVMs to have larger mutual angles. The MAR is employed to regularize LVMs, which results in non-convex and non-smooth problems that are not amenable for optimization. To cope with this issue, we derive a smooth lower bound to surrogate the MAR and prove that the monotonicity of the lower bound is closely aligned with the MAR. We study the effectiveness of MAR both theoretically and empirically. In theory, we analyze how MAR affects the generalization performance of supervised LVMs, using neural network as a specific instance. It is shown that increasing the MAR can reduce estimation error and would increase approximation error. Empirically, we demonstrate that the MAR can greatly improve the performance of LVMs, using restricted Boltzmann machine and distance metric learning as model instances. In particular, the experiments corroborate that MAR is able to capture long-tail patterns, improve interpretability and reduce model complexity without compromising expressivity.



\section{Acknowledgments}

The authors would like to thank the anonymous reviewers for the helpful suggestions.
This work is supported by the following grants to Eric P. Xing:
ASFOR FA95501010247; NSF IIS1111142, IIS447676.

\appendix
\section{Proof of Lemmas in Optimization}
\label{apd:opt}

\setcounter{mylemma}{0}

\begin{mylemma}
Let the parameter vector $\mb{\tilde{a}}_{i}$ of component $i$ be decomposed into $\mb{\tilde{a}}_{i}=\mb{x}_{i}+l_i\mb{e}_{i}$, where $\mb{x}_{i}=\sum_{j=1,j\neq i}^{K}\alpha_{j}\mb{\tilde{a}}_{j}$ lies in the subspace $L$ spanned by $\{\mb{\tilde{a}}_{1},\cdots,\mb{\tilde{a}}_{K}\}\backslash\{\mb{\tilde{a}}_{i}\}$, $\mb{e}_{i}$ is in the orthogonal complement of $L$, $\|\mb{e}_i\|=1$, $\mb{e}_{i}\cdot \mb{\tilde{a}}_{i}>0$, $l_i$ is a scalar. Then $\textrm{det}(\mathbf{\widetilde{A}}^\mathsf{T} \mathbf{\widetilde{A}}) = \textrm{det}(\mathbf{\widetilde{A}}_{-i}^\mathsf{T} \mathbf{\widetilde{A}}_{-i}) (l_i\mb{e}_{i}\cdot \mathbf{\tilde{a}}_i)$, where $\mb{\widetilde{A}}_{-i}=[\mb{\tilde{a}}_{1},\cdots,\mb{\tilde{a}}_{i-1},\mb{\tilde{a}}_{i+1},\cdots,\mb{\tilde{a}}_{K}]$ with $\mb{\tilde{a}}_{i}$ excluded..
\end{mylemma}
\begin{proof}
Part of the proof follows \citep{shafarevich2012linear}.
\begin{equation}
\det(\mb{\widetilde{A}}^{\mathsf{T}}\mb{\widetilde{A}})=
\begin{vmatrix}
\mb{\tilde{a}}_1\cdot\mb{\tilde{a}}_1& \cdots &\mb{\tilde{a}}_1\cdot\mb{\tilde{a}}_i & \cdots &\mb{\tilde{a}}_{1}\cdot\mb{\tilde{a}}_K \\
\mb{\tilde{a}}_2\cdot\mb{\tilde{a}}_1& \cdots & \mb{\tilde{a}}_2\cdot\mb{\tilde{a}}_i & \cdots &\mb{\tilde{a}}_2\cdot\mb{\tilde{a}}_K \\
\vdots& \ddots & \vdots &\ddots &\vdots \\
\mb{\tilde{a}}_K\cdot\mb{\tilde{a}}_1& \cdots & \mb{\tilde{a}}_K\cdot\mb{\tilde{a}}_i & \cdots &\mb{\tilde{a}}_K\cdot\mb{\tilde{a}}_K\\
\end{vmatrix}
\end{equation}

Let $\mb{c}_{i}$ denote the $i$th column of the Gram matrix $\mb{\widetilde{A}}^{\mathsf{T}}\mb{\widetilde{A}}$. Subtracting $\sum_{j=1,j\neq i}^{k}\alpha_{j}\mb{c}_{j}$ from $\mb{c}_{i}$ \citep{shafarevich2012linear}, where $\alpha_{j}$ is the linear coefficient in $\mb{x}_{i}$, we get
\begin{equation}
\det(\mb{\widetilde{A}}^{\mathsf{T}}\mb{\widetilde{A}})=
\begin{vmatrix}
\mb{\tilde{a}}_1\cdot\mb{\tilde{a}}_1& \cdots &0 & \cdots &\mb{\tilde{a}}_1\cdot\mb{\tilde{a}}_K \\
\mb{\tilde{a}}_2\cdot\mb{\tilde{a}}_1& \cdots & 0 & \cdots &\mb{\tilde{a}}_2\cdot\mb{\tilde{a}}_K \\
\vdots& \ddots & \vdots &\ddots &\vdots \\
\mb{\tilde{a}}_i\cdot\mb{\tilde{a}}_i& \cdots & l_i\mb{e}_{i}\cdot\mb{\tilde{a}}_i & \cdots &\mb{\tilde{a}}_i\cdot\mb{\tilde{a}}_K \\
\vdots& \ddots & \vdots & \ddots &\vdots \\
\mb{\tilde{a}}_K\cdot\mb{\tilde{a}}_1& \cdots & 0 & \cdots &\mb{\tilde{a}}_K\cdot\mb{\tilde{a}}_K \\
\end{vmatrix}
\end{equation}
Expanding the determinant according to the $i$th column, we get $\textrm{det}(\mathbf{\widetilde{A}}^\mathsf{T} \mathbf{\widetilde{A}}) = \textrm{det}(\mathbf{\widetilde{A}}_{-i}^\mathsf{T} \mathbf{\widetilde{A}}_{-i}) (l_i\mb{e}_{i}\cdot \mathbf{\tilde{a}}_i)$.
\end{proof}

\begin{mylemma}
Let the parameter vector $\mb{\tilde{a}}_{i}$ of component $i$ be decomposed into $\mb{\tilde{a}}_{i}=\mb{x}_{i}+l_i\mb{e}_{i}$, where $\mb{x}_{i}=\sum_{j=1,j\neq i}^{K}\alpha_{j}\mb{\tilde{a}}_{j}$ lies in the subspace $L$ spanned by $\{\mb{\tilde{a}}_{1},\cdots,\mb{\tilde{a}}_{K}\}\backslash\{\mb{\tilde{a}}_{i}\}$, $\mb{e}_{i}$ is in the orthogonal complement of $L$, $\|\mb{e}_i\|=1$, $\mb{e}_{i}\cdot \mb{\tilde{a}}_{i}>0$, $l_i$ is a scalar. Then the gradient of $\Gamma(\mb{\widetilde{A}})$ w.r.t $\mb{a}_{i}$ is $k_{i}\mb{e}_{i}$, where $k_{i}$ is a positive scalar.
\end{mylemma}
\begin{proof}
According to chain rule, the gradient of $\Gamma(\mb{\widetilde{A}})$ w.r.t $\mb{\tilde{a}_{i}}$ can be written as $g'(\det(\mb{\widetilde{A}}^{\mathsf{T}}\mb{\widetilde{A}}))\\\frac{\partial \det(\mb{\widetilde{A}}^{\mathsf{T}}\mb{\widetilde{A}})}{\partial \mb{\tilde{a}_{i}}}$, where $g(x)=\arcsin(\sqrt{x})-(\frac{\pi}{2}-\arcsin(\sqrt{x}))^2$. It is easy to check that $g(x)$ is an increasing function and $g'(x)>0$. Now we discuss the $\frac{\partial \det(\mb{\widetilde{A}}^{\mathsf{T}}\mb{\widetilde{A}})}{\partial \mb{\tilde{a}_{i}}}$ term. According to Lemma \ref{lem:expand}, we have $\det(\mb{\widetilde{A}}^{\mathsf{T}}\mb{\widetilde{A}})=\det(\mb{\widetilde{A}}_{-i}^{\mathsf{T}}\mb{\widetilde{A}}_{-i}) l_i\mb{e}_{i}\cdot\mb{\tilde{a}_{i}}$. From this equation, we have $\frac{\partial \det(\mb{\widetilde{A}}^{\mathsf{T}}\mb{\widetilde{A}})}{\partial \mb{\tilde{a}_{i}}}=\det(\mb{\widetilde{A}}_{-i}^{\mathsf{T}}\mb{\widetilde{A}}_{-i})l_i\mb{e}_{i}$. As assumed earlier, the component vectors in $\mb{\widetilde{A}}$ are linearly independent and hence $\det(\mb{\widetilde{A}}^{\mathsf{T}}\mb{\widetilde{A}})>0$ and $\det(\mb{\widetilde{A}}_{-i}^{\mathsf{T}}\mb{\widetilde{A}}_{-i})>0$. From $\det(\mb{\widetilde{A}}^{\mathsf{T}}\mb{\widetilde{A}})=\det(\mb{\widetilde{A}}_{-i}^{\mathsf{T}}\mb{\widetilde{A}}_{-i}) l_i\mb{e}_{i}\cdot\mb{\tilde{a}}_{i}$ and $\mb{e}_{i}\cdot \mb{\tilde{a}}_{i}>0$, we know $l_i>0$. Overall, the gradient of $\Gamma(\mb{\widetilde{A}})$ w.r.t $\mb{\tilde{a}_{i}}$ can be written as $k_{i}\mb{e}_{i}$, where $k_{i}=g'(\det(\mb{\widetilde{A}}^{\mathsf{T}}\mb{\widetilde{A}}))\det(\mb{\widetilde{A}}_{-i}^{\mathsf{T}}\mb{\widetilde{A}}_{-i}) l_i>0$. The proof completes.
\end{proof}
\begin{mylemma}
$\forall (i,j) \in V$, we have $\theta(\mb{\tilde{a}}_{i}^{(t+1)}, \mb{\tilde{a}}_{j}^{(t+1)}) - \theta(\mb{\tilde{a}}_{i}^{(t)}, \mb{\tilde{a}}_{j}^{(t)}) = o(\eta)$, where $\lim\limits_{\eta \to 0} \frac{o(\eta)}{\eta} = 0$.
\end{mylemma}
\begin{proof}
For $(i,j) \in V$, $\mb{\tilde{a}}_{i}^{(t)} \cdot \mb{\tilde{a}}_{j}^{(t)} = 0$, thereby $x_{ij}^{(t)} =0$ and $\theta(\mb{\tilde{a}}_{i}^{(t+1)}, \mb{\tilde{a}}_{j}^{(t+1)}) - \theta(\mb{\tilde{a}}_{i}^{(t)}, \mb{\tilde{a}}_{j}^{(t)}) =\textrm{arccos}(x_{ij}^{(t+1)}) - \textrm{arccos}(x_{ij}^{(t)}) = \textrm{arccos}(x_{ij}^{(t+1)}) - \frac{\pi}{2}$. Now we prove $\lim\limits_{\eta \to 0} \frac{\textrm{arccos}(x_{ij}^{(t+1)}) - \frac{\pi}{2}}{\eta} = 0$. Plugging in $\mb{\tilde{a}}_{i}^{(t)} \cdot \mb{\tilde{a}}_{j}^{(t)} = 0$ into $x_{ij}^{(t+1)}$, we have $x_{ij}^{(t+1)} = \frac{| \eta^2 k_i k_j \mb{e}_i^{(t)}\cdot \mb{e}_j^{(t)}|}{ \sqrt{1+2 l_i k_i\eta + k_i^2 \eta^2} \sqrt{1 + 2 l_j k_j\eta + k_j^2 \eta^2}}$. Thereby $\lim\limits_{\eta \to 0} \frac{x_{ij}^{(t+1)}}{\eta}= 0$ (equivalently, $x_{ij}^{(t+1)} = o(\eta)$) and $\lim\limits_{\eta \to 0} x_{ij}^{(t+1)}= 0$. According to the Taylor expansion of $\textrm{arccos}(x)$ at $x=0$, $\textrm{arccos}(x) = \frac{\pi}{2} - x + o(x)$, so $\lim\limits_{x \to 0} \frac{\textrm{arccos}(x)-\frac{\pi}{2}}{x}=-1$.
Since $\lim\limits_{\eta \to 0} x_{ij}^{(t+1)}= 0$, $\lim\limits_{\eta \to 0} \frac{\textrm{arccos}(x_{ij}^{(t+1)})-\frac{\pi}{2}}{x_{ij}^{(t+1)}}=\lim\limits_{x_{ij}^{(t+1)} \to 0} \frac{\textrm{arccos}(x_{ij}^{(t+1)})-\frac{\pi}{2}}{x_{ij}^{(t+1)}} = -1$. Since $\lim\limits_{\eta \to 0} \frac{x_{ij}^{(t+1)}}{\eta} = 0$, we have $\lim\limits_{\eta \to 0} \frac{\textrm{arccos}(x_{ij}^{(t+1)}) - \frac{\pi}{2}}{\eta} = \lim\limits_{\eta \to 0} \frac{\textrm{arccos}(x_{ij}^{(t+1)}) - \frac{\pi}{2}}{x_{ij}^{(t+1)}} \frac{x_{ij}^{(t+1)}}{\eta} = 0$. The proof completes.
\end{proof}

\begin{mylemma}
$\forall (i,j) \in N$, $\exists c_{ij} > 0$, such that $\theta(\mb{\tilde{a}}_{i}^{(t+1)}, \mb{\tilde{a}}_{j}^{(t+1)}) - \theta(\mb{\tilde{a}}_{i}^{(t)}, \mb{\tilde{a}}_{j}^{(t)}) = c_{ij} \eta + o(\eta)$, where $\lim\limits_{\eta \to 0} \frac{o(\eta)}{\eta} = 0$.
\end{mylemma}
\begin{proof}
Using the Taylor expansion of $\textrm{arccos}(x)$ at $x = x_{ij}^{(t)}$, we have
\begin{equation}
\begin{array}{l}
\textrm{arccos}(x_{ij}^{(t+1)}) - \textrm{arccos}(x_{ij}^{(t)})
= -\frac{1}{\sqrt{1 - {x_{ij}^{(t)}}^2}} (x_{ij}^{(t+1)} - x_{ij}^{(t)}) + o(x_{ij}^{(t+1)} - x_{ij}^{(t)})
\end{array}
\end{equation}
According to the definition of $x_{ij}^{(t+1)}$, we have
\begin{equation}
\begin{array}{l}
x_{ij}^{(t+1)} = x_{ij}^{(t)} \frac{|1+ \eta^2 \frac{k_i k_j \mb{e}_i^{(t)}\cdot \mb{e}_j^{(t)}}{\mb{\tilde{a}}_i^{(t)} \cdot\mb{\tilde{a}}_j^{(t)}}|}{\sqrt{1+2 l_i k_i\eta + k_i^2 \eta^2} \sqrt{1 + 2 l_j k_j\eta + k_j^2 \eta^2}}
\end{array}
\end{equation}
Using the Taylor expansion of $\frac{1}{\sqrt{1+x}}$ at $x=0$, we can obtain that $\frac{1}{\sqrt{1+2l_i k_i\eta + k_i^2\eta^2}} = 1 - \frac{1}{2}(2l_i k_i\eta + k_i^2\eta^2) + o(2l_i k_i\eta + k_i^2\eta^2)$. As $\eta^2 = o(\eta)$ and $o(2l_i k_i\eta + k_i^2\eta^2) = o(\eta)$, we can obtain that $\frac{1}{\sqrt{1+2l_i k_i\eta + k_i^2\eta^2}} = 1 - l_i k_i\eta + o(\eta)$. Similarly, $\frac{1}{\sqrt{1+2l_j k_j\eta + k_j^2\eta^2}} = 1 - l_j k_j\eta + o(\eta)$. When $\eta \to 0$, $|1+\eta^2 \frac{k_i k_j \mb{e}_i^{(t)}\cdot \mb{e}_j^{(t)}}{\mb{\tilde{a}}_i^{(t)} \cdot\mb{\tilde{a}}_j^{(t)}}| = 1$.
Thereby when $\eta$ is small enough, $|1+\eta^2 \frac{k_i k_j \mb{e}_i^{(t)}\cdot \mb{e}_j^{(t)}}{\mb{\tilde{a}}_i^{(t)} \cdot\mb{\tilde{a}}_j^{(t)}}|=1+\eta^2 \frac{k_i k_j \mb{e}_i^{(t)} \cdot\mb{e}_j^{(t)}}{\mb{\tilde{a}}_i^{(t)}\cdot \mb{\tilde{a}}_j^{(t)}}$, so $|1+\eta^2 \frac{k_i k_j \mb{e}_i^{(t)}\cdot \mb{e}_j^{(t)}}{\mb{a}_i^{(t)} \cdot\mb{a}_j^{(t)}}|=1+o(\eta)$.
Substituting the above equations to $x_{ij}^{(t+1)}$, we can obtain that $x_{ij}^{(t+1)} - x_{ij}^{(t)} = x_{ij}^{(t)}((1+o(\eta))(1-l_i k_i\eta + o(\eta))(1-l_j k_j\eta + o(\eta)) -1) = - x_{ij}^{(t)}(l_i k_i+l_j k_j)\eta + o(\eta)$. So $\lim\limits_{\eta \to 0} \frac{x_{ij}^{(t+1)}-x_{ij}^{(t)}}{\eta} = -x_{ij}^{(t)}(l_i k_i+l_j k_j) $. As $\lim\limits_{\eta \to 0} \frac{o(x_{ij}^{(t+1)} - x_{ij}^{(t)})}{x_{ij}^{(t+1)} - x_{ij}^{(t)}}
= \lim\limits_{x_{ij}^{(t+1)} \to x_{ij}^{(t)}} \frac{o(x_{ij}^{(t+1)} - x_{ij}^{(t)})}{x_{ij}^{(t+1)} - x_{ij}^{(t)}}=0$, we can draw the conclusion that $\lim\limits_{\eta \to 0} \frac{o(x_{ij}^{(t+1)} - x_{ij}^{(t)})}{\eta} = \lim\limits_{\eta \to 0} \frac{o(x_{ij}^{(t+1)} - x_{ij}^{(t)})}{x_{ij}^{(t+1)} - x_{ij}^{(t)}} \frac{x_{ij}^{(t+1)} - x_{ij}^{(t)}}{\eta}=0$, hence $o(x_{ij}^{(t+1)} - x_{ij}^{(t)})=o(\eta)$. So $\textrm{arccos}(x_{ij}^{(t+1)}) - \textrm{arccos}(x_{ij}^{(t)}) = -\frac{1}{\sqrt{1 - {x_{ij}^{(t)}}^2}} (x_{ij}^{(t+1)} - x_{ij}^{(t)}) + o(x_{ij}^{(t+1)} - x_{ij}^{(t)})= -\frac{1}{\sqrt{1 - {x_{ij}^{(t)}}^2}} (-x_{ij}^{(t)} (l_i k_i+l_j k_j)\eta + o(\eta)) + o(\eta) = \frac{x_{ij}^{(t)}(l_i k_i+l_j k_j)}{\sqrt{1 - {x_{ij}^{(t)}}^2}}\eta + o(\eta)$. Let $c_{ij} = \frac{x_{ij}^{(t)}(l_i k_i+l_j k_j)}{\sqrt{1 - {x_{ij}^{(t)}}^2}}$, clearly $c_{ij}>0$. The proof completes.
\end{proof}

\begin{mylemma}
Given a non-decreasing sequence $b=(b_i)_{i = 1}^{n}$ and a strictly decreasing function $g(x)$ which satisfies
$
0 \le g(b_i) \le \textrm{min}\{b_{i+1} - b_i: i = 1, 2, \cdots, n-1, b_{i+1} \ne b_i\}
$,
we define a sequence $c=(c_i)_{i = 1}^{n}$ where $c_i = b_i + g(b_i)$. If $b_1 < b_n$, then $\textrm{var}(c) < \textrm{var}(b)$, where $var(\cdot)$ denotes the variance of a sequence. Furthermore, let $n' = \textrm{max}\{j: b_j \ne b_n\}$, we define a sequence $b'=(b'_i)_{i = 1}^{n}$ where $b_i' = b_i + g(b_n) + (g(b_{n'}) - g(b_n))\mathbb{I}(i \le n')$ and $\mathbb{I}(\cdot)$ is the indicator function, then $\textrm{var}(c) \le \textrm{var}(b') < \textrm{var}(b)$.
\end{mylemma}

\begin{proof}
The intuition behind the proof is that we can view the difference between corresponding elements of the sequence $b$ and sequence $c$ as "updates", and we can find that the updates lead to smaller elements "catch up" larger elements. Alternatively, we can obtain the new sequence $c$ through a set of updates: First, we update the whole sequence $b$ by the update value of the largest elements, then the largest elements have found their correct values. Then we pick up the elements that are smaller than the largest elements, and update those by the update value of the second largest elements minus the previous update, then the second largest elements have found their correct values. In this manner, we can obtain a sequence of sequences, where the first sequence is $b$, the third sequence is $b'$, the last sequence is $c$, and the adjacent sequences only differ by a simpler update: to the left of some element, each element is updated by a same value; and to the right of the element, each value remains. We can prove that such simpler update can guarantee decreasing of the variance under certain conditions, and we can use that to prove $\textrm{var}(c)\le \textrm{var}(b') < \textrm{var}(b)$.

The formal proof starts here: First, following the intuition stated above, we construct a sequence of sequences with decreasing variance, in which the variance of the first sequence is $\textrm{var}(b)$ and the variance of the last sequence is $\textrm{var}(c)$. We sort the unique values in $b$ in ascending order and denote the resultant sequence as $d=(d_j)_{j=1}^{m}$. Let $l(j) = \textrm{max}\{i:b_i = d_{j}\}$, $u(i) = \{j:d_j=b_i\}$, we construct a sequence of sequences $h^{(j)}=(h_i^{(j)})_{i=1}^{n}$ where $j=1,2,\cdots,m+1$, in the following way:
\begin{itemize}
\item $h_i^{(1)} = b_i$, where $i = 1,2,\cdots,n$
\item $h_i^{(j+1)} = h_i^{(j)}$, where $j = 1,2,\cdots,m$ and $l(m-j+1)< i \le n$
\item $h_i^{(2)} = h_i^{(1)} + g(d_m)$, where $1 \le i \le l(m)$
\item $h_i^{(j+1)} = h_i^{(j)} + g(d_{m-j+1}) - g(d_{m-j+2})$, where $j = 2,3,\cdots,m$ and $1 \le i \le l(m-j+1)$.
\end{itemize}

From the definition of $h^{(j)}$, we know $\textrm{var}(h^{(1)}) = \textrm{var}(b)$. As $b_1 < b_n$, we have $m \ge 2$. Now we prove that $\textrm{var}(h^{(m+1)}) = \textrm{var}(c)$ and $\forall j = 1, 2, \cdots, m, \, \textrm{var}(h^{(j+1)}) < \textrm{var}(h^{(j)})$.

First, we prove $\textrm{var}(h^{(m+1)}) = \textrm{var}(c)$. Actually, we can prove $h^{(m+1)}=c$: \\
\begin{align*}
h_i^{(m+1)} &= \sum_{j=1}^m (h_i^{(j+1)} - h_i^{(j)}) + h_i^{(1)} \\
&= \sum_{j=1}^{m+1-u(i)} (h_i^{(j+1)} - h_i^{(j)}) + \sum_{j=m+2-u(i)}^{m} (h_i^{(j+1)} - h_i^{(j)}) + b_i
\end{align*}
As $j \ge m+2 - u(i) \Longleftrightarrow u(i) \le m+2-j \Longleftrightarrow d_{m+2-j} \ge d_{u(i)} = b_i \Longleftrightarrow l(m+1-j) < i$, we know that
$$
h_i^{(j+1)}=\left\{
\begin{aligned}
&h_i^{(j)} \textrm{, when } j \ge m+2-u(i)\\
&h_i^{(j)} + g(d_{m-j+1})-g(d_{m-j+2}) \textrm{, when } 2 \le j \le m+1-u(i)\\
&h_i^{(j)} + g(d_m) \textrm{, when } j=1
\end{aligned}
\right.
$$
So we have
\begin{align*}
h_i^{(m+1)} & = \sum_{j=1}^{m+1-u(i)} (h_i^{(j+1)} - h_i^{(j)}) + b_i \\
&= g(d_m) + \sum_{j=2}^{m+1-u(i)} (g(d_{m-j+1})-g(d_{m-j+2})) + b_i \\
&= g(d_m)+g(d_{u(i)})-g(d_{m})+b_i\\
& =g(d_{u(i)})+b_i\\
&=g(b_i)+b_i\\
&= c_i
\end{align*}
So $\textrm{var}(h^{(m+1)}) = \textrm{var}(c)$.

Then we prove that $\forall j = 1, 2, \cdots, m, \, \textrm{var}(h^{(j+1)}) < \textrm{var}(h^{(j)})$. First, we need to prove that for any $j$, $h_i^{(j)}$ is a non-decreasing sequence in terms of $i$. In order to prove that, we only need to prove $\forall j = 2, \cdots, n, h_i^{(j+1)} - h_i^{(j)} < h_{l(m-j+1)+1}^{(j)} - h_{l(m-j+1)}^{(j)}$. Then from $h^{(j+1)}$ to $h^{(j)}$, as elements before $l(m-j+1)$ are updated by the same value, and elements after $l(m-j+1)+1$ are unchanged, so if the ${l(m-j+1)}^{th}$ element does not exceed the ${l(m-j+1)+1}^{th}$ element, then the order of the whole sequence remains during the update. The proof is as follows: $\forall j \ge 2$, $h_{l(m-j+1)+1}^{(j)}=\sum_{k=1}^{j-1} (h_{l(m-j+1)+1}^{(k+1)} - h_{l(m-j+1)+1}^{(k)}) + h_{l(m-j+1)}^{(1)}$. As $k \le j-1 \Rightarrow l(m-k+1) \ge l(m-j+2) = l(m-j+1+1) \ge l(m-j+1)+1$, from the definition of the $h$ we know that
$$
h_{l(m-j+1)+1}^{(k+1)} - h_{l(m-j+1)+1}^{(k)}=\left\{
\begin{aligned}
&g(d_{m-k+1}) - g(d_{m-k+2}) \textrm{, when } k \ge 2\\
&g(d_{m}) \textrm{, when } k=1\\
\end{aligned}
\right.
$$
So we have
\begin{align*}
h_{l(m-j+1)+1}^{(j)}&=\sum_{k=1}^{j-1} (h_{l(m-j+1)+1}^{(k+1)} - h_{l(m-j+1)+1}^{(k)}) + h_{l(m-j+1)}^{(1)}\\
& = g(d_m) + \sum_{k=2}^{j-1} (g(d_{m-k+1})-g(d_{m-k+2})) + b_{l(m-j+1)+1}\\
&= g(d_{m-j+2}) + b_{l(m-j+1)+1}
\end{align*}
From the definition of $l(\cdot)$, we have that $b_{l(m-j+1)+1}=d_{m-j+2}$, so $h_{l(m-j+1)+1}^{(j)}=g(b_{l(m-j+1)})+b_{l(m-j+1)+1}$. Similarly, $h_{l(m-j+1)}^{(j)} = b_{l(m-j+1)}+g(b_{l(m-j+2)}) = b_{l(m-j+1)}+
g(b_{l(m-j+1)}) - (g(d_{m-j+1})-g(d_{m-j+2}))$. So $h_{l(m-j+1)+1}^{(j)} - h_{l(m-j+1)}^{(j)}=b_{l(m-j+1)+1}-b_{l(m-j+1)}
\\
+(g(b_{l(m-j+1)+1})-g(b_{l(m-j+1)})) + (g(d_{m-j+1})-g(d_{m-j+2}))$. As the function $g(x)$ is bounded between $0$ and $b_{i+1}-b_i$, we have $g(b_{l(m-j+1)+1})-g(b_{l(m-j+1)}) > -(b_{l(m-j+1)+1}-b_{l(m-j+1)})$. So
\begin{equation}
\begin{array}{lll}
h_{l(m-j+1)+1}^{(j)} - h_{l(m-j+1)}^{(j)}&=&b_{l(m-j+1)+1}-b_{l(m-j+1)}+(g(b_{l(m-j+1)+1})-g(b_{l(m-j+1)})) \\
&&+ (g(d_{m-j+1})-g(d_{m-j+2})) \\
&>& 0 + (g(d_{m-j+1})-g(d_{m-j+2}))\\
&=&g(d_{m-j+1})-g(d_{m-j+2})
\end{array}
\end{equation}
As $h_i^{(j+1)}-h_i^{(j)}$ is either 0 or $g(d_{m-j+1})-g(d_{m-j+2})$ which is positive, we have proved $\forall j \ge 2$, $h_i^{(j+1)} - h_i^{(j)} < h_{l(m-j+1)+1}^{(j)} - h_{l(m-j+1)}^{(j)}$. According to former discussion, for a fixed $j$, $h_i^{(j)}$ is a non-decreasing sequence.

We can prove $\textrm{var}(h^{(j+1)})< \textrm{var}(h^{(j)})$ now:

If $j=1$, $l(m)=n$, $\forall i = 1,2,\cdots,n$, $h_i^{(2)}-h_i^{(1)} = g(d_m)$, so $\textrm{var}(h^{(2)}) = \textrm{var}(h^{(1)})$.

For $j \ge 2$, let $\Delta^{(j)} = g(d_{m-j+1})-g(d_{m-j+2})$, let $l=l(m-j+1)$, we first use the recursive definition of $h$ to express $h^{(j+1)}$ by $h^{(j)}$:
\begin{align*}
\textrm{var}(h^{(j+1)})&=\frac{1}{n} \sum_{i=1}^{n} (h_i^{(j+1)}-\frac{1}{n}\sum_{i=1}^n h_i^{(j+1)})^2 \\
&= \frac{1}{n}\sum_{i=1}^n (h_i^{(j)} + \mathbb{I}(i \le l)\Delta^{(j)} - \frac{1}{n}\sum_{i=1}^n h_i^{(j)} - \frac{l}{n}\Delta^{(j)})^2
\end{align*}
Then following simple algebra to expand the above equation, we have
\begin{align*}
\textrm{var}(h^{(j+1)})&= \textrm{var}(h^{(j)}) + \frac{l \Delta^{(j)}}{n}[2(\frac{1}{l} \sum_{i=1}^l h_i^{(j)} - \frac{1}{n} \sum_{i=1}^n h_i^{(j)}) + \frac{n-l}{n}\Delta^{(j)}] \\
&= \textrm{var}(h^{(j)}) + \frac{l \Delta^{(j)}}{n} [2(\frac{1}{l} \sum_{i=1}^l h_i^{(j)} - \frac{1}{n} \sum_{i=1}^l h_i^{(j)}- \frac{1}{n} \sum_{i=l+1}^n h_i^{(j)}) + \frac{n-l}{n}\Delta^{(j)}] \\
\end{align*}
Note that $h_{l+1}^{(j)} - h_l^{(j)} = \Delta^{(j)}$, we can further obtain
\begin{align*}
\textrm{var}(h^{(j+1)}) &= \textrm{var}(h^{(j)}) + \frac{l \Delta^{(j)}}{n}[2(\frac{1}{l} \sum_{i=1}^l h_i^{(j)} - \frac{1}{n} \sum_{i=1}^l h_i^{(j)} - \frac{1}{n} \sum_{i=l+1}^n (h_i^{(j)} - h_{l+1}^{(j)}+h_l^{(j)})) - \frac{n-l}{n}\Delta^{(j)}]
\end{align*}
Note that for a fixed $j$, $h_i^{(j)}$ is a non-decreasing sequence, we have $\forall i \ge l+1$, $h_i^{(j)} - h_{l+1}^{(j)} + h_l^{(j)} \ge h_{l+1}^{(j)} - h_{l+1}^{(j)} + h_l^{(j)} = h_l^{(j)}$, and $\frac{1}{l}\sum_{i=1}^l h_i^{(j)} \le h_l^{(j)}$, so $\frac{1}{l}\sum_{i=1}^l h_i^{(j)} - \frac{1}{n} \sum_{i=1}^l h_i^{(j)} - \frac{1}{n} \sum_{i=l+1}^n (h_i^{(j)} - h_{l+1}^{(j)}+h_l^{(j)}) \le 0$, so $\textrm{var}(h^{(j+1)}) \le \textrm{var}(h^{(j)}) - \frac{l\Delta^{(j)}}{n}\frac{n-l}{n}\Delta^{(j)} < \textrm{var}(h^{(j)})$.

Putting above results together, since $\textrm{var}(h^{(j+1)}) < \textrm{var}(h^{(j)})$ and $\textrm{var}(h^{(m+1)}) = \textrm{var}(c)$, we know that $\textrm{var}(c) < \textrm{var}(h^{(1)}) = \textrm{var}(b)$. Furthermore, let $n'=\textrm{max}\{j:b_j \ne b_n\}$, then $\forall i$, $h_i^{(2)}=h_i^{(1)}+g(d_m)=b_i + g(b_n)$, $h_i^{(3)} = h_i^{(2)}+(g(d_{m-1})-g(d_{m}))\mathbb{I}(i\le l(m-1))=b_i +g(b_n) + (g(b_{n'})-g(b_n))\mathbb{I}(i \le n') = b_i'$, so $\textrm{var}(c)\le \textrm{var}(b') < \textrm{var} (b)$. The proof completes.
\end{proof}

\setcounter{mylemma}{6}

\section{Proof of Lemmas in Estimation Error}
\label{apd:est}
\begin{mylemma}
Let $\mathcal{R}_n(\mathcal{F})$ denote the Rademacher complexity of the hypothesis set $\mathcal{F}=\{f|f(\mb{x})=\sum\limits_{j=1}^{m}\alpha_j h(\mb{w_j}^{\mathsf{T}}\mb{x})\}$, then
\begin{equation}
\mathcal{R}_n(\mathcal{F})\leq \frac{2 L C_{1}C_{3}C_{4}\sqrt{m}}{\sqrt{n}} + \frac{C_{4}|h(0)|\sqrt{m}}{\sqrt{n}}
\end{equation}
\end{mylemma}
\begin{proof}
\begin{equation}
\begin{array}{lll}
\mathcal{R}_n(\mathcal{F}) &= \mathbb{E}[\sup_{f\in \mathcal{F}}\frac{1}{n}\sum_{i=1}^n\sigma_i\sum_{j=1}^m \alpha_j h(\mb{w_j}^T \mb{x}^{(i)})]\\
&= \mathbb{E}[\sup_{f\in \mathcal{F}}\frac{1}{n}\sum_{j=1}^m \alpha_j \sum_{i=1}^n \sigma_i h(\mb{w_j}^T \mb{x}^{(i)})]
\end{array}
\end{equation}
Let $\mb{\alpha} = [\alpha_1, \cdots, \alpha_m]^T$ and $\mb{h} = [\sum_{i=1}^n\sigma_ih(\mb{w_1}^T\mb{x}^{(i)}),\cdots,\sum_{i=1}^n\sigma_ih(\mb{w_m}^T\mb{x}^{(i)})]^T$, the inner product $\mb{\alpha}\cdot\mb{h}\le\|\mb{\alpha}\|_1\|\mb{h}\|_\infty$ as $\|\cdot\|_1$ and $\|\cdot\|_\infty$ are dual norms. Therefore
\begin{equation}
\begin{array}{lll}
\mb{\alpha}\cdot\mb{h}&\le&\|\mb{\alpha}\|_1 \|\mb{h}\|_\infty\\
&=& (\sum_{j=1}^m |\alpha_j|)(\max_{j=1,\cdots,m}|\sum_{i=1}^n \sigma_ih(\mb{w_j}^T \mb{x}^{(i)})|)\\
&\le& \sqrt{m} \|\mb{\alpha}\|_2\cdot \max_{j=1,\cdots,m}|\sum_{i=1}^n\sigma_ih(\mb{w_j}^T\mb{x}^{(i)})|\\
&\le&\sqrt{m}C_{4}\cdot \max_{j=1,\cdots,m}|\sum_{i=1}^n\sigma_ih(\mb{w_j}^T\mb{x}^{(i)})|
\end{array}
\end{equation}
So $\mathcal{R}_n(\mathcal{F})\le \sqrt{m}C_4\mathbb{E}[\sup_{f\in \mathcal{F}}\frac{1}{n}|\sum\limits_{i=1}^n\sigma_ih(\mb{w}^T\mb{x}^{(i)})|]$. Denote $\mathcal{R}_{||}(\mathcal{F})=\mathbb{E}[\sup_{f\in\mathcal{F}}|\frac{2}{n}\sum\limits_{i=1}^{n}\sigma_i f(\mb{x}^{(i)})|]$, which is another form of Rademacher complexity used in some literature such as \citep{bartlett2003rademacher}. Let $\mathcal{F}'=\{f'|f'(\mb{x})=h(\mb{w}^T\mb{x})\}$ where $\mb{w}, \mb{x}$ satisfy the conditions specified in Section \ref{sec:setup}, then $\mathcal{R}_n(\mathcal{F})\le\frac{\sqrt{m}C_{4}}{2}\mathcal{R}_{||}(\mathcal{F'})$.

Let $\mathcal{G} = \{g|g(\mb{x})=\mb{w}^T\mb{x}\}$ where $\mb{w}, \mb{x}$ satisfy the conditions specified in Section \ref{sec:setup}, then $\mathcal{R}_{||}(\mathcal{F'})=\mathcal{R}_{||}(h\circ g)$. Let $h'(\cdot) = h(\cdot) - h(0)$, then $h'(0)=0$ and $h'$ is also L-Lipschitz. Then
\begin{equation}
\label{eq:R_h_g}
\begin{array}{lll}
\mathcal{R}_{||}(\mathcal{F'})&=&\mathcal{R}_{||}(h\circ g) \\
&=& \mathcal{R}_{||}(h'\circ g+h(0))\\
&\le &\mathcal{R}_{||}(h'\circ g) + \frac{2|h(0)|}{\sqrt{n}} \textrm{ (Theorem 12 in \citep{bartlett2003rademacher})}\\
&\le& 2L \mathcal{R}_{||}(g)+ \frac{2|h(0)|}{\sqrt{n}} \textrm{ (Theorem 12 in \citep{bartlett2003rademacher})}
\end{array}
\end{equation}
Now we bound $\mathcal{R}_{||}(g)$:
\begin{equation}
\label{eq:R_g}
\begin{array}{lll}
\mathcal{R}_{||}(g)&=&\mathbb{E}[\sup_{g\in{\mathcal{G}}}|\frac{2}{n}\sum_{i=1}^n\sigma_i\mb{w}^T\mb{x}^{(i)}|]\\
&\le &\frac{2}{n}\mathbb{E}[\sup_{g\in\mathcal{G}}\|\mb{w}\|_2\cdot\|\sum_{i=1}^n\sigma_i\mb{x}^{(i)}\|_2]\\
&\le &\frac{2C_3}{n}\mathbb{E}[\|\sum_{i=1}^n\sigma_i\mb{x}^{(i)}\|_2]\\
&=&\frac{2C_3}{n}\mathbb{E}_{\mb{x}}[\mathbb{E}_{\sigma}[\|\sum_{i=1}^n\sigma_i\mb{x}^{(i)}\|_2] ]\\
&\le &\frac{2C_3}{n}\mathbb{E}_{\mb{x}}[\sqrt{\mathbb{E}_{\sigma}[\|\sum_{i=1}^n\sigma_i\mb{x}^{(i)}\|_2^2]} ]\textrm{ (concavity of $\sqrt{\cdot}$)}\\
&=&\frac{2C_3}{n}\mathbb{E}_{\mb{x}}[\sqrt{\mathbb{E}_{\sigma}[\sum_{i=1}^n \sigma_i^2\|\mb{x}^{(i)}\|_2^2]}]\textrm{ ($\forall i\neq j\ \sigma_i\ci\sigma_j$)}\\
&=&\frac{2C_3}{n}\mathbb{E}_{\mb{x}}[\sqrt{\sum_{i=1}^n \|\mb{x}^{(i)}\|_2^2}]\\
&\le& \frac{2C_1C_3}{\sqrt{n}}
\end{array}
\end{equation}
Putting Eq.(\ref{eq:R_h_g}) and Eq.(\ref{eq:R_g}) together, we have $\mathcal{R}_{||}(\mathcal{F'})\le \frac{4LC_1C_3}{\sqrt{n}}+\frac{2|h(0)|}{\sqrt{n}}$. Plugging into $\mathcal{R}_n(\mathcal{F})\le\frac{\sqrt{m}C_{4}}{2}\mathcal{R}_{||}(\mathcal{F'})$ completes the proof.
\end{proof}

\begin{mylemma}
With probability at least $\tau$
\begin{equation}
\sup_{\mb{x},f}|f(\mb{x})| \le \sqrt{\mathcal{J}}
\end{equation}
where $\mathcal{J}=mC_4^2h^2(0)+L^2C_1^2C_3^2C_4^2((m-1)\cos\theta+1) + 2\sqrt{m}C_1C_3C_4^2L|h(0)|\sqrt{(m-1)\cos\theta+1}$.
\end{mylemma}
\begin{proof}
Let $\mb{\alpha} = [\alpha_1,\cdots,\alpha_m]^T$, $\mb{W}=[\mb{w_1},\cdots,\mb{w_m}]$, $\mb{h} = [h(\mb{w_1}^T\mb{x}), \cdots, h(\mb{w_m}^T\mb{x})]^T$, then we have
\begin{equation}
\label{eq:f_square}
\begin{array}{lll}
f^2(\mb{x})&=& (\sum_{j=1}^m \alpha_j h(\mb{w_j}^T \mb{x}))^2\\
&=& (\mb{\alpha}\cdot \mb{h})^2\\
&\le& (\|\mb{\alpha}\|_2 \|\mb{h}\|_2)^2\\
&\le& C_4^2 \|\mb{h}\|_2^2
\end{array}
\end{equation}

Now we want to derive an upper bound for $\|\mb{h}\|_2$. As $h(t)$ is L-Lipschitz, $|h(\mb{w_j}^T \mb{x})|\le L|\mb{w_j}^T\mb{x}| + |h(0)|$. Therefore
\begin{equation}
\label{eq:bound_h}
\begin{array}{lll}
\|\mb{h}\|_2^2&=& \sum_{j=1}^m h^2(\mb{w_j}^T \mb{x})\\
&\le& \sum_{j=1}^m (L|\mb{w_j}^T\mb{x}| + |h(0)|)^2\\
&=& \sum_{j=1}^m h^2(0)+L^2(\mb{w_j}^T\mb{x})^2 + 2L|h(0)||\mb{w_j}^T\mb{x}|\\
&=& mh^2(0)+L^2\|\mb{W}^T\mb{x}\|_2^2 + 2L|h(0)||\mb{W}^T\mb{x}|_1\\
&\le& mh^2(0)+L^2\|\mb{W}^T\mb{x}\|_2^2 + 2\sqrt{m}L|h(0)|\|\mb{W}^T\mb{x}\|_2\\
&\le& mh^2(0)+L^2\|\mb{W}^T\|_{op}^2\|\mb{x}\|_2^2 + 2\sqrt{m}L|h(0)|\|\mb{W}^T\|_{op}\|\mb{x}\|_2\\
&=& mh^2(0)+L^2\|\mb{W}\|_{op}^2\|\mb{x}\|_2^2 + 2\sqrt{m}L|h(0)|\|\mb{W}\|_{op}\|\mb{x}\|_2\\
&\le& mh^2(0)+L^2C_1^2\|\mb{W}\|_{op}^2 + 2\sqrt{m}C_1L|h(0)|\|\mb{W}\|_{op}\\
\end{array}
\end{equation}
where $\|\cdot\|_{op}$ denotes the operator norm. We can make use of the lower bound of $\rho(\mb{w_j}, \mb{w_k})$ for $j\neq k$ to get a bound for $\|\mb{W}\|_{op}$:
\begin{equation}
\begin{array}{lll}
\|\mb{W}\|_{op}^2&=&\sup_{\|\mb{u}\|_2=1} \|\mb{W} \mb{u}\|_2^2\\
&=& \sup_{\|\mb{u}\|_2=1} (\mb{u}^T \mb{W}^T \mb{W} \mb{u})\\
&=& \sup_{\|\mb{u}\|_2=1} \sum_{j=1}^{m} \sum_{k=1}^{m} u_j u_k \mb{w_j}\cdot \mb{w_k}\\
&\le& \sup_{\|\mb{u}\|_2=1} \sum_{j=1}^{m} \sum_{k=1}^{m} |u_j| |u_k| \|\mb{w_j}\|_2\|\mb{w_k}\|_2\cos(\rho(\mb{w_j},\mb{w_k}))\\
&\le& C_3^2\sup_{\|\mb{u}\|_2=1} \sum_{j=1}^{m} \sum_{k=1,k\neq j}^{m} |u_j| |u_k| \cos\theta + \sum_{j=1}^{m}|u_j|^2 \\
&&\text{(with probability at least $\tau$)}\\
\end{array}
\end{equation}
Define $\mb{u}'=[|u_1|, \cdots, |u_{m}|]^T$, $\mb{Q}\in \mathbb{R}^{m\times m}$: $Q_{jk}=\cos\theta$ for $j\neq k$ and $Q_{jj}=1$, then $\|\mb{u}'\|_2 = \|\mb{u}\|_2$ and
\begin{equation}
\begin{array}{lll}
\|\mb{W}\|_{op}^2&\le& C_3^2\sup_{\|\mb{u}\|_2=1} \mb{u}'^T \mb{Q} \mb{u}'\\
&\le& C_3^2\sup_{\|\mb{u}\|_2=1} \lambda_1(\mb{Q})\|\mb{u}'\|_2^2\\
&\le& C_3^2 \lambda_1(\mb{Q})
\end{array}
\end{equation}
where $\lambda_1(\mb{Q})$ is the largest singular value of $\mb{Q}$. By simple linear algebra we can get $\lambda_1(\mb{Q}) = (m-1)\cos\theta + 1$, so
\begin{equation}
\|\mb{W}\|_{op}^2 \le ((m-1)\cos\theta + 1)C_3^2
\end{equation}
Substitute to Eq.(\ref{eq:bound_h}), we have
\begin{multline}
\|\mb{h}\|_2^2\le mh^2(0)+L^2C_1^2C_3^2((m-1)\cos\theta+1) + 2\sqrt{m}C_1C_3L|h(0)|\sqrt{(m-1)\cos\theta+1}
\end{multline}
Substitute to Eq.(\ref{eq:f_square}):
\begin{equation}
\label{eq:f_square_2}
\begin{array}{lll}
f^2(\mb{x})\le mC_4^2h^2(0)+L^2C_1^2C_3^2C_4^2((m-1)\cos\theta+1) +2\sqrt{m}C_1C_3C_4^2L|h(0)|\sqrt{(m-1)\cos\theta+1}
\end{array}
\end{equation}
In order to simplify our notations, define
\begin{multline}
\mathcal{J}=mC_4^2h^2(0)+L^2C_1^2C_3^2C_4^2((m-1)\cos\theta+1) + 2\sqrt{m}C_1C_3C_4^2L|h(0)|\sqrt{(m-1)\cos\theta+1}
\end{multline}
Then $\sup_{\mb{x},f}|f(\mb{x})|= \sqrt{\sup_{\mb{x},f}f^2(\mb{x})} \le \sqrt{\mathcal{J}}$. Proof completes.
\end{proof}

\begin{mylemma}
Let $\mathcal{R}_n(\mathcal{F}^P)$ denote the Rademacher complexity of the hypothesis set $\mathcal{F}^P$, then
\begin{multline}
\mathcal{R}_n(\mathcal{F}^P) \le \frac{(2L)^PC_1C_3^P}{\sqrt{n}}\prod_{p=0}^{P-1}\sqrt{m^p}C_3^p
+\frac{|h(0)|}{\sqrt{n}}\sum_{p=0}^{P-1}(2L)^{P-1-p}\prod_{j=p}^{P-1}\sqrt{m^j}C_3^{j+1}
\end{multline}
\end{mylemma}
\begin{proof}
Notice that $\mathcal{R}_n(\mathcal{F}^P))\le \frac{1}{2}\mathcal{R}_{||}(\mathcal{F}^P)$:
\begin{equation}
\label{eq:R_rel}
\begin{array}{lll}
\mathcal{R}_n (\mathcal{F}^P) &=&\mathbb{E}[\sup_{f\in\mathcal{F}^P}\frac{1}{n}\sum_{i=1}^n\sigma_if(\mb{x}^{(i)})]\\
&\le& \mathbb{E}[\sup_{f\in\mathcal{F}^P}|\frac{1}{n}\sum_{i=1}^n\sigma_if(\mb{x}^{(i)})|]\\
&=& \frac{1}{2}\mathcal{R}_{||}(\mathcal{F}^P)
\end{array}
\end{equation}
So we can bound $\mathcal{R}_n(\mathcal{F}^P)$ by bounding $\frac{1}{2}\mathcal{R}_{||}(\mathcal{F}^P)$. We bound $\frac{1}{2}\mathcal{R}_{||}(\mathcal{F}^p)$ recursively: $\forall p=1,\cdots,P$, we have
\begin{equation}
\label{eq:R_mult_rec}
\begin{array}{lll}
\mathcal{R}_{||}(\mathcal{F}^p)&=&\mathbb{E}[\sup_{f\in\mathcal{F}^p}|\frac{2}{n}\sum_{i=1}^n\sigma_if(\mb{x}^{(i)})|]\\
&=&\mathbb{E}[\sup_{f_j\in\mathcal{F}^{p-1}}|\frac{2}{n}\sum_{i=1}^n\sigma_i\sum_{j=1}^{m^{p-1}}{w_j}^ph(f_j(\mb{x}^{(i)}))|]\\
&\le& \sqrt{m^{p-1}}C_3^p \mathbb{E}[\sup_{f_j\in\mathcal{F}^{p-1}}|\frac{2}{n}\sum_{i=1}^n\sigma_ih(f_j(\mb{x}^{(i)}))|]\\
&\le& \sqrt{m^{p-1}}C_3^p(2L\mathcal{R}_{||}(\mathcal{F}^{p-1})+\frac{2|h(0)|}{\sqrt{n}})
\end{array}
\end{equation}
where the last two steps follow similar steps as the proof of Lemma \ref{lem:rc_f}. Applying the inequality in Eq.(\ref{eq:R_mult_rec}) recursively, and noting from the proof of Lemma \ref{lem:rc_f} that $\mathcal{R}_{||}(\mathcal{F}^0) \le \frac{2C_1C_3^0}{\sqrt{n}}$ we have
\begin{equation}
\begin{array}{lll}
&\mathcal{R}_{||}(\mathcal{F}^P) &\le \frac{2(2L)^PC_1C_3^P}{\sqrt{n}}\prod_{p=0}^{P-1}\sqrt{m^p}C_3^p+\frac{2|h(0)|}{\sqrt{n}}\sum_{p=0}^{P-1}(2L)^{P-1-p}\prod_{j=p}^{P-1}\sqrt{m^j}C_3^{j+1}
\end{array}
\end{equation}
Plugging into Eq.(\ref{eq:R_rel}) completes the proof.
\end{proof}

\begin{mylemma}
With probability at least $\prod_{p=0}^{P-1}\tau^p$,
$\sup_{\mb{x}, f^P\in\mathcal{F}^p}|f^P(\mb{x})| \le \sqrt{\mathcal{J}^P}$, where
\begin{equation}
\begin{array}{lll}
\mathcal{J}^0 & =& C_1^2(C_3^0)^2\\
\mathcal{J}^{p} &=& m^{p-1}(C_3^p)^2h^2(0)+L^2(C_3^p)^2((m^{p-1}-1)\cos\theta^{p-1}+1)\mathcal{J}^{p-1}\\
&&+2\sqrt{m^{p-1}}(C_3^p)^2L|h(0)|\sqrt{((m^{p-1}-1)\cos\theta^{p-1}+1)\mathcal{J}^{p-1}} (p=1,\cdots,P)
\end{array}
\end{equation}
\end{mylemma}
\begin{proof}
For a given neural network and input $\mb{x}$, we denote the outputs of hidden layer $p$ before applying the activation function as $\mb{v}^p$:
\begin{equation}
\begin{array}{lll}
\mb{v}^0 &=& [{\mb{w_1}^0}\cdot\mb{x}, \cdots, {\mb{w_{m^0}}^0}\cdot\mb{x}]^T\\
\mb{v}^p &=& [\sum_{j=1}^{m^{p-1}}{w_{1,j}^p}h(v_j^{p-1}), \cdots,\sum_{j=1}^{m^{p-1}}{w_{m^p,j}^p}h(v_j^{p-1})]^T (p=1,\cdots,P)
\end{array}
\end{equation}
where ${w_{i,j}^p}$ is the connecting weight from the $j$-th unit of the hidden layer $p-1$ to the $i$-th unit of the hidden layer $p$.

We also denote the outputs of hidden layer $p$ after applying the activation function as $\mb{h}^p$:
\begin{equation}
\mb{h}^p = [h(v_1^{p}), \cdots, h(v_{m^{p}}^{p})]^T (p=0,\cdots,P)
\end{equation}
where $v_i^{p}$ is the $i$-th element of $v^{p}$.

To facilitate the derivation of bounds, we denote
\begin{equation}
\mb{w_i}^p = [{w_{i,1}}^p, \cdots, w_{i,m^{p-1}}^p]^T
\end{equation}

$\forall f\in \mathcal{F}^P$, we can bound $|f|$ by bounding $\|\mb{h}^P\|_2$ as
\begin{equation}
\label{eq:f_h_bound}
|f| = |\sum_{j=1}^{m^{P-1}} w_j^P h(v_j^{P-1})| = |\mb{w}^P\cdot \mb{h}^{P-1}| \le C_3^P \|\mb{h}^{P-1}\|_2
\end{equation}
where we only have one subscript for $\mb{w}^P$ because the output is a scalar.

We bound $\|\mb{h}^p\|_2$ recursively below:

First, we make use of the Lipschitz-continuous property of $h(t)$:
\begin{equation}
\label{eq:bd_hp}
\begin{array}{lll}
\|\mb{h}^p\|_2^2&=& \sum_{j=1}^{m^{p}} h^2(v_j^{p}) \\
&\le& \sum_{j=1}^{m^{p}} (|h(0)|+L|v_j^{p}|)^2 \\
&=& \sum_{j=1}^{m^{p}} h^2(0)+L^2(v_j^{p})^2 + 2L|h(0)||v_j^{p}| \\
&=& m^{p}h^2(0)+L^2\|\mb{v}^{p}\|_2^2 + 2L|h(0)|\|\mb{v}^{p}\|_1\\
&\le& m^{p}h^2(0)+L^2\|\mb{v}^{p}\|_2^2 + 2L|h(0)|\sqrt{m^{p}}\|\mb{v}^{p}\|_2\\
\end{array}
\end{equation}

Note that we can write $\mb{v}^p$ as
\begin{equation}
\mb{v}^p = [\mb{w_1}^p\cdot\mb{h}^{p-1}, \cdots, \mb{w_{m^p}}^p\cdot\mb{h}^{p-1}]^T
\end{equation}
for $p=1,\cdots,P$.

Hence we have
\begin{equation}
\|\mb{v}^p\|_2^2 = \sum_{i=1}^{m^p}(\mb{w_i}^p\cdot\mb{h}^{p-1})^2 (p=1,\cdots,P)
\end{equation}
Denote $\mb{W} = [\mb{w_1}^p, \cdots, \mb{w_{m^p}}^p]$ (note that $\mb{W}$ depends on $p$, but we omit that for brevity), then
\begin{equation}
\begin{array}{lll}
\label{eq:v_p_recur}
\|\mb{v}^p\|_2^2&=& \|\mb{W}^T \mb{h}^{p-1}\|_2^2\\
&\le& \|\mb{W}^T\|_{op}^2 \|\mb{h}^{p-1}\|_2^2\\
&=& \|\mb{W}\|_{op}^2 \|\mb{h}^{p-1}\|_2^2
\end{array}
\end{equation}
where $\|\cdot\|_{op}$ denotes the operator norm.

We can make use of the lower bound of $\rho(\mb{w_j}^p, \mb{w_k}^p)$ for $j\neq k$ to get a bound for $\|\mb{W}\|_{op}$:
\begin{equation}
\begin{array}{lll}
\|\mb{W}\|_{op}^2&=&\sup_{\|\mb{u}\|_2=1} \|\mb{W} \mb{u}\|_2^2\\
&=& \sup_{\|\mb{u}\|_2=1} (\mb{u}^T \mb{W}^T \mb{W} \mb{u})\\
&=& \sup_{\|\mb{u}\|_2=1} \sum_{j=1}^{m^p} \sum_{k=1}^{m^p} u_j u_k \mb{w_j}^p\cdot \mb{w_k}^p\\
&\le& \sup_{\|\mb{u}\|_2=1} \sum_{j=1}^{m^p} \sum_{k=1}^{m^p} |u_j| |u_k| |\mb{w_j}^p||\mb{w_k}^p|\cos(\rho(\mb{w_j}^p,\mb{w_k}^p))\\
&\le& (C_3^p)^2(\sup_{\|\mb{u}\|_2=1} \sum_{j=1}^{m^p} \sum_{k=1,k\neq j}^{m^p}) |u_j| |u_k| \cos\theta^p + \sum_{j=1}^{m^p}|u_j|^2\\
&&(\text{with probability at least $\tau^p$})\\
\end{array}
\end{equation}
Define $\mb{u}'=[|u_1|, \cdots, |u_{m^p}|]^T$, $\mb{Q}\in \mathbb{R}^{m^p\times m^p}$: $Q_{jk}=\cos\theta^p$ for $j\neq k$ and $Q_{jj}=1$, then $\|\mb{u}'\|_2 = \|\mb{u}\|_2$ and
\begin{equation}
\begin{array}{lll}
\|\mb{W}\|_{op}^2&\le& (C_3^p)^2\sup_{\|\mb{u}\|_2=1} \mb{u}'^T \mb{Q} \mb{u}'\\
&\le& (C_3^p)^2\sup_{\|\mb{u}\|_2=1} \lambda_1(\mb{Q})\|\mb{u}'\|_2^2\\
&\le& (C_3^p)^2 \lambda_1(\mb{Q})
\end{array}
\end{equation}
where $\lambda_1(\mb{Q})$ is the largest singular value of $\mb{Q}$. By simple linear algebra we can get $\lambda_1(\mb{Q}) = (m^p-1)\cos\theta^p + 1$, so
\begin{equation}
\label{eq:bd_wop}
\|\mb{W}\|_{op}^2 \le ((m^p-1)\cos\theta^p + 1)(C_3^p)^2
\end{equation}
Substituting Eq.(\ref{eq:bd_wop}) back to Eq.(\ref{eq:v_p_recur}), we have
\begin{equation}
\begin{array}{l}
\label{eq:v_p_recur_2}
\|\mb{v}^p\|_2^2 \le (C_3^p)^2((m^p-1)\cos\theta^p + 1) \|\mb{h}^{p-1}\|_2^2
\end{array}
\end{equation}

Substituting Eq.(\ref{eq:v_p_recur_2}) to Eq.(\ref{eq:bd_hp}), we have
\begin{equation}
\label{eq:h_recur}
\begin{array}{lll}
\|\mb{h}^p\|_2^2 &\le& m^{p}h^2(0)
+L^2(C_3^p)^2((m^p-1)\cos\theta^p+1)\|\mb{h}^{p-1}\|_2^2\\
&&+2L|h(0)|\sqrt{m^p}C_3^p\sqrt{(m^p-1)\cos\theta^p+1}\|\mb{h}^{p-1}\|_2
\end{array}
\end{equation}

Note that
\begin{equation}
\|\mb{h}^0\|_2^2\le m^{0}h^2(0)+L^2\|\mb{v}^{0}\|_2^2 + 2L|h(0)|\sqrt{m^{0}}\|\mb{v}^{0}\|_2
\end{equation}
, and similar to the derivation of Eq.(\ref{eq:v_p_recur_2}), we can get a bound on $\|\mb{v}^0\|_2$:
\begin{equation}
\begin{array}{lll}
\|\mb{v}^0\|_2^2&\le& (C_3^0)^2((m^0-1)\cos\theta^0 + 1) \|\mb{x}\|_2^2\\
&\le& C_1^2 (C_3^0)^2((m^0-1)\cos\theta^0 + 1)
\end{array}
\end{equation}
Therefore
\begin{multline}
\label{eq:h_0}
\|\mb{h}^0\|_2^2 \le m^{0}h^2(0)
+L^2C_1^2 (C_3^0)^2((m^0-1)\cos\theta^0 + 1)
+ 2\sqrt{m^{0}}C_1C_3^0L|h(0)|\sqrt{(m^0-1)\cos\theta^0 + 1}
\end{multline}
Using Eq.(\ref{eq:h_recur}) and Eq.(\ref{eq:h_0}) we can bound $(C_3^p)^2\|\mb{h}^{p-1}\|_2^2$ recursively now.
Denote
\begin{equation}
\begin{array}{lll}
\mathcal{J}^0 &=& C_1^2(C_3^0)^2\\
\mathcal{J}^{p} &=& m^{p-1}(C_3^p)^2h^2(0)+L^2(C_3^p)^2((m^{p-1}-1)\cos\theta^{p-1}+1)\mathcal{J}^{p-1}\\
&&+2\sqrt{m^{p-1}}(C_3^p)^2L|h(0)|\sqrt{((m^{p-1}-1)\cos\theta^{p-1}+1)\mathcal{J}^{p-1}} (p=1,\cdots,P)
\end{array}
\end{equation}
then $(C_3^p)^2\|\mb{h}^{p-1}\|_2^2 \le \mathcal{J}^{p} (p=1,\cdots,P)$ and it's straightforward to verify that $\mathcal{J}^p$ decreases when $\theta^i(i=0,\cdots,p-1)$ increases.

Now we are ready to bound $\sup_{\mb{x},f^P\in\mathcal{F}^P}|f^P(\mb{x})|$ using Eq.(\ref{eq:f_h_bound}):
\begin{equation}
\begin{array}{lll}
\sup_{\mb{x},f\in\mathcal{F}^P}|f(\mb{x})|&\le& C_3^P\|\mb{h}^{P-1}\|_2\\
&\le& \sqrt{\mathcal{J}^P}
\end{array}
\end{equation}
\end{proof}

\begin{mylemma}
Let the loss function $\ell(f(\mb{x}),y)=\log(1+\exp(-yf(\mb{x})))$ be the logistic loss where $y\in\{-1,1\}$, then with probability at least $(1-\delta)\tau$
\begin{equation}
\begin{array}{lll}
L(\hat{f})-L(f^*)
\leq \frac{4}{1+\exp(-\sqrt{\mathcal{J}})}(2LC_1C_3C_4+C_4|h(0)|)\frac{\sqrt{m}}{\sqrt{n}}+ \log (1+\exp(\sqrt{J}))\sqrt{\frac{2\log(2/\delta)}{n}}
\end{array}
\end{equation}
\end{mylemma}
\begin{proof}
\begin{equation}
|\frac{\partial l(f(\mb{x}),y)}{\partial f}| = \frac{\exp(-y f(\mb{x}))}{1+\exp(-yf(\mb{x}))}=\frac{1}{1+\exp(yf(\mb{x}))}
\end{equation}
With probability at least $\tau$, $|\frac{1}{1+\exp(yf(\mb{x}))}|\le \frac{1}{1+\exp(-\sup_{f,\mb{x}} |f(\mb{x})|)}=\frac{1}{1+\exp(-\sqrt{\mathcal{J}})}$, hence we have proved that the Lipschitz constant L of $\ell(\cdot,y)$ can be bounded by $\frac{1}{1+\exp(-\sqrt{\mathcal{J}})}.$

And the loss function $\ell(f(\mb{x}),y)$ can be bounded by
\begin{equation}
\begin{array}{lcl}
|\ell(f(\mb{x}),y)|\le \log(1+\exp(\sqrt{J}))&\text{(with probability at least $\tau$)}&
\end{array}
\end{equation}

Similar to the proof of Theorem \ref{thm:est_err}, we can finish the proof by applying the composition property of Rademacher complexity, Lemma \ref{lem:rc_f} and Lemma \ref{lem:rc_bd}.
\end{proof}

\begin{mylemma}
Let $\ell(f(\mb{x}),y)=\max(0,1-yf(\mb{x}))$ be the hinge loss where $y\in\{-1,1\}$, then with probability at least $(1-\delta)\tau$
\begin{equation}
\begin{array}{lll}
L(\hat{f})-L(f^*)\leq 4(2LC_1C_3C_4+C_4|h(0)|)\frac{\sqrt{m}}{\sqrt{n}}+ (1+\sqrt{J})\sqrt{\frac{2\log(2/\delta)}{n}}
\end{array}
\end{equation}
\end{mylemma}
\begin{proof}
Given $y$, $\ell(\cdot,y)$ is Lipschitz with constant 1. And the loss function can be bounded by
\begin{equation}
\begin{array}{lcl}
\ell(f(\mb{x}),y)|\le 1+\sqrt{J}\text{(with probability at least $\tau$)}
\end{array}
\end{equation}

The proof can be completed using similar proof of Lemma \ref{lem:logistic_ub}.
\end{proof}


\begin{mylemma}
Let $\ell(f(\mb{x}),\mb{y})$ be the cross-entropy loss, then with probability at least $\tau$, for any $f$, $f'$
\begin{equation}
|\ell(f(\mb{x}),y)-\ell(f'(\mb{x}),y)|\le \frac{K-1}{K-1+\exp(-2\sqrt{\mathcal{J}})}\sum_{k=1}^K|f(\mb{x})_k-f'(\mb{x})_k|
\end{equation}
\end{mylemma}
\begin{proof}
With probability at least $\tau$, $\textrm{sup}_{{\mb{x}},f}|f(\mb{x})_k|\le \sqrt{\mathcal{J}}$. Note that $\mb{y}$ is a 1-of-K coding vector where exactly one element is 1 and all others are 0. Without loss of generality, we assume $y_{k'}=1$ and $y_{k}=0$ for $k\neq k'$. Then
\begin{equation}
\ell(f(\mb{x}),\mb{y})= -\log\frac{\exp(f(\mb{x})_{k'})}{\sum_{j=1}^{K}\exp(f(\mb{x})_{j})}
\end{equation}
Hence for $k\neq k'$ we have
\begin{equation}
\begin{array}{lll}
|\frac{\partial l(f(\mb{x}), \mb{y})}{\partial f(\mb{x})_{k}}|= \frac{1}{1+\sum_{j\neq k'}\exp(f(\mb{x})_j)}\leq \frac{1}{1+(K-1)\exp(-2\sqrt{\mathcal{J}})}
\end{array}
\end{equation}
and for $k'$ we have
\begin{equation}
\begin{array}{lll}
|\frac{\partial l(f(\mb{x}), \mb{y})}{\partial f(\mb{x})_{k'}}|= \frac{\sum_{j\neq k'}\exp(f(\mb{x})_j)}{1+\sum_{j\neq k'}\exp(f(\mb{x})_j)}\leq \frac{K-1}{K-1+\exp(-2\sqrt{\mathcal{J}})}
\end{array}
\end{equation}
As $\frac{K-1}{K-1+\exp(-2\sqrt{\mathcal{J}})}\ge \frac{1}{1+(K-1)\exp(-2\sqrt{\mathcal{J}})}$, we have proved that for any $k$, $|\frac{\partial l(f(\mb{x}), \mb{y})}{\partial f(\mb{x})_{k}}|\leq \frac{K-1}{K-1+\exp(-2\sqrt{\mathcal{J}})}$. Therefore
\begin{equation}
\|\nabla_{f(\mb{x})} \ell(f(\mb{x}),\mb{y})\|_\infty \le \frac{K-1}{K-1+\exp(-2\sqrt{\mathcal{J}})}
\end{equation}

Using mean value theorem, for any $f$, $f'$, $\exists \mb{\xi}\in\mathbb{R}^K$ such that
\begin{equation}
\begin{array}{lll}
|\ell(f(\mb{x}),\mb{y})-\ell(f'(\mb{x}),\mb{y})|&=& \nabla_{\mb{\xi}} \ell(\mb{\xi},y) \cdot (f(\mb{x})-f'(\mb{x}))\\
&\le& \|\nabla_{\mb{\xi}} \ell(\mb{\xi},\mb{y})\|_\infty \|f(\mb{x})-f'(\mb{x})\|_1\\
&\le& \frac{K-1}{K-1+\exp(-2\sqrt{\mathcal{J}})}\sum_{k=1}^K|f(\mb{x})_k-f'(\mb{x})_k|
\end{array}
\end{equation}
\end{proof}
With Lemma \ref{lem:cross_ent}, we can get the Rademacher complexity of cross entropy loss by performing the Rademacher complexity analysis for each $f(\mb{x})_k$ as that in Section \ref{sec:thm1_proof} separately, and multiplying the sum of them by $\frac{K-1}{K-1+\exp(-2\sqrt{\mathcal{J}})}$ to get the Rademacher complexity of $\ell(f(\mb{x}),\mb{y})$. And as the loss function can be bounded by
\begin{equation}
|\ell(f(\mb{x}),\mb{y})|\le \log (1+(K-1)\exp(2\sqrt{\mathcal{J}}))
\end{equation}
we can use similar techniques as the proof of Theorem \ref{thm:est_err} to get the estimation error bound.

\section{Proof of Lemmas in Approximation Error}
\label{apd:app}

\begin{mylemma}
For any three nonzero vectors $\mb{u_1}$, $\mb{u_2}$, $\mb{u_3}$, let $\theta_{12} = \arccos (\frac{\mb{u_1} \cdot \mb{u_2}}{\|\mb{u_1}\|_2 \|\mb{u_2}\|_2})$, $\theta_{23} = \arccos (\frac{\mb{u_2} \cdot \mb{u_3}}{\|\mb{u_2}\|_2 \|\mb{u_3}\|_2})$, $\theta_{13} = \arccos (\frac{\mb{u_1} \cdot \mb{u_3}}{\|\mb{u_1}\|_2 \|\mb{u_3}\|_2})$. We have $\theta_{13}\le\theta_{12}+\theta_{23}$.
\end{mylemma}
\begin{proof}
Without loss of generality, assume $\|\mb{u_1}\|_2=\|\mb{u_2}\|_2=\|\mb{u_3}\|_2=1$. Decompose $\mb{u_1}$ as $\mb{u_1}=\mb{u_{1//}}+\mb{u_{1\perp}}$ where $\mb{u_{1//}} = c_{12} \mb{u_2}$ for some $c_{12}\in\mathbb{R}$ and $\mb{u_{1\perp}}\perp \mb{u_{2}}$. As $\mb{u_{1}}\cdot \mb{u_2} = \cos\theta_{12}$, we have $c_{12}=\cos\theta_{12}$ and $\|\mb{u_{1\perp}}\|_2=\sin\theta_{12}$.

Similarly, decompose $\mb{u_3}$ as $\mb{u_3}=\mb{u_{3//}}+\mb{u_{3\perp}}$ where $\mb{u_{3//}} = c_{32} \mb{u_2}$ for some $c_{32}\in\mathbb{R}$ and $\mb{u_{3\perp}}\perp \mb{u_{2}}$. We have $c_{23}=\cos\theta_{23}$ and $\|\mb{u_{3\perp}}\|_2=\sin\theta_{23}$.

So we have
\begin{equation}
\begin{array}{lll}
\cos\theta_{13}&=&\mb{u_1}\cdot \mb{u_3}\\
&=& (\mb{u_{1//}}+\mb{u_{1\perp}})\cdot (\mb{u_{3//}}+\mb{u_{3\perp}})\\
&=& \mb{u_{1//}}\cdot \mb{u_{3//}} + \mb{u_{1\perp}}\cdot \mb{u_{3\perp}}\\
&=& \cos\theta_{12}\cos\theta_{23}+ \mb{u_{1\perp}}\cdot \mb{u_{3\perp}}\\
&\ge& \cos\theta_{12}\cos\theta_{23} - \sin\theta_{12}\sin\theta_{23}\\
&=&\cos(\theta_{12}+\theta_{23})
\end{array}
\end{equation}
If $\theta_{12}+\theta_{23}\le\pi$, $\arccos (\cos(\theta_{12}+\theta_{23}))=\theta_{12}+\theta_{23}$. As $\arccos(\cdot)$ is monotonously decreasing, we have $\theta_{13}\le\theta_{12}+\theta_{23}$. Otherwise, $\theta_{13}\le\pi\le\theta_{12}+\theta_{23}$.
\end{proof}

\begin{mylemma}
Let $\mathcal{F'} = \{f|f(\mb{x})=\sum_{j=1}^m \alpha_j h(\mb{w_j}^T \mb{x})\}$ be the function class satisfying the following constraints:
\begin{itemize}
\item $|\alpha_j| \le 2C$
\item $\|\mb{w_j}\|_2 \le C_3$
\end{itemize}
Then for every $g\in \Gamma_C$ with $g(0)=0$, $\exists f' \in \mathcal{F'}$ such that
\begin{equation}
\|g(\mb{x}) - f'(\mb{x})\|_L \le 2C(\frac{1}{\sqrt{n}} + \frac{1+ 2\ln C_1C_3}{C_1C_3})
\end{equation}
\end{mylemma}
\begin{proof}
Please refer to Theorem 3 in \citep{barron1993universal} for the proof. Note that the $\tau$ used in their paper is $C_1 C_3$ here. Furthermore, we omit the bias term $b$ as we can always add a dummy feature $1$ to the input $\mb{x}$ to avoid the bias term.
\end{proof}

\begin{mylemma}
For any $0 < \theta < \frac{\pi}{2}$, $m\le2(\lfloor\frac{\frac{\pi}{2}-\theta}{\theta}\rfloor+1)$, $(\mb{w_j}')_{j=1}^m$, $\exists (\mb{w_j})_{j=1}^m$ such that
\begin{align}
&\forall j\neq k\in\{1,\cdots,m\}, \rho(\mb{w_j},\mb{w_k})\ge \theta\\
&\forall j\in\{1,\cdots,m\}, \|\mb{w_j}\|_2 = \|\mb{w_j}'\|_2\\
&\forall j \in \{1,\cdots,m\},\arccos (\frac{\mb{w_j} \cdot \mb{w_j}'}{\|\mb{w_j}\|_2\|\mb{w_j}'\|_2}) \le \theta'
\end{align}
where $\theta' = \min(3m\theta,{\pi})$.
\end{mylemma}
\begin{proof}
For brevity, let $\phi(\mb{a},\mb{b}) = \arccos(\frac{\mb{a}\cdot \mb{b}}{\|\mb{a}\|_2\|\mb{b}\|_2})$. We begin our proof by considering a 2-dimensional case ($d=2$): Let
\begin{equation}
k = \lfloor\frac{\frac{\pi}{2}-\theta}{\theta}\rfloor
\end{equation}
Let index set $\mathcal{I} = \{-(k+1), -k,\cdots,-1,1,2,\cdots,k+1\}$. We define a set of vectors $(\mb{e_i})_{i\in\mathcal{I}}$: $\mb{e_i} = (\sin \theta_i, \cos \theta_i)$, where $\theta_i \in (-\frac{\pi}{2},\frac{\pi}{2})$ is defined as follows:
\begin{equation}
\label{eq:def_theta}
\theta_i = \mathrm{sgn}(i)(\frac{\theta}{2} + (|i|-1)\theta)
\end{equation}
From the definition we can verify the following conclusions:
\begin{align}
&\forall i\neq j\in\mathcal{I}, \rho(\mb{e_i},\mb{e_j}) \ge \theta\\
&-\frac{\pi}{2}+\frac{\theta}{2}\le\theta_{-(k+1)}<-\frac{\pi}{2}+\frac{3}{2}\theta\\
&\frac{\pi}{2}-\frac{3}{2}\theta<\theta_{k+1}\le\frac{\pi}{2}-\frac{\theta}{2}\\
\end{align}
And we can further verify that $\forall i\in\mathcal{I}$, there exists different $i_1, \cdots,i_{2k+1}\in\mathcal{I}\backslash i$ such that $\phi(\mb{e_{i}}, \mb{e_{i_j}})\le j\theta$.

For any $\mb{e} = (\sin \beta, \cos \beta)$ with $\beta\in[-\frac{\pi}{2},\frac{\pi}{2}]$, we can find $i\in\mathcal{I}$ such that $\phi(\mb{e_i}, \mb{e})\le\frac{3}{2}\theta$:
\begin{itemize}
\item if $\beta\ge\theta_{k+1}$, take $i=k+1$, we have $\phi(\mb{e_i},\mb{e})\le \frac{\pi}{2}-\theta_{k+1}<\frac{3}{2}\theta$.
\item if $\beta\le\theta_{-(k+1)}$, take $i=-(k+1)$, we also have $\phi(\mb{e_i},\mb{e})\le\frac{3}{2}\theta$
\item otherwise, take $i$ = $\textrm{sgn}(\beta)\lceil\frac{\beta-\frac{\theta}{2}}{\theta}\rceil$, we also have $\phi(\mb{e_i},\mb{e})\le \theta<\frac{3}{2}\theta$.
\end{itemize}
Recall that for any $i$, there exists different $i_1, \cdots,i_{2k+1}\in\mathcal{I}\backslash i$ such that $\phi(\mb{e_{i}}, \mb{e_{i_j}})\le j\theta$, and use Lemma \ref{lem:theta_sum_bound}, we can draw the conclusion that for any $\mb{e} = (\sin \beta, \cos \beta)$ with $\beta\in[-\frac{\pi}{2},\frac{\pi}{2}]$, there exists different $i_1, \cdots,i_{2k+2}$ such that $\phi(\mb{e_{i}}, \mb{e_{i_j}})\le \frac{3}{2}\theta+(j-1)\theta=(j+\frac{1}{2})\theta$.

For any $(\mb{w_j}')_{j=1}^m$, assume $\mb{w_j}' = \|\mb{w_j}'\|_2(\sin\beta_j,\cos\beta_j)$, and we assume $\beta_j\in[-\frac{\pi}{2},\frac{\pi}{2}]$. Using the above conclusion, for $\mb{w_1}'$, we can find some ${r_1}$ such that $\phi(\mb{w_1}', \mb{e_{r_1}})\le\frac{3}{2}\theta$. For $\mb{w_2}'$, we can find different $i_1,i_2$ such that $\phi(\mb{w_2}',\mb{e_{i_1}})\le \frac{3}{2}\theta<(\frac{3}{2}+1)\theta$ and $\phi(\mb{w_2}',\mb{e_{i_2}})\le(\frac{3}{2}+1)\theta$. So we can find $r_2\neq r_1$ such that $\phi(\mb{w_2}', \mb{e_{r_2}})\le (\frac{3}{2}+1)\theta$. Following this scheme, we can find $r_j\notin\{r_1,\cdots, r_{j-1}\}$ and $\phi(\mb{w_j}', \mb{e_{r_j}})\le (j+\frac{1}{2})\theta<3m\theta$ for $j=1,\cdots,m$, as we have assumed that $m\le2(k+1)$. Let $\mb{w_j} = \|\mb{w_j}'\|_2\mb{e_{r_j}}$, then we have constructed $(\mb{w_j})_{j=1}^m$ such that
\begin{align}
&\forall j\in\{1,\cdots,m\},\phi(\mb{w_j}',\mb{w_j})\le 3m\theta\\
&\forall j\in\{1,\cdots,m\}, \|\mb{w_j}'\|_2 = \|\mb{w_j}\|_2\\
&\forall j\neq k, \rho(\mb{w_j},\mb{w_k})\ge \theta
\end{align}
Note that we have assumed that $\forall j=1,\cdots,m$, $\beta_j\in[-\frac{\pi}{2},\frac{\pi}{2}]$. In order to show that the conclusion holds for general $\mb{w_j}'$, we need to consider the cases where $\beta_j\in[-\frac{3}{2}\pi,-\frac{\pi}{2}]$. For those cases, we can let $\beta_j'=\beta_j+\pi$, then $\beta_j'\in[-\frac{\pi}{2},\frac{\pi}{2}]$. Let $\mb{w_j}''=\|\mb{w_j}'\|_2(\sin\beta_j',\cos\beta_j')$, we can find the $\mb{e_{r_j}}$ such that $\phi(\mb{w_j}'',\mb{e_{r_j}})\le m\theta$ following the same procedure. Let $\mb{w_j}=-\|\mb{w_j}'\|_2\mb{e_{r_j}}$, then $\phi(\mb{w_j}',\mb{w_j}) = \phi(\mb{w_j}'',\mb{e_{r_j}})\le 2m\theta$ and as $\rho(-\mb{e_{r_j}},\mb{e_k})=\rho(\mb{e_{r_j}},\mb{e_k})$, the $\rho(\mb{w_j},\mb{w_k})\ge \theta$ condition is still satisfied. Also noting that $\phi(\mb{a},\mb{b})\le\pi$ for any $a$, $b$, the proof for 2-dimensional case is completed.

Now we consider a general $d$-dimensional case. Similar to the 2-dimensional one, we construct a set of vectors with unit $L_2$ norm such that the pairwise angles $\rho(\mb{w_j},\mb{w_k})\ge \theta$ for $j\neq k$. We do the construction in two phases:

In the first phase, we construct a sequence of unit vector sets indexed by $\mathcal{I} = \{-(k+1),\cdots,-1,1,\cdots,k+1\}$:
\begin{equation}
\forall i \in \mathcal{I}, \mathcal{E}_i = \{\mb{e}\in\mathbb{R}^d|\|\mb{e}\|_2=1, \mb{e}\cdot(1,0,\cdots,0)=\cos\theta_i\}
\end{equation}
where $\theta_i = \mathrm{sgn}(i)(\frac{\theta}{2} + (|i|-1)\theta)$ is defined the same as we did in Eq.(\ref{eq:def_theta}). It can be shown that $\forall i\neq j$, $\forall \mb{e_i} \in \mathcal{E}_i, \mb{e_j} \in \mathcal{E}_j$,
\begin{equation}
\rho(\mb{e_i},\mb{e_j})\ge \theta
\end{equation}
The proof is as follows.
First, we write $\mb{e_i}$ as $\mb{e_i} = (\cos\theta_i,0,\cdots,0) + \mb{r_i}$, where $\|\mb{r_i}\|_2 = |\sin \theta_i|$. Similarly, $\mb{e_j} = (\cos\theta_j,0,\cdots,0) + \mb{r_j}$, where $\|\mb{r_j}\|_2 = |\sin \theta_j|$. Hence we have
\begin{equation}
\mb{e_i} \cdot \mb{e_j} = \cos\theta_i \cos\theta_j + \mb{r_i} \cdot \mb{r_j}
\end{equation}
Hence
\begin{equation}
\begin{array}{lll}
\cos(\rho(\mb{e_i},\mb{e_j}))&=&|\mb{e_i}\cdot \mb{e_j}|\\
&\le& \cos\theta_i\cos\theta_j+|\sin\theta_i \sin\theta_j|\\
&=& \max(\cos(\theta_i+\theta_j), \cos(\theta_i-\theta_j))
\end{array}
\end{equation}
We have shown in the 2-dimensional case that $\cos(\theta_i+\theta_j)\ge\cos\theta$ and $\cos(\theta_i-\theta_j)\ge\cos\theta$, hence $\rho(\mb{e_i},\mb{e_j})\ge \theta$. In other words, we have proved that for any two vectors from $\mathcal{E}_i$ and $\mathcal{E}_j$, their pairwise angle is lower bounded by $\theta$. Now we proceed to construct a set of vectors for each $\mathcal{E}_i$ such that the pairwise angles can also be lower bounded by $\theta$. The construction is as follows.

First, we claim that for any $\mathcal{E}_i$, if $\mathcal{W}\subset \mathcal{E}$ satisfies
\begin{equation}
\forall \mb{w_j}\neq \mb{w_k} \in \mathcal{W}, \phi(\mb{w_j},\mb{w_k})\ge \theta
\end{equation}
then $|W|$ is finite. In order to prove that, we first define $B(\mb{x},r) = \{\mb{y}\in\mathbb{R}^n: \|\mb{y}-\mb{x}\|_2< r\}$. Then $\mathcal{E}_i\subset \cup_{\mb{e}\in\mathcal{E}_i} B(\mb{e}, \frac{1-\cos\frac{\theta}{2}}{1+\cos\frac{\theta}{2}})$. From the definition of $\mathcal{E}_i$, it is a compact set, so the open cover has a finite subcover. Therefore we have $\exists V\subset\mathcal{E}_i$ with $|V|$ being finite and
\begin{equation}
\mathcal{E}_i\subset \cup_{\mb{v}\in V} B(\mb{v}, \frac{1-\cos\frac{\theta}{2}}{1+\cos\frac{\theta}{2}})
\end{equation}
Furthermore, we can verify that $\forall \mb{v} \in V, \forall \mb{e_1},\mb{e_2}\in B(\mb{v}, \frac{1-\cos\frac{\theta}{2}}{1+\cos\frac{\theta}{2}})$, $\phi(\mb{e_1},\mb{e_2})\le\theta$. So if $\mathcal{W}\subset\mathcal{E}_i$ satisfies $\forall \mb{w_j}\neq \mb{w_k} \in \mathcal{W}, \phi(\mb{w_j},\mb{w_k})\ge \theta$, then for each $\mb{v}\in V$, $|B(\mb{v}, \frac{1-\cos\frac{\theta}{2}}{1+\cos\frac{\theta}{2}})\cap\mathcal{W}|=1$. As $\mathcal{W}\subset\mathcal{E}_i$, we have
\begin{equation}
\begin{array}{lll}
|W| &=& |W\cap \mathcal{E}_i|\\
&=& |W \cap (\cup_{\mb{v}\in V} B(\mb{v}, \frac{1-\cos\frac{\theta}{2}}{1+\cos\frac{\theta}{2}}))|\\
&=& |\cup_{\mb{v}\in V}W\cap B(\mb{v}, \frac{1-\cos\frac{\theta}{2}}{1+\cos\frac{\theta}{2}})|\\
&\le& \sum_{\mb{v}\in V} |W\cap B(\mb{v}, \frac{1-\cos\frac{\theta}{2}}{1+\cos\frac{\theta}{2}})|\\
&\le& \sum_{\mb{v}\in V} 1\\
&=& |V|
\end{array}
\end{equation}
Therefore, we have proved that $|W|$ is finite. Using that conclusion, we can construct a sequence of vectors $\mb{w_j}\in\mathcal{E}_i (j=1,\cdots,l)$ in the following way:
\begin{enumerate}
\item Let $\mb{w_1}\in\mathcal{E}_i$ be any vector in $\mathcal{E}_i$.
\item For $j=2,\cdots$, let $\mb{w_j}\in\mathcal{E}_i$ be any vector satisfying
\begin{align}
\label{eq:construct_w}
&\forall k=1,\cdots,j-1, \phi(\mb{w_j},\mb{w_k})\ge\theta\\
&\exists k\in\{,\cdots,j-1\}, \phi(\mb{w_j},\mb{w_k})=\theta
\end{align}
until we cannot find such vectors any more.
\item As we have proved that $|W|$ is finite, the above process will end in finite steps. Assume that the last vector we found is indexed by $l$.
\end{enumerate}
We will verify that such constructed vectors satisfy
\begin{equation}
\forall j\neq k\in\{1,\cdots,l\}, \rho(\mb{w_j},\mb{w_k})\ge \theta
\end{equation}
Due to the construction, $\phi(\mb{w_j},\mb{w_k})\ge\theta$, as $\rho(\mb{w_j},\mb{w_k})=\min(\phi(\mb{w_j},\mb{w_k}),\pi-\phi(\mb{w_j},\mb{w_k}))$, we only need to show that $\pi-\phi(\mb{w_j},\mb{w_k})\ge\theta$. To show that, we use the definition of $\mathcal{E}_i$ to write $\mb{w_j}$ as $\mb{w_j} = (\cos\theta_i,0,\cdots,0) + \mb{r_j}$, where $\|\mb{r_j}\|_2 = |\sin \theta_i|$. Similarly, $\mb{w_k} = (\cos\theta_i,0,\cdots,0) + \mb{r_k}$, where $\|\mb{r_k}\|_2 = |\sin \theta_i|$. Therefore $\cos(\phi(\mb{w_j},\mb{w_k}))=\mb{w_j}\cdot \mb{w_k} \ge \cos^2\theta_i-\sin^2\theta_i=\cos (2\theta_i)\ge\cos(\pi-\theta)$, where the last inequality follows from the construction of $\theta_i$. So $\pi-\phi(\mb{w_j},\mb{w_k})\ge\theta$, the proof for $\rho(\mb{w_j},\mb{w_k})\ge\theta$ is completed.

Now we will show that $\forall \mb{e}\in\mathcal{E}_i$, we can find $j\in\{1,\cdots,l\}$ such that $\phi(\mb{e},\mb{w_j})\le \theta$. We prove it by contradiction: assume that there exists $\mb{e}$ such that $\min_{j\in\{1,\cdots,l\}} \phi(\mb{e},\mb{w_j})>\theta$, then as $\mathcal{E}_j$ is a connected set, there is a path $q:t\in[0,1]\to \mathcal{E}_j$ connecting $\mb{e}$ to $\mb{w_1}$, and when $t=0$, the path starts at $q(0)=\mb{e}$; when $t=1$, the path ends at $q(1)=\mb{w_1}$. We define functions $r_j(t) = \phi(q(t), \mb{w_j})$ for $t\in[0,1]$ and $j=1,\cdots,l$. It is straightforward to see that $r_j(t)$ is continuous, hence $\min_j(r_j(t))$ is also continuous. As $\min_j(r_j(0))>\theta$ and $\min_j(r_j(0))=0<\theta$, there exists $t^*\in(0,1)$ such that $\min_j(r_j(0))=\theta$. Then $q(t^*)$ satisfies Condition \ref{eq:construct_w}, which contradicts the construction in $W$ as the construction only ends when we cannot find such vectors. Hence we have proved that
\begin{equation}
\forall \mb{e}\in\mathcal{E}_i,\exists j\in\{1,\cdots,l\}, \phi(\mb{e},\mb{w_j})\le \theta
\end{equation}

Now we can proceed to prove the main lemma. For each $i\in\mathcal{I}$, we use Condition \ref{eq:construct_w} to construct a sequence of vectors $\mb{w_{ij}}$. Then such constructed vectors $\mb{w_{ij}}$ have pairwise angles greater than or equal to $\theta$. Then for any $\mb{e} \in \mathcal{R}^d$ with $\|\mb{e}\|_2=1$, we write $\mb{e}$ in sphere coordinates as $\mb{e}=(\cos r_1, \sin r_1\cos r_2 ,\cdots,\prod_{j=1}^d\sin r_j)$. Use the same method as we did for the 2-dimensional case, we can find $\theta_i$ such that $|\theta_i-r|\le\frac{3}{2}\theta$. Then $\mb{e}'=(\cos\theta_i, \sin
\theta_i \cos r_2,\cdots, \sin\theta_i \prod_{j=2}^d\sin r_j)\in \mathcal{E}_i$. It is easy to verify that $\phi(\mb{e}, \mb{e}')=|\theta_i-r|\leq\frac{3}{2}\theta$. As $\mb{e}'\in\mathcal{E}_i$, there exists $\mb{w_{ij}}$ as we constructed such that $\phi(\mb{e}',\mb{w_{ij}})\leq\theta$. So $\phi(\mb{e},\mb{w_{ij}})\leq\phi(\mb{e},\mb{e}')+\phi(\mb{e}',\mb{w_{ij}})\le \frac{5}{2}\theta<3\theta$. So we have proved that for any $\mb{e}\in \mathbb{R}^d$ with $\|\mb{e}\|_2=1$, we can find $\mb{w_{ij}}$ such that $\phi(\mb{e},\mb{w_{ij}}) < 3\theta$.

For any $\mb{w_{ij}}$, assume $i+1\in \mathcal{I}$, we first project $\mb{w_{ij}}$ to $\mb{w}^*\in\mathcal{E}_{i+1}$. We have proved that $\phi(\mb{w_{ij}},\mb{w}^*)\le\frac{3}{2}\theta$. We've also proved that we can find $\mb{w_{i+1,j'}}\in\mathcal{E}_{i+1}$ such that $\phi(\mb{w_{i+1,j'}},\mb{w}^*)\le\theta$. So we have found $\mb{w_{i+1,j'}}$ such that $\phi(\mb{w_{ij}},\mb{w_{i+1,j'}})\le\frac{5}{2}\theta < 3 \theta$. We can use similar scheme to prove that $\forall \mb{w_{ij}}$, there exists different $\mb{w_{i_1,j_1}}\cdots, \mb{w_{i_{2k+1},j_{2k+1}}}$ such that $(i_r,j_r)\neq (i,j)$ and $\phi(\mb{w_{ij}},\mb{w_{i_r,j_r}})\le 3r\theta$. Following the same proof as the 2-dimensional case, we can prove that if $m\le 2k+1$, then we can find a set of vectors $(\mb{w_j})_{j=1}^m$ such that
\begin{align}
&\forall j\in\{1,\cdots,m\},\phi(\mb{w_j}',\mb{w_j})\le \min(3m\theta,\pi)\\
&\forall j\in\{1,\cdots,m\}, \|\mb{w_j}'\|_2 = \|\mb{w_j}\|_2\\
&\forall j\neq k, \rho(\mb{w_j},\mb{w_k})\ge \theta
\end{align}
The proof completes.
\end{proof}
\begin{mylemma}
For any $f' \in \mathcal{F}'$, $\exists f \in \mathcal{F}''$ such that
\begin{equation}
\|f' - f\|_L \le 4mCC_1C_3\sin(\frac{\theta'}{2})
\end{equation}
where $\theta' = \min(3m\theta,{\pi})$.
\end{mylemma}
\begin{proof}
According to the definition of $\mathcal{F}'$, $\forall f' \in \mathcal{F}'$, there exists $(\alpha_j')_{j=1}^m$, $(\mb{w_j}')_{j=1}^m$ such that
\begin{align}
&f' = \sum_{j=1}^m \alpha_j' h(\mb{w_j}'^T x)\\
&\forall j\in\{1,\cdots,m\}, |\alpha_j'|\le 2 C\\
&\forall j\in\{1,\cdots,m\}, \|\mb{w_j}'\|_2 \le C_4
\end{align}
According to Lemma \ref{lem:theta_appro}, there exists $(\mb{w_j})_{j=1}^m$ such that
\begin{align}
&\forall j\neq k\in\{1,\cdots,m\}, \rho(\mb{w_j},\mb{w_k})\ge \theta\\
&\forall j\in\{1,\cdots,m\}, \|\mb{w_j}\|_2 = \|\mb{w_j}'\|_2\\
&\forall j \in \{1,\cdots,m\},\arccos (\frac{\mb{w_j} \cdot \mb{w_j}'}{\|\mb{w_j}\|_2\|\mb{w_j}'\|_2}) \le \theta'
\end{align}
where $\theta' = \min(m\theta,\frac{\pi}{2})$. Let $f = \sum_{j=1}^m \alpha_j h(\mb{w_j}'^T \mb{x})$, then $\|\mb{\alpha}\|_2 \le \sqrt{\|\mb{\alpha}\|_1\|\mb{\alpha}\|_\infty} \le 2\sqrt{m}C\le C_4$. Hence $f\in\mathcal{F}$. Then all we need to do is to bound $\|f-f'\|_L$:
\begin{equation}
\label{eq:bound_f_f_prime}
\begin{array}{lll}
\|f-f'\|_L^2&=& \int_{\|\mb{x}\|_2\le C_1} (f(\mb{x})-f'(\mb{x}))^2 dP(\mb{x})\\
&=& \int_{\|\mb{x}\|_2\le C_1} (\sum_j \alpha_j h(\mb{w_j}^T \mb{x})-\sum_j \alpha_j h(\mb{w_j}'^T \mb{x}))^2 dP(\mb{x})\\
&=& \int_{\|\mb{x}\|_2\le C_1} (\sum_j \alpha_j (h(\mb{w_j}^T \mb{x})- h(\mb{w_j}'^T \mb{x})))^2 dP(\mb{x})\\
&\le& \int_{\|\mb{x}\|_2\le C_1} (\sum_j |\alpha_j| |\mb{w_j}^T \mb{x}- \mb{w_j}'^T \mb{x}|)^2 dP(\mb{x})\\
&\le& C_1^2\int_{\|\mb{x}\|_2\le C_1} (\sum_j |\alpha_j| \|\mb{w_j}- \mb{w_j}'\|_2)^2 dP(\mb{x})
\end{array}
\end{equation}
As $\arccos (\frac{\mb{w_j} \cdot \mb{w_j}'}{\|\mb{w_j}\|_2\|\mb{w_j}'\|_2}) \le \theta'$, we have $\mb{w_j}\cdot \mb{w_j}'\ge \|\mb{w_j}\|_2^2 \cos \theta'$. Hence
\begin{equation}
\begin{array}{lll}
\|\mb{w_j}- \mb{w_j}'\|_2^2&=& 2\|\mb{w_j}\|_2^2 - 2 \mb{w_j} \cdot \mb{w_j}'\\
&\le& 2\|\mb{w_j}\|_2^2 - 2\|\mb{w_j}\|_2^2 \cos \theta'\\
&\le& 4C_3^2\sin^2(\frac{\theta'}{2})
\end{array}
\end{equation}
Substituting back to Eq.(\ref{eq:bound_f_f_prime}), we have
\begin{equation}
\begin{array}{lll}
\|f-f'\|_L^2&\le& C_1^2\int_{\|\mb{x}\|_2\le C_1} (\sum_j |\alpha_j| 2 C_3 \sin(\frac{\theta'}{2}) )^2 dP(\mb{x})\\
&\le& 16m^2C^2C_1^2C_3^2\sin^2(\frac{\theta'}{2})
\end{array}
\end{equation}
\end{proof}

\section{Replicated Softmax RBM}
\label{apd:rsr}
To better model word counts, \cite{hinton2009replicated} proposed Replicated Softmax RBM. Let $\mb{V}$ be a $D\times J$ observed binary matrix of a document containing $D$ tokens. $J$ is the vocabulary size. Row $i$ of $\mb{V}$ is the 1-of-$J$ coding vector of the $i$th token in this document. $V_{ij}=1$ if the $i$th token is the $j$th word in the vocabulary. Under this representation, the energy function $E(\mb{h},\mb{V})$ is defined as
\begin{equation}
\begin{array}{l}
-\sum\limits_{i=1}^{D}\sum\limits_{j=1}^{J}\alpha_{j}V_{ij}-D\sum\limits_{k=1}^{K}
\beta_{k}h_{k}
-\sum\limits_{i=1}^{D}\sum\limits_{j=1}^{J}\sum\limits_{k=1}^{K}A_{jk}V_{ij}h_{k}
\end{array}
\end{equation}
Given the observed tokens $\mb{V}$, inferring the latent representation $\mb{h}$ can be done very efficiently
\begin{equation}
p(h_{k}=1|\mb{V})=\sigma(D\beta_{k}+\sum\limits_{i=1}^{D}\sum\limits_{j=1}^{J}A_{jk}V_{ij})
\end{equation}
where $\sigma(x)=1/(1+\exp(-x))$ is the logistic function.
The model parameters can be learned by maximizing the data likelihood $\mathcal{L}(\{\mb{V}_{n}\}_{n=1}^{N};\mb{A},\boldsymbol\alpha, \boldsymbol\beta)$ using the contrastive divergence \citep{hinton2002training} method.

\bibliographystyle{abbrv}
\bibliography{sample}

\end{document}